\documentclass[11pt]{article}

\usepackage{amsmath}
\usepackage{amsfonts}
\usepackage{amssymb, amsthm, epsfig}
\usepackage{mathrsfs}
\usepackage{euscript}

\usepackage{graphicx}
\usepackage{verbatim}
\usepackage{color}
\usepackage{epstopdf}

\usepackage{mathtools}
\usepackage{float}
\usepackage{algorithm}
\usepackage[noend]{algorithmic}

\usepackage{xspace}
\usepackage{color}

\usepackage{wrapfig}
\usepackage{bbm}

\usepackage{thmtools} 
\usepackage{thm-restate}
\usepackage{booktabs}

\usepackage[numbers,sort&compress]{natbib}
\usepackage[colorlinks=true, linkcolor=blue, citecolor=blue]{hyperref}

\newcommand{\eps}{\varepsilon}

\newcommand{\dir}{\ensuremath{\mathsf{d}}}

\DeclareMathOperator{\poly}{\textsf{poly}}

\newcommand{\norm}[1]{\lVert #1 \rVert}

\newcommand\numberthis{\addtocounter{equation}{1}\tag{\theequation}}
\newcommand{\abs}[1]{| #1 |}
\newcommand{\setdef}[2]{\left\{ #1 \mid #2 \right\}}
\newcommand{\inner}[2]{\langle #1,#2 \rangle}

\newcommand{\Gr}{\textsf{Grid}}

\newcommand{\conv}{*} 
\DeclareMathOperator{\supp}{supp}
\DeclareMathOperator{\spn}{span}

\newtheorem{theorem}{Theorem}
\newtheorem{lemma}[theorem]{Lemma}

\newtheorem{remark}[theorem]{Remark}

\usepackage[inline]{enumitem}

\newcommand\blfootnote[1]{%
  \begingroup
  \renewcommand\thefootnote{}\footnote{#1}%
  \addtocounter{footnote}{-1}%
  \endgroup
}

\usepackage[a4paper]{geometry}
\title{Tight Bounds on the Hardness of Learning Simple\\Nonparametric Mixtures}
\author{Bryon Aragam \\ University of Chicago \and Wai Ming Tai \\ University of Chicago}
\date{}

\begin{document}

\maketitle

\begin{abstract}%

We study the problem of learning nonparametric distributions in a finite mixture, and establish tight bounds on the sample complexity for learning the component distributions in such models.
Namely, we are given i.i.d. samples from a pdf $f$ where 
$$
f=w_1f_1+w_2f_2, \quad w_1+w_2=1, \quad w_1,w_2>0
$$
and we are interested in learning each component $f_i$.
Without any assumptions on $f_i$, this problem is ill-posed.
In order to identify the components $f_i$, we assume that each $f_i$ can be written as a convolution of a Gaussian and a compactly supported density $\nu_i$ with $\text{supp}(\nu_1)\cap \text{supp}(\nu_2)=\emptyset$.

Our main result shows that $(\frac{1}{\varepsilon})^{\Omega(\log\log \frac{1}{\varepsilon})}$ samples are required for estimating each $f_i$. 
The proof relies on a quantitative Tauberian theorem that yields a fast rate of approximation with Gaussians, which may be of independent interest. 
To show this is tight, we also propose an algorithm that uses $(\frac{1}{\varepsilon})^{O(\log\log \frac{1}{\varepsilon})}$ samples to estimate each $f_i$. 
Unlike existing approaches to learning latent variable models based on moment-matching and tensor methods, our proof instead involves a delicate analysis of an ill-conditioned linear system via orthogonal functions.
Combining these bounds, we conclude that the optimal sample complexity of this problem properly lies in between polynomial and exponential, which is not common in learning theory.

\end{abstract}

\blfootnote{Accepted for presentation at the Conference on Learning Theory (COLT) 2023}

\section{Introduction}\label{sec:intro}

A mixture model is a probabilistic latent variable model that corresponds to a mixture of unknown distributions such that each distribution represents a subpopulation within an overall population. 
Easily the most commonly studied model is the Gaussian mixture model (GMM), which dates (at least) back to \cite{pearson1894contributions}. Although GMMs are a notoriously complex model to analyze, recent years have born witness to substantial progress on statistical and algorithmic fronts \citep{dasgupta1999learning,vempala2004spectral,regev2017learning,moitra2010settling,hardt2015tight,wu2018optimal,doss2020optimal,feldman2006pac,daskalakis2014faster,wu2018improved,suresh2014near,li2015nearly,bhaskara2015sparse,wu2018improved,chan2014efficient,acharya2017sample}.
Nevertheless, in applications, parametric assumptions such as Gaussianity are often unrealistic, and it is preferred to allow the component distributions to be as flexible as possible. In this case we are faced with the problem of learning a \emph{nonparametric} mixture model. More broadly, nonparametric mixtures represent the simplest nonparametric latent variable model of interest in applications: If learning a nonparametric mixture is hard, this suggests fundamental barriers to learning more complex latent variable models (e.g. deep generative models, autoencoders, etc.) in practice.

Nonparametric mixtures come in a variety of forms with many applications. 
In this paper, we are interested in simple, two-component mixture models in which the mixture components themselves are allowed to be nonparametric:
Define a probability density function (pdf) $f$ as
\begin{align}
\label{eq:npmix}
    f 
    = w_1f_1+w_2f_2
\end{align}
where $w_1,w_2>0$, $w_1+w_2=1$ and $f_1,f_2$ are some (unknown) pdfs.
Here, each $f_i$ will be allowed to come from a flexible, nonparametric family of distributions. 
Our goal is to study the sample complexity of this problem.
Unfortunately, without additional assumptions on the $f_i$, it is clear that this problem is ill-posed: There are infinitely many possible ways to write $f$ as a mixture model of the form \eqref{eq:npmix}. 
We will assume that $f_{i}=\nu_{i}\conv g_0$, where $g_{0}$ is the pdf of the standard Gaussian distribution centered at $0$, $\nu_{i}$ is a probability density supported on an interval and $\supp(\nu_{1})\cap \supp(\nu_{2})=\emptyset$ and $\conv$ is the convolution operator. 
This setting cleanly encapsulates the nonparametric setting we are interested in by allowing each $f_i$ to be essentially arbitrary while still ensuring identifiability (owing to the separation condition $\supp(\nu_{1})\cap \supp(\nu_{2})=\emptyset$), and has been studied previously \citep{koltchinskii2000empirical,nguyen2013convergence,aragam2018npmix,aragam2023uniform}.

Given the apparent generality of the problem under consideration, it is worthwhile to compare the sample complexity of recovering $f_{i}$ from $f$ to similar problems such as learning GMMs and deconvolution (see Table~\ref{table:1}).
Our main result shows that this problem, which is a natural generalization of GMM learning, is strictly harder and cannot be solved in polynomial-time.
To complement this hardness result, we also prove that our super-polynomial lower bound is tight. 
As described in more detail in the next section, this suggests an interesting middle ground between parameter learning and density estimation that has not been observed in the mixture literature previously.

The proof of our results may also be independently interesting. The lower bound construction involves the analysis of a delicate Gaussian approximation scheme and its rate of convergence, which provides a quantitative version of Wiener's Tauberian theorem \citep{wiener1932tauberian,wiener1933fourier}.
The upper bound analysis sidesteps traditional parametric approaches such as moment matching, tensor decompositions, and the EM algorithm and instead solves a nearly ill-conditioned linear system that arises from a Hermite polynomial expansion of the $f_i$. 

Finally, although this problem is interesting on its own, we mention two important applications that motivate this work:

\paragraph{Nonparametric clustering.} Here the goal is to partition a set of $n$ data points into $k$ clusters while making as few assumptions on the clusters as possible. 
In model-based clustering, we assume a mixture model as in \eqref{eq:npmix}, where each $f_{i}$ represents a single ``cluster'', and each sample is drawn from a randomly selected cluster with probability $w_{i}$. 
The optimal clustering is then given by the Bayes optimal partition, which is defined by the resulting Bayes classifier.
This problem has been well-studied in the literature \citep{achlioptas2005,kannan2008,kumar2010clustering,mixon2017,aragam2018npmix}.

\paragraph{Nonparametric latent variable models.}
Latent variable models with flexible nonparametric dependencies arise in many applications, and the two-component mixture we consider here is arguably one of the simplest such models. The analysis of \eqref{eq:npmix} provides fundamental insight into complexity of more general latent variable models. 
Examples of such models include variational autoencoders (VAEs), generative adversarial networks (GANs), normalizing flows, and diffusion models. Given the widespread popularity and adoption of these methods, understanding the complexity of identifying and learning the latent structure of these models is a fundamental problem that has received surprisingly little attention. Although there have been substantial developments in our understanding of \emph{density estimation} in these models \citep[e.g.][]{uppal2019nonparametric,biau2021some,ding2020high,belomestny2021rates}, our interest here is learning the underlying \emph{latent structure}, which is a more difficult problem.
Another application in which learning the components of a nonparametric mixture model explicitly arises is the problem of causal representation learning: Here, the goal is to learn high-level latent variables with meaningful causal relationships from low-level observations.
Recently, \cite{kivva2021learning,kivva2022identifiability} showed that this problem can be reduced to the problem of learning a nonparametric mixture model such as \eqref{eq:npmix}.

\subsection{Problem Definition}\label{sec:prob_def}

For any $\mu\in\mathbb{R}$, let $g_\mu$ be the pdf of a unit variance Gaussian distribution centered at $\mu$, i.e.
\begin{align*}
    g_\mu(x) &= \frac{1}{\sqrt{2\pi}} e^{-\frac{1}{2}(x-\mu)^2} \qquad\text{for all $x\in\mathbb{R}$.}
\end{align*}
For any set $S$, let $\mathcal{P}_S$ be the set of pdfs of all distributions on $S$.
For any interval $I\subset\mathbb{R}$, let
\begin{align*}
	\mathcal{G}_I
	:= \setdef{f\in \mathcal{P}_{\mathbb{R}}}{f = \int_{\mu\in I}\nu(\mu)g_{\mu}\dir\mu \text{ where $\nu\in\mathcal{P}_I$}}. \numberthis\label{eqn:ig}
\end{align*}
Namely, $\mathcal{G}_I$ is the collection of convolutions of a standard Gaussian $g_0$ with some distribution $\nu$ whose support lies in the interval $I$.
We call such distributions \emph{interval Gaussians}.
When $I_1$ and $I_2$ are clear in the context, we use $\mathcal{G}_i$ as a shorthand for $\mathcal{G}_{I_i}$.

Given two intervals $I_1$ and $I_2$, define a pdf $f$ by
\begin{align}
\label{eq:main:model}
    f=w_1f_1 + w_2f_2,
	\quad f_i\in\mathcal{G}_{i},
	\quad I_1\cap I_2=\emptyset,
	\quad w_i> 0,
	\quad w_1+w_2=1.
\end{align}
Since $f_i\in\mathcal{G}_i$, we can write $f_{i}=\nu_{i}\ast g_{0}$ for $i=1,2$, where $\nu_i\in\mathcal{P}_{I_i}$. 
We let $\nu=w_1\nu_{1}+w_2\nu_{2}$ denote the global mixing density, whence $f=\nu\ast g_0$.

Suppose we are given a set of samples drawn from $f$.
Then, what is the sample complexity for estimating each component $f_i$?
Before answering this question, we must first address the identifiability of this model.
If we assume that these two intervals are known, it is easy to see that this model is identifiable.
If these two intervals are unknown, then as long as they are well-separated, the model will be identifiable
(see Section~\ref{sec:sep} for details).
Formally, we have the following problem: 
\begin{quote}
	\emph{Let $P$ be a set of $n$ i.i.d. samples drawn from $f = w_1f_1+w_2f_2$, where $f$ is defined as in \eqref{eq:main:model} and $I_1,I_2$ are unknown and well-separated. For a sufficiently small error $\eps>0$, what is the threshold $\tau_\eps$ such that
	\begin{itemize}
	    \item if $n<\tau_\eps$, then no algorithm taking $P$ as the input returns two pdfs $\widehat f_1,\widehat f_2$ such that $\norm{f_i - \widehat f_i}_1 < \eps$ with probability at least $1-\frac{1}{100}$ for some $f$?
	    \item if $n>\tau_\eps$, then there is an algorithm that takes $P$ as the input and returns two pdfs $\widehat f_1,\widehat f_2$ such that $\norm{f_i - \widehat f_i}_1 < \eps$ with probability at least $1-\frac{1}{100}$ for any $f$?
	\end{itemize}}
\end{quote}
Without loss of generality, we can assume that $I_1$ is the left interval and $I_2$ is the right interval by reordering the indices.
Here, we are focusing on learning the components $f_i$ and treating the weights $w_i$ as nuisance parameters.

\subsection{Separation Assumptions}
\label{sec:sep}

To ensure identifiability when the intervals are unknown, some kind of separation is needed.
Let $R$ be the minimum distance between the endpoints of two intervals.
Unsurprisingly, the difficulty of the problem depends acutely on how this value varies. 
Our main interest is the case where $R$ is independent of $\eps$, and in particular, does not diverge as the number of samples increases. 
Formally, we may consider three separate regimes: $\Theta(1)$-separation (our focus), $\omega(1)$-separation ($R\to\infty$),  and $o(1)$-separation ($R\to0$).
\begin{itemize}
    \item $\Theta(1)$-separation: $R$ is independent of $\eps$; this regime is our main focus. 
    \item $\omega(1)$-separation: $R\rightarrow \infty$ as $\eps\rightarrow 0$.
    This learning problem is easy.
    For example, when $R=\sqrt{\log 1/\eps}$, 
    one can apply the clustering technique (e.g. \cite{kumar2010clustering}) to learn each component in polynomial time. 
    
    \item $o(1)$-separation: $R\rightarrow 0$ as $\eps\rightarrow 0$.
    If the intervals are unknown, this will cause identifiability issues when $R$ is larger than the length of the intervals.
\end{itemize}
It is worth pointing out that the main difficulty in the analysis arises when the intervals are known, and a simple pre-processing step suffices to reduce the unknown case to known intervals (see Remark~\ref{rem:known_int} and Appendix~\ref{sec:extension}).
For a more refined analysis of the relationship between separation and identifiability, see \citep{aragam2018npmix,aragam2023uniform}.

\subsection{Learning Goal}
\label{sec:goal}
To provide additional context for this problem, we recall that mixture modeling problems can be broadly classified into two general categories:
\begin{itemize}
	\item \emph{Parameter learning.}
	The most common example of parameter learning is for GMMs, in which case we seek to estimate the weights $w_i$ and centers $\mu_i$ (and possibly the variances $\Sigma_{i}$) for each component.
	In our nonparametric setting,
	recalling our definition of the mixing density $\nu=w_1\nu_1+w_2\nu_2$,
	parameter learning would mean estimating $\nu$, i.e. we find another pdf $\nu'$ such that $\nu$ and $\nu'$ are close in say the Wasserstein distance.
	Under our assumptions, this is equivalent to deconvolution, which requires exponentially many samples (see Section~\ref{sec:related} for details).
	
	\item \emph{Density estimation.}
	Here we estimate the mixture distribution $f$ directly, i.e. we find another pdf $f'$ such that $f$ and $f'$ are close in say the total variation or the Hellinger distance.
	It is known that it only needs polynomially many samples to achieve this goal.
\end{itemize}
It is obvious that parameter learning implies density estimation; in particular, parameter learning is at least as hard as density estimation.
These two general problems inspire an intriguing question: 
\begin{itemize}
	\item \emph{Can we acquire any guarantee in between parameter learning and density estimation?} Instead of parameter learning or density estimation, we seek to learn the components $f_i$ rather than the mixture distribution $f$ or the mixing density $\nu$.
	For this task, we do not need to estimate each $\nu_i$.
	Of course, one could learn $\nu$ in the traditional sense such as parameter learning, however, this is not necessary.
\end{itemize}
Since nonparametric density estimation can be done efficiently and deconvolution is provably hard, what can be said about this ``in-between'' problem? 
Our results shed light on this problem from a new perspective.

\section{Related Work}\label{sec:related}

Table~\ref{table:1} compares the results for our setting to other related problems discussed in this section.

\paragraph{Mixture Models}

Roughly speaking, our model can be viewed as a GMM with infinitely many Gaussians whose centers are well-clustered.
Since there are infinitely many centers in this model, traditional techniques for learning mixtures of finitely many Gaussians may not be applicable.
In parameter learning, the goal is to estimate the means and weights (and variances) of the Gaussians \citep{dasgupta1999learning,vempala2004spectral,regev2017learning,moitra2010settling,hardt2015tight,wu2018optimal,doss2020optimal}. 
Parameter learning for GMMs has an exponential dependence of $k$ and hence when $k$ is a constant it can be accomplished in polynomial-time in $\frac{1}{\eps}$.
Density estimation, on the other hand, can further be split into two categories: proper learning and improper learning.
In proper learning, the outputs are restricted to be a mixture of $k$ Gaussians where $k$ is the number of Gaussians in the underlying model \citep{feldman2006pac,daskalakis2014faster,wu2018improved,suresh2014near,li2015nearly} while, in improper learning, the output is unrestricted \citep{bhaskara2015sparse,wu2018improved,chan2014efficient,acharya2017sample}. 
Note that the sample complexity under these settings is polynomial in $k$ and $\frac{1}{\eps}$.

\begin{table}[t]
\centering
\begin{tabular}{l l l} 
    \toprule
    Setting & Sample bound & Learning goal \\
    \midrule
    Density estimation & $\poly(1/\eps)$ & learning the density $f$\\
    $k$-GMM & $(1/\eps)^{O(k)}$ & learning the parameters $(w_i,\mu_i)$\\
    Deconvolution & $2^{\poly(1/\eps)}$ & learning the mixing density $\nu$\\
    \midrule
    \bf{Our setting} & $(1/\eps)^{\log\log (1/\eps)}$ & learning the components $f_i$ \\
    \bottomrule
\end{tabular}
\caption{Comparison of our result and related work.}
\label{table:1}
\end{table}

Compared to learning GMMs, less is known about nonparametric mixtures. 
One strand of literature beginning with \cite{teicher1967} assumes that each $f_{i}$ is a product distribution while allowing each marginal to be nonparametric. In this case, the parameters $(w_{i},f_{i})$ are identifiable, and consistent estimators can be constructed \citep{hall2003,elmore2005,hall2005mixture}.
Recently there has been progress on learning algorithms for this model \citep{chaudhuri2008learning,rabani2014learning,li2015learning,gordon2021hadamard,gordon2021source}.
We note also related work on nonparametric mixtures in the statistics literature  \citep{shi2009,allman2009,nguyen2013convergence,vandermeulen2019operator}. 
Variants of the convolution model \eqref{eqn:ig} have been studied previously, however, precise hardness or sample complexity bounds are missing. For example, \cite{koltchinskii2000empirical} discusses recovery of the intervals $I_1,\ldots,I_k$ and \cite{aragam2018npmix} proves identifiability and asymptotic consistency without finite-sample theory. 

Beyond parameter learning, the literature has also studied clustering, i.e. achieving low misclassification error of the Bayes classifier defined by the mixture \eqref{eq:npmix}, without the need to impose identifiability assumptions \citep{achlioptas2005,kumar2010clustering}. Assuming $\omega(1)$-separation, \cite{kannan2008} are able to learn general log-concave mixtures and \cite{mixon2017clustering} learn subgaussian mixtures.

\paragraph{Deconvolution}

Even though traditional techniques for learning GMMs might not be helpful, nonparametric deconvolution is one way to solve our problem, albeit with suboptimal sample complexity.
Recall that by \eqref{eqn:ig}-\eqref{eq:main:model}, we can write $f = (w_1\nu_1 + w_2\nu_2)\conv g_0 = \nu\conv g_0$ where $\nu = w_1\nu_1 + w_2\nu_2$.
Algorithms for deconvolution return another mixing density $\widehat \nu$ such that $\widehat \nu \approx \nu$ given a set of samples drawn from $f$.
Since we assume that the support of $\nu$ is in the union of two disjoint intervals, a simple truncation argument provides a way to break $\widehat \nu $ into two parts, $\widehat \nu_1$ and $\widehat \nu_2$, such that $\nu\approx w_1\widehat\nu_1 + w_2\widehat \nu_2$ and $\widehat\nu_i\conv g_0 \approx f_i$. 
It is worth noting that although results on deconvolution often assume some smoothness conditions on $\nu$---which we do not assume---this can easily be fixed by smoothing the mixing density $\nu$ first. 
Regardless, learning $\nu$ directly \emph{requires} exponentially many samples, and this cannot be improved---see \cite{meister2009deconvolution} for a detailed account.

For example, \cite{zhang1990fourier} showed that the minimax rate in estimating $\nu$ in the $L^2$ norm is bounded from above by $\poly(\log n)^{-1}$ where $n$ is the number of samples.
More recently, in \cite{nguyen2013convergence} it was shown that the Wasserstein distance between any two mixing densities is bounded from above by $\poly(\log \frac{1}{V})^{-1}$ where $V$ is the total variation between the two mixture densities. These results imply that exponentially many (i.e. $2^{\Theta(\frac{1}{\eps})}$) samples are required to estimate $\nu$ directly, either in $L^2$ or the weaker Wasserstein metric. Other related results on deconvolution include \cite{carroll1988optimal,stefanski1990deconvolving,fan1991optimal,gassiat2020deconvolution}.
Thus, in order to break the exponential barrier for our problem \eqref{eq:npmix}, deconvolution techniques must be avoided.

\paragraph{Latent Variable Models}
A standard approach to learning latent variable models is moment matching, which is closely related to tensor decompositions that have been used for learning topic models \citep{anandkumar2015}, mixed regression models \citep{chaganty2013,chen2014convex,hand2018convex}, hidden Markov models \citep{anandkumar2012mixture,gassiat2013finite,mossel2005learning}, and latent graphical models \citep{anandkumar2012learning,anandkumar2013} in addition to mixture models. Due to their widespread applicability, tensor methods have been the subject of intense scrutiny in the theory literature \citep{anandkumar2014tensor,allman2009,diakonikolas2020small,bhaskara2014smoothed}. Another standard approach is the EM algorithm. Although theoretical guarantees on the EM algorithm are more difficult to obtain, recent work has produced some exceptions for GMMs \citep{balakrishnan2017statistical,cai2017chime} and mixed regression \citep{kwon2020minimax}. Our proof technique, by contrast is distinct by necessity: Both moment-based methods and the EM algorithm are notoriously difficult to analyze for infinite-dimensional (i.e. nonparametric) models. Instead, we use orthogonal functions to reduce our problem to a linear system whose analysis involves careful control over the approximation rate and conditioning.

\section{Our Results}\label{sec:result}

Our main result shows that given a set of samples from a mixture of two interval Gaussians as in \eqref{eq:main:model}, estimating each interval Gaussian requires super-polynomially many samples, and the requisite sub-exponential sample complexity is tight.

We first show the sample complexity has a super-polynomial lower bound.
Although our problem definition allows the weights $w_1,w_2$ to be arbitrary, our result shows that even when the weights are known to be balanced, the problem is still hard.
Recall that, given any interval $I$, $\mathcal{G}_I$ is defined by \eqref{eqn:ig} as the collection of convolutions of a standard Gaussian $g_0$ with some distribution $\nu$ such that $\supp(\nu)\subset I$.
Formally, we have the following lower bound:

\begin{theorem}\label{thm:main_lower}
    Let $\eps>0$ be a sufficiently small error and $I_1,I_2$ be two known disjoint intervals.
    There exists a distribution whose pdf is $f^*=\frac{1}{2}f^*_1+\frac{1}{2}f^*_2$ where $f^*_i\in \mathcal{G}_{i}$ such that no algorithm taking  a set of $n$ i.i.d. samples drawn from $f^*$ as input returns two pdfs $\widehat f_1,\widehat f_2$ such that $\norm{f^*_i - \widehat f_i}_1 < \eps$ with probability at least $1-\frac{1}{100}$ whenever $n<(\frac{1}{\eps})^{C\log\log \frac{1}{\eps}}$ where $C$ is an absolute constant.
\end{theorem}

This theorem makes no assumptions on the separation $R$ (the distance between the closest endpoints of two intervals), and in particular holds in the regime of $R=O(1)$. The only implied assumption on $R$ is that $I_1\cap I_2=\emptyset$---i.e $R>0$---which allows for fixed separation as $\eps\to0$. In particular, our result holds when $R=\Theta(1)$ stays bounded away from zero.

Since we already know that this problem has an \emph{exponential} upper bound from deconvolution (see Section~\ref{sec:related}), the lower bound from Theorem~\ref{thm:main_lower} leaves open the question whether or not there is a \emph{sub-exponential} algorithm that matches the super-polynomial lower bound.
Our second main result shows that this is indeed the case, i.e. the lower bound in Theorem~\ref{thm:main_lower} is tight.

\begin{remark}
\label{rem:known_int}
In Theorem \ref{thm:main_upper} below, we assume knowledge of the intervals $(I_1,I_2)$, but this is purely for simplicity: The main difficulties in the proof arise even when these intervals are known, and it is straightforward to approximate the intervals as a pre-processing step when they are unknown. For completeness, we have included these details in Appendix \ref{sec:extension}.
\end{remark}
Formally, we have the following theorem:

\begin{theorem}\label{thm:main_upper}
    Let $\eps>0$ be  a sufficiently small error and $I_1,I_2$ be two known disjoint intervals of length $1$ such that $r>4$ where $r$ is the distance between the centers of the intervals.
    There exists an algorithm such that, for any distribution whose pdf is $f=w_1f_1+w_2f_2$ where $f_i\in \mathcal{G}_{i}$, $w_i=\Omega(\eps)$ and $w_1+w_2=1$, the algorithm taking a set of $n$ i.i.d. samples from $f$ as input returns two pdfs $\widehat f_1,\widehat f_2$ such that $\norm{f_i - \widehat f_i}_1 < \eps$ with probability at least $1-\frac{1}{100}$ whenever $n>(\frac{1}{\eps})^{C\log\log\frac{1}{\eps}}$ where $C$ is an absolute constant.

\end{theorem}

Theorem \ref{thm:main_upper} is stated so as to draw attention to the assumptions and characteristics of the problem that reflect the main technical challenges addressed by our analysis. Nonetheless, it is possible to generalize this result in several directions, as discussed in the remarks below.

\begin{remark}
\label{rem:gen:simple}
In several places we have not bothered to optimize the analysis, which we outline here for the interested reader:
 \begin{itemize}
    \item Although the distance between two centers of the intervals, $r$, is assumed to be larger than $4$, this lower bound is not optimized.
    Clearly, it cannot be less than $1$ since otherwise the two intervals intersect, leading to identifiability issues.
    \item The length of the intervals is assumed to be $1$ for simplicity; this can be replaced with any constant $2s$ for $s>0$.
    We just need to modify the proof accordingly and $r$ needs to be larger than $8s$.
    \item In our algorithm, we assume that the exact computation of an integral can be done.
    One can always approximate an integral arbitrarily well and we assume that this error is negligible. 
    It does not change the sample complexity.
\end{itemize}
Optimizing these dependencies is an interesting direction for future work.
\end{remark}

\section{Proof Overview} 

In this section, we will give an overview of our proofs.
We first outline the lower bound result in  Section \ref{sec:overview_lower}.
Then, we outline the upper bound result in Section \ref{sec:overview_upper}.
Relevant preliminaries and detailed proofs are deferred to the appendix.

\subsection{Lower Bound}\label{sec:overview_lower}

Our goal is to construct two mixtures $f=\tfrac12 f_1+\tfrac12f_2$ and $f'=\tfrac12 f_1'+\tfrac12f_2'$ such that the difference between the components $f_i$ and $f_i'$ is large compared to the difference between the mixtures $f$ and $f'$. More specifically, we want the following:
\begin{align}
	\label{eq:goal}
	\norm{f_1-f_1'}_1>\eps, \quad
	\norm{f_2-f_2'}_1>\eps, \quad
	\norm{f-f'}_1<o(\eps^c) \quad \text{for any $c>0$}.
\end{align}
By the standard information theoretic lower bounds, this implies that $\omega(1/\eps^c)$ samples are needed to distinguish $f$ and $f'$.
Suppose $f_1,f_1'\in\mathcal{G}_{1}$, $f_2,f_2'\in\mathcal{G}_{2}$, $I_1:=[0,1]$ and $I_2:=[-2,-1]$. 
One way for $\norm{f-f'}_1$ to be bounded above is that $f_1-f_1'$ should assign nontrivial mass outside of $[0,1]$, and similarly $f_2-f_2'$ should assign nontrivial mass outside of $[-2,-1]$. A simple way to accomplish this is to have
\begin{align*}
	f_1-f_1'
	\approx \lambda\cdot(g_{-1} - g_{0})
	\quad\text{and}\quad
	f_2-f_2'
	\approx \lambda\cdot(g_{0} - g_{-1})
\end{align*}
for some small $\lambda$.
Consequently, we have
\begin{align*}
	f-f' \approx \frac{1}{2}\lambda\cdot(g_{-1} - g_{0}) + \frac{1}{2}\lambda\cdot(g_{0} - g_{-1}) =0.
\end{align*}
But then
\begin{align*}
	L_1
	:= \frac{1}{\lambda}(f_1 - f_1') + g_{0}
	\approx g_{-1} 
	\quad    \text{and} \quad
	L_2
	:= \frac{1}{\lambda}(f_2 - f_2') + g_{-1}
	\approx g_{0} 
\end{align*}
where $L_1$ is a linear combination of Gaussians centred inside $[0,1]$. Note that $L_1\notin\mathcal{G}_{1}$ since $L_1$ is a \emph{linear} combination not a \emph{convex} combination; that is, it may have large and negative coefficients in its expansion. A similar argument applies to $L_2$, which is a linear combination of Gaussians centred inside $[-2,-1]$.
Thus, as long as we can construct $L_1$ and $L_2$---along with the appropriate rates in \eqref{eq:goal}, we can achieve the desired goal.
The key to this construction is the surprising fact that a single Gaussian centred anywhere can be approximated extremely well by a linear combination of Gaussians centred at points in an arbitrary interval.

To construct such linear combination with the appropriate rates, we consider the following construction.
Let $\Gr(\Delta)$ be a grid of cell width $\Delta>0$ over the reals, i.e.
\begin{align*}
	\Gr(\Delta)
	= \setdef{j\cdot \Delta}{\text{$j$ is an integer}}.
\end{align*}
Recall that we want to approximate $g_{-1}$ (resp. $g_{0}$) by a linear combination of Gaussians centered inside $[0,1]$ (resp. $[-2,-1]$).
If we project $g_{-1}$ onto the subspace spanned by the Gaussians  centered at the grid points $[0,1]\cap \Gr(\Delta)$ for a small $\Delta$, the projection by definition is a linear combination of the Gaussians centered at the grid points.
Moreover, it can be proven that $g_{-1}$  is indeed close to the linear combination, i.e.
\begin{align*}
    g_{-1} \approx \Pi_{\mathcal{V}}(g_{-1}) := \text{the projection of $g_{-1}$ onto $\mathcal{V}$} = \sum_{\mu\in [0,1]\cap \Gr(\Delta)}\alpha_\mu g_\mu \numberthis\label{eq:ex1}
\end{align*}
where $\mathcal{V} = \spn\setdef{g_{\mu}}{\mu\in [0,1]\cap\Gr(\Delta)}$ and $\alpha_\mu$ are some coefficients.
The quantity $\Delta$ defines the approximation quality, i.e. the smaller $\Delta$ is the better the approximation is.
By symmetry, we have 
\begin{align*}
    g_0 \approx \Pi_{\mathcal{V}}(g_{0}) := \text{the projection of $g_{0}$ onto $\mathcal{V}'$} = \sum_{\mu\in [-2,-1]\cap \Gr(\Delta)}\alpha_\mu g_\mu\numberthis\label{eq:ex2}
\end{align*}
where $\mathcal{V}' = \spn\setdef{g_\mu}{[-2,-1]\cap\Gr(\Delta)}$.
For example, if we take $\Delta= 0.2$, we have
\begin{align*}
	\Pi_{\mathcal{V}}(g_{-1})
	& =
	80.609g_0 -260.774g_{0.2} + 331.9g_{0.4} -195.489g_{0.6} +   44.741g_{0.8}
	\approx
	g_{-1}.
\end{align*}
See Figure \ref{fig:proj}.
\begin{figure}[!t]
	\begin{center}
		\includegraphics[width=0.4\linewidth]{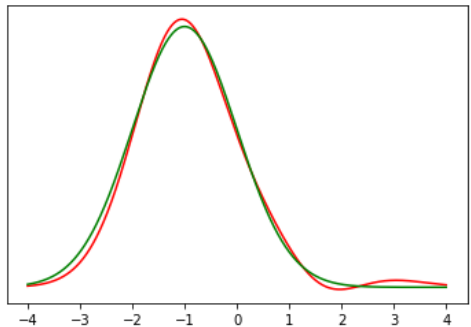}
	\end{center}
	\caption{\label{fig:proj} Graph of $\Pi_{\mathcal{V}}(g_{-1})$ (red) and $g_{-1}$ (green)}
\end{figure}
Indeed, we can prove the following lemma, which is proved in the appendix, to show how good the approximation is. 
\begin{lemma}
	\label{lem:gauss:approx}
	Let $\Pi_{\mathcal{V}}(g_{-1})$ be the projection of $g_{-1}$ onto $\mathcal{V}$ where $\mathcal{V} = \spn\setdef{g_{\mu}}{\mu\in [0,1]\cap\Gr(\Delta)}$.
	For any sufficiently small $\Delta>0$, we have
	\begin{align*}
	    \|g_{-1} - \Pi_{\mathcal{V}}(g_{-1})\|_2 < 2^{-\Omega((1/\Delta)\log(1/\Delta))}.
	\end{align*}
\end{lemma}

The fact that translates of a Gaussian are dense in the set of square-integrable functions dates back to classical results such as Wiener's Tauberian theorem \citep{wiener1932tauberian,wiener1933fourier}, which says that a square-integrable function $g$ can be approximated by linear combinations of translations of another function $f$ so long as the Fourier transform of $f$ does not vanish. 
A more recent result can be found in \cite{calcaterra2008}; a detailed account of Tauberian theory can be found in \citet{korevaar2013tauberian}. 
So it is known that such approximations are possible,
however, existing results stop short of proving explicit approximation rates. 
With this in mind, Lemma~\ref{lem:gauss:approx} is a quantitative version of Wiener's theorem for the special case where both $f$ and  $g$ are Gaussians.\footnote{By suitably mofidying the proof of Lemma~\ref{lem:gauss:approx}, the function $g_{-1}$ may be replaced with any $g_{a}$.}
Determining a fast rate is a crucial step in our proof: For example, if this rate were ``merely'' exponential, the desired superpolynomial lower bound would not follow.

Now, we are ready to construct $f_1,f_2,f_1',f_2'$.
We split \eqref{eq:ex1} and \eqref{eq:ex2} into two parts: the Gaussians with positive coefficients and the Gaussians with negative coefficients.
Note that these two parts are two unnormalized interval Gaussians (the sums of the coefficients are not $1$ in magnitude and indeed are large).
Then, $f_1$ (resp. $f_2$) is defined to be the sum of the Gaussians with normalized positive coefficients in \eqref{eq:ex1} (resp. \eqref{eq:ex2}).
Also, $f_1'$ (resp. $f_2'$) is defined to be the sum of the Gaussians with normalized negative coefficients in \eqref{eq:ex1} (resp. \eqref{eq:ex2}) and $g_0$ (resp. $g_{-1}$) with the weight that balances the total weight in $f_1$ (resp. $f_2$).

In general, as $\Delta\rightarrow 0$, we have the following.
Let $C$ be the sum of the absolute value of the coefficients in the linear combination.
Then, the term $\norm{f_1 - f'_1}_1$ is bounded from below by $\sim 1/C$ which decays slower than the rate of $1/2^{O(1/\Delta)}$.
On the other hand, the term $\norm{f-f'}_1$ is bounded from above by $\sim \|g_{-1} - \Pi_{\mathcal{V}}(g_{-1})\|_2/C$ which decays faster than the rate of $1/2^{\Omega((1/\Delta)\log(1/\Delta))}$.
Another perspective is to view the approximation as the Hermite function expansions.
When $\Delta\to 0$, the subspace $\mathcal{V}$ is indeed the subspace spanned by the Hermite functions.
Moreover, if we express a Gaussian as the Hermite function expansion, the coefficient at the $m$-th term is indeed $\sim \frac{1}{\sqrt{m!}} = 1/2^{\Omega(m\log m)}$.

\subsection{Upper Bound} \label{sec:overview_upper}

For the upper bound, we propose an algorithm that uses sub-exponentially many samples to estimate each component, and thus breaks the exponential complexity of deconvolution and shows that our super-polynomial lower bound is tight.
The basic idea behind our algorithm is to re-formulate the learning problem as a linear system, approximate the terms in this system via sample quantities, and then carefully analyze the resulting approximation error. Although the overall approach is deceptively simple, obtaining precise bounds is highly nontrivial, and represents the main technical hurdle we overcome. 
In this section, we provide a high-level overview of the main steps and the challenges in the analysis.

Recall our problem definition from \eqref{eq:main:model}.
For simplicity, we will hereafter assume that $r_1=0$, $r_2=r$, and the lengths of the intervals are $1$, i.e.
\begin{align*}
    f=w_1f_1 + w_2f_2, \qquad \text{where $f_i\in\mathcal{G}_{I_i}$ and $I_1 = [r_1-\frac{1}{2},r_1+\frac{1}{2}]$, \quad $I_2=[r_2-\frac{1}{2},r_2+\frac{1}{2}]$.}
\end{align*}
The components $f_i$ in the mixture are interval Gaussians which are nonparametric, i.e. there may not be a finite number of parameters to describe each component.
Hence, we must first ask the following fundamental question: 
\emph{How} to appropriately discretize the infinite-dimensional problem \eqref{eq:main:model}?
A natural approach is to express $f_i$ as a sum of orthonormal functions such as the Hermite functions.
Namely, 
\begin{align*}
    f_i = \alpha_{i,0}\psi_{0,r_i} + \alpha_{i,1}\psi_{1,r_i} + \cdots
\end{align*}
where $\psi_{j,r_i}$ is the $j$-th Hermite function centered at $r_i$ and $\alpha_{i,j}=\inner{f_i}{\psi_{j,r_i}}$ for $i=1,2$ and $j\in\mathbb{N}_{0}$.
It turns out that $\alpha_{i,j}$ decays at a fast rate as $j\rightarrow \infty$, so that if we truncate this expansion, we have
\begin{align*}
    f_i \approx \alpha_{i,0}\psi_{0,r_i} + \alpha_{i,1}\psi_{1,r_i} + \cdots + \alpha_{i,\ell-1}\psi_{\ell-1,r_1} \numberthis \label{eq:fi_approx}
\end{align*}
for a sufficiently large integer $\ell$.
If we manage to give a good approximation on each $\alpha_{i,j}$ for $i=1,2$ and $j\in[\ell]$, we will be able to give a good approximation on $f_i$.
Intuitively, our goal now is to learn $2\ell$ parameters which are the first $2\ell$ coefficients in the Hermite function expansion of $f_i$.
Namely, we perform regression of $f_i$ onto the subspace spanned by the first $\ell$ Hermite functions centered at $r_1$ and $r_2$.
Note that $\ell$ is not an absolute constant and is supposed to tend to infinity as the error tends to $0$.
We preview that $\ell$ is indeed $\Theta(\log \frac{1}{\eps})$.

Let $\lambda_{i,j}$ 
be $w_i\alpha_{i,j}$ for $i=1,2$ and $j\in\mathbb{N}_{0}$.
Now, we plug $\eqref{eq:fi_approx}$ into $f$ and we have
\begin{align*}
    f & \approx \lambda_{1,0}\psi_{0,r_1} + \lambda_{1,1}\psi_{1,r_1} + \cdots + \lambda_{1,\ell-1}\psi_{\ell-1,r_1} \\
    & \qquad+ \lambda_{2,0}\psi_{0,r_2} + \lambda_{2,1}\psi_{1,r_2} + \cdots + \lambda_{2,\ell-1}\psi_{\ell-1,r_2}.
\end{align*}
Furthermore, if we project $f$ onto $\psi_{j,r_i}$ for $i=1,2$ and $j\in[\ell]$, we have
\begin{align*}
\begin{aligned}
    \inner{f}{\psi_{j,r_i}} & \approx \lambda_{1,0}\inner{\psi_{0,r_1}}{\psi_{j,r_i}} + \lambda_{1,1}\inner{\psi_{1,r_1}}{\psi_{j,r_i}} + \cdots + \lambda_{1,\ell-1}\inner{\psi_{\ell-1,r_1}}{\psi_{j,r_i}} \\
    & \qquad+ \lambda_{2,0}\inner{\psi_{0,r_2}}{\psi_{j,r_i}} + \lambda_{2,1}\inner{\psi_{1,r_2}}{\psi_{j,r_i}} + \cdots + \lambda_{2,\ell-1}\inner{\psi_{\ell-1,r_2}}{\psi_{j,r_i}}.
\end{aligned}
\end{align*}
This can be written as a system of linear equations
\begin{align*}
    A\lambda \approx y
\end{align*}
where
\begin{align*}
    &\text{$A$ is the $2\ell$-by-$2\ell$ matrix whose entries are given by } \inner{\psi_{j_1,r_{i_1}}}{\psi_{j_2,r_{i_2}}}  \text{ for $i_1,i_2=1,2$; $j_1,j_2\in[\ell]$,} \\
    &\text{$\lambda$ is the $2\ell$-dimensional vector whose entries are given by } w_i\alpha_{i,j} \text{ for $i=1,2$; $j\in[\ell]$ and} \\
	&\text{$y$ is the $2\ell$-dimensional vector whose entries are given by } \inner{f}{\psi_{j,r_i}} \text{ for $i=1,2$ ; $j\in[\ell]$.}
\end{align*}
For example, when $\ell=2$, 
\begin{align*}
    A &= \begin{bmatrix}
        1 & 0 & \inner{\psi_{0,r_2}}{\psi_{0,r_1}} & \inner{\psi_{1,r_2}}{\psi_{0,r_1}}\\
        0 & 1 & \inner{\psi_{0,r_2}}{\psi_{1,r_1}} &  \inner{\psi_{1,r_2}}{\psi_{1,r_1}}\\
        \inner{\psi_{0,r_1}}{\psi_{0,r_2}} & \inner{\psi_{1,r_1}}{\psi_{0,r_2}} & 1 & 0\\
        \inner{\psi_{0,r_1}}{\psi_{1,r_2}}& \inner{\psi_{1,r_1}}{\psi_{1,r_2}} & 0 & 1
    \end{bmatrix}\\
    \lambda & =\begin{bmatrix}
        w_1\alpha_{1,0} & w_1\alpha_{1,1} & w_2\alpha_{2,0} & w_2\alpha_{2,1}
    \end{bmatrix}^\top \\
    y & =\begin{bmatrix}
        \inner{f}{\psi_{0,r_1}} & \inner{f}{\psi_{1,r_1}} & \inner{f}{\psi_{0,r_2}} & \inner{f}{\psi_{1,r_2}}
    \end{bmatrix}^\top 
\end{align*}
One might notice that we still do not know the entries of $y$:
To estimate these values, we will use the samples to first find an approximation $f'$ of $f$ that $\norm{f'-f}_2$ is small.
Once we have $f'$ we can approximate the entries of $y$ by $\inner{f'}{\psi_{j,r_i}}$ for $i=1,2$ and $j\in[\ell]$.

Let $y'$ be the resulting $2\ell$-dimensional vector that approximates $y$ (i.e. by replacing each entry $\inner{f}{\psi_{j,r_i}}$ with $\inner{f'}{\psi_{j,r_i}}$).
Recall that the matrix $A$ defined above is already known (i.e. independent of the data), so we may consider the following system of linear equations in the indeterminate $x$:
\begin{align*}
    Ax =  y'.
\end{align*}
Let $\widehat \lambda$ be the solution of this system of linear equations.
From the above discussion, we expect that the solution of this system, $\widehat \lambda = A^{-1} y'$, is close to $\lambda$.
A useful perspective is to view this as analyzing the condition number of the system $Ax =  y'$.

In the preceding construction, we incurred two sources of error in approximating $ \lambda$ with $\widehat \lambda$:
The error from truncating the Hermite function expansion of $f_i$ (the truncation error) and the error from estimating the terms $\inner{f}{\psi_{j,r_i}}$ (the approximation error).
Formally, we have
\begin{align}
\label{eq:approx:error}
    \widehat \lambda - \lambda  = \underbrace{(A^{-1}( y' - y))}_{:=\mathcal{E}_a} + \underbrace{(A^{-1}(y - A\lambda))}_{:=\mathcal{E}_t}
\end{align}
where $\mathcal{E}_t$ is the truncation error and $\mathcal{E}_a$ is the approximation error.
Namely, we need to show that each entry of $\mathcal{E}_t,\mathcal{E}_a$ is small.
Observe that while the approximation error $\mathcal{E}_a$ depends on both $\ell$ and $n$, the truncation error $\mathcal{E}_t$ is independent of $n$ (i.e. it depends only on $\ell$).

To bound these errors, we observe that these errors have the term $A^{-1}$.
We can argue that $\det(A) \approx 0$ as follows.
Let $u$ be the $2\ell$ dimensional vector
\begin{align*}
    u=\begin{bmatrix}
\inner{\psi_{0,0}}{\psi_{0,r}} & \inner{\psi_{1,0}}{\psi_{0,r}} & \cdots & \inner{\psi_{\ell-1,0}}{\psi_{0,r}} & -1 & 0 & \cdots & 0 \end{bmatrix}^\top.
\end{align*}
By direct calculation, $u^\top A u = 1-\sum_{j=0}^{\ell-1}\inner{\psi_{j,0}}{\psi_{0,r}}^2 = \sum_{j=\ell}^{\infty} e^{-\frac{1}{2}r^2}\frac{1}{j!}(\frac{r^2}{2})^j$ where the last equality is from the fact that $\inner{\psi_{j,0}}{\psi_{0,r}}^2 = e^{-\frac{1}{2}r^2}\frac{1}{j!}(\frac{r^2}{2})^j$.
Hence, $u^\top Au \to \infty$ as $\ell\to \infty$ which further implies the smallest eigenvalue of $A$ tends to $0$.

For the approximation error $\mathcal{E}_a$, one can intuitively think that even though the term $A^{-1}$ may blow up the error for large $\ell$ the error can still be bounded as long as $y'$ and $y$ are close enough to offset the effect of $\det A$ being close to $0$. Crucially, we can make the difference $y' - y$ small by increasing the number of samples without increasing $\ell$. In other words, $\mathcal{E}_a$ can be controlled simply by adding more samples.
On the other hand, the truncation error $\mathcal{E}_t$ is purely dictated by $\ell$ and hence is independent of the number of samples.
When $\ell$ is getting larger, the effect of $\det A$ being close to $0$ makes the analysis challenging and hence nontrivial insights are needed. 

To obtain the desired bound for the truncation error, observe that we can explicitly express \eqref{eq:fi_approx} as 
\begin{align*}
    y - A\lambda = \sum_{j=\ell}^\infty\lambda_{1,j}z_{1,j} + \sum_{j=\ell}^\infty\lambda_{2,j}z_{2,j}
\end{align*}
where 
\begin{align*}
    z_{1,j} &= \begin{bmatrix}0 & \cdots & 0 & \inner{\psi_{j,r_1}}{\psi_{0,r_2}} & \cdots& \inner{\psi_{j,r_1}}{\psi_{\ell-1,r_2}} \end{bmatrix}^\top \\
    z_{2,j} &= \begin{bmatrix}\inner{\psi_{j,r_2}}{\psi_{0,r_1}} & \cdots & \inner{\psi_{j,r_2}}{\psi_{\ell-1,r_1}} & 0 & \cdots & 0 \end{bmatrix}^\top.
\end{align*}
As we mentioned before, $\lambda_{i,j} = w_i\alpha_{i,j}$ decays at a fast rate as $j\rightarrow\infty$.
More precisely, we can show that $\abs{\lambda_{i,j}} \leq w_i\cdot \frac{1}{\sqrt{j!}}\cdot 2^{O(j)}$.
On the other hand, the entries of $A^{-1}z_{i,j}$ are roughly bounded by $\sqrt{j!}\cdot \frac{1}{r^{\Omega(j)}}$.
Hence, we conclude that the entries of $\lambda_{i,j}A^{-1}z_{i,j}$ are bounded by $w_i\cdot\frac{1}{2^{\Omega(j)}}$ for a large constant $r$.
Finally, the entries of $\mathcal{E}_t$ are bounded by $\frac{1}{2^{\Omega(\ell)}}$.
To obtain the desired bound for the approximation error, one can show that the entries of $\mathcal{E}_a$ are roughly bounded by $2^{O(\ell\log \ell)}\norm{f'-f}_2$.

This shows that the entries of $\widehat \lambda - \lambda$ are bounded by $\frac{1}{2^{\Omega(\ell)}} + 2^{O(\ell\log \ell)}\cdot\norm{f'-f}_2$.
Recall that $\lambda_{i,j} = w_i\alpha_{i,j}$ and hence the function 
\begin{align*}
    \widetilde f_i = \sum_{j=0}^{\ell-1} \widehat \lambda_{i,j} \psi_{j,r_i}
\end{align*}
is expected to be close to $w_if_i$ as long as $\ell$ is large enough and $f'$ is close to $f$.
More precisely, by picking $\ell = \Theta(\log \frac{1}{\eps})$ and approximating $f$ such that $\norm{f'-f}_2 \leq \eps^{\Theta(\log\log\frac{1}{\eps})}$, we have
\begin{align*}
    \norm{\widetilde f_i  - w_i f_i}_1 \leq \eps^{\Omega(1)}.
\end{align*}
Note that $\widetilde  f_i$ may not be a positive function since it is just a linear combination of Hermite functions.
It is easy to show that, assuming we are able to calculate an integral accurately, 
\begin{align*}
    \norm{\widehat f_i - f_i}_1 \leq \eps.
\end{align*}
where $\widehat f_i = \frac{(\widetilde f_i)_+}{\norm{(\widetilde f_i)_+}_1}$ and $(\cdot)_+ = \max\{0,\cdot\}$.

To analyze the sample complexity, it boils down to the question of how many samples drawn from the distribution whose pdf is $f$ are required to have a good approximation on $f$.
It is known that, for any $\eps'>0$, we only need to use $\poly(\frac{1}{\eps'})$ samples to return a function $f'$ such that $\norm{f'-f}_2\leq \eps'$.
By substituting $\eps'=\eps^{\Theta(\log\log\frac{1}{\eps})}$, the sample complexity of our algorithm is $(\frac{1}{\eps})^{\Theta(\log\log\frac{1}{\eps})}$.

\section{Conclusion}\label{sec:conclusion}

In this paper, we studied the problem of learning mixture components from a nonparametric mixture model $f = w_1f_1+w_2f_2$ where each component $f_i$ is an interval Gaussian.
Given samples drawn from a nonparametric mixture model of this form, we are interested in the sample complexity for estimating the components $f_i$.
Roughly speaking, our guarantee lies in between parameter learning and density estimation.
We showed that super-polynomially many samples are required to achieve this guarantee.
To the best of our knowledge, no such nontrivial lower bounds have been established previously.
Moreover, we proposed an algorithm that uses sub-exponentially many samples drawn from the nonparametric mixture model to estimate the components.
We can conclude that the optimal sample complexity of this problem properly lies in between polynomial and exponential, which is not common in learning theory.

\bibliography{ref}
\bibliographystyle{abbrvnat}

\appendix

\section{Preliminaries}\label{sec:prelim}

For any function $p:\mathbb{R}\rightarrow \mathbb{R}$ that $\int_{-\infty}^\infty \abs{p(x)} \dir x < \infty$, the $L^1$ norm of $p$ is defined as $\norm{p}_1=\int_{-\infty}^\infty \abs{p(x)} \dir x$.
For any two square-integrable functions $p,q:\mathbb{R}\rightarrow\mathbb{R}$, the inner product of $p,q$ is defined as $\inner{p}{q} = \int_{-\infty}^{\infty} p(x)q(x) \dir x$ and, for any square-integrable function $q$, the $L^2$ norm of $q$ is defined as $\norm{q}_2 = \sqrt{\inner{q}{q}}$.
We use $\mathbb{N}_{0}$ to denote the set of nonnegative integers and, for any $m\in\mathbb{N}_{0}$, $[m]$ to denote the set $\{0,1,\cdots,m-1\}$.
For any $\mu\in\mathbb{R}$, we let $g_{\mu}$ be the pdf of the standard Gaussian distribution, i.e. 
\begin{align*}
	g_\mu(x) &= \frac{1}{\sqrt{2\pi}} e^{-\frac{1}{2}(x-\mu)^2} \qquad\text{for all $x\in\mathbb{R}$.}
\end{align*}

\paragraph{Gram-Schmidt Process}
Recall that the key observation in our lower bound proof is expressing a Gaussian as a linear combination of Gaussians centered at points inside an interval.
When we analyze the relevant convergence rate, we treat these Gaussians as vectors and often need to express the linear combination by the orthonormal basis for the subspace spanned by these vectors.
Hence, Gram-Schmidt process is a process to construct the orthonormal basis for the subspace spanned by a set of given vectors and works as follows.

Suppose $u_0,\cdots,u_{m}$ are $m+1$ linearly independent vectors.
Define the vector $\tilde u_i$ for $i\in [m+1]$ as follows.
\begin{align*}
	\widetilde u_0  = u_0 \qquad \text{and}\qquad \widetilde u_{i} = u_{i} - \sum_{j=0}^{i-1}\frac{\inner{u_{i}}{\widetilde u_j}}{\inner{\widetilde u_j}{\widetilde u_j}}\widetilde u_j \qquad\text{for $i>0$}
\end{align*}
Then, these vectors are orthogonal, i.e. $\inner{\tilde u_i}{\tilde u_j}=0$ for any $i,j\in[m+1]$.
Note that they are not normalized, i.e. $\norm{\widetilde u_i}_2\neq 1$.

\paragraph{Information Theoretic Lower Bounds}
In our lower bound proof, we construct two distributions that the $L^1$-norm of the difference of them is small.
To make the connection with the sample complexity, we will invoke the following standard information theoretic lower bound (e.g. see \cite{canonne2022topics}). 
This is an immediate corollary of the well-known Neyman-Pearson lemma.

\begin{lemma}\label{lem:infothe}
    Let $f$ and $f'$ be two pdfs and $\delta = \norm{f-f'}_1$.
    Suppose we are given $n$ samples drawn from either $f$ or $f'$.
    If $n<C\cdot \frac{1}{\delta}$ where $C$ is an absolute constant, no algorithm taking these $n$ samples as the input can test which pdf the samples are drawn from with probability $1-\frac{1}{100}$. 
\end{lemma}

\paragraph{Hermite Functions}
Let $h_j(x)$ be the (physicist's) Hermite polynomials, i.e.
\begin{align*}
	h_j(x) & = (-1)^je^{x^2}\frac{\dir^j}{\dir \xi^j}e^{-\xi^2}\bigg|_{\xi=x} \qquad\text{for all $x\in\mathbb{R}$}
\end{align*}
and $ \psi_j(x)$ be the (physicist's) Hermite functions, i.e.
\begin{align*}
	 \psi_j(x) & = (-1)^j\frac{1}{\sqrt{2^jj!\sqrt{\pi}}}h_j(x)e^{-\frac{1}{2}x^2}  \qquad\text{for all $x\in\mathbb{R}$.}
\end{align*}
The Hermite functions are orthonormal, i.e. $\norm{ \psi_{i}}_2 = 1$ and $\inner{ \psi_i}{ \psi_j}=0$ for any $i,j\in\mathbb{N}_{0}$.

For any $\mu\in\mathbb{R}$ and any $j\in\mathbb{N}_0$, we use $ \psi_{j,\mu}(\cdot)$ as a shorthand for $ \psi_{j}(\cdot - \mu)$.
Note that $ \psi_{j,0} =  \psi_j$.
It is known that, for any $\mu\in\mathbb{R}$ and any $i,j\in\mathbb{N}_0$, the inner product of $ \psi_{i,0}$ and $ \psi_{j,\mu}$ is 
\begin{align*}
	\MoveEqLeft\inner{ \psi_{i,0}}{ \psi_{j,\mu}} \\
	& = 
	\int_{-\infty}^{\infty}  \psi_{i,0}(x) \psi_{j,\mu}(x)\dir x \\
	& =
	\sum_{k=0}^{\min\{i,j\}}\bigg(e^{-\frac{1}{8}{\mu}^2}(-1)^{i}\sqrt{\frac{k!}{i!}}{i \choose k} (\frac{\mu}{\sqrt{2}})^{i-k}\bigg)\bigg(e^{-\frac{1}{8}{\mu}^2}(-1)^{j}\sqrt{\frac{k!}{j!}}{j\choose k} (\frac{-{\mu}}{\sqrt{2}})^{j-k}\bigg). \numberthis\label{eq:hermite_inner}
\end{align*}
In particular, we have
\begin{align*}
	\inner{ \psi_{j,0}}{g_{\mu}} = \frac{(-1)^j}{\sqrt{2^{j+1}j!\sqrt{\pi}}}e^{-\frac{1}{4}\mu^2} \mu^j \numberthis\label{eq:hermite_inner_g}
\end{align*}
since $g_\mu = \frac{1}{\sqrt{2\pi}}e^{-\frac{1}{2}(x-\mu)^2} = \frac{1}{\sqrt{2\sqrt{\pi}}} \psi_{0,\mu}$.

We will express our density functions as Hermite function expansions and truncate the expansion.
Lemma \ref{lem:hermite_l1} gives the decay rate for  $\norm{ \psi_j}_1$ as $j\rightarrow \infty$.
Hence, it helps controlling the error of the tail expansion.

\begin{lemma}[\cite{aptekarev1995asymptotic,aptekarev2012asymptotics}]\label{lem:hermite_l1}
	For any $j\in\mathbb{N}_{0}$, we have $\norm{ \psi_{j}}_1 = O(j^{1/4})$.
\end{lemma}

\paragraph{Systems of Linear Equations}
We mentioned that our algorithm involves solving a system of linear equations and analyzing its solution.
Cramer's rule provides an explicit formula for the solution of a system of linear equations whenever the system has a unique solution.
\begin{lemma}[Cramer's rule]\label{lem:cramer}
	Consider the following system of $n$ linear equations with $n$ variables.
	\begin{align*}
		Ax=b
	\end{align*}
	where $A$ is a $n$-by-$n$ matrix with nonzero determinant and $b$ is a $n$ dimensional vector.
	Then, the solution of this system $\widehat x = A^{-1}b$ satisfies that the entry of $\widehat x$ indexed at $i\in[n]$ is 
	\begin{align*}
		\det(A^{i\to b}) / \det(A)
	\end{align*}
	where $A^{i\to b}$ is the same matrix as $A$ except that the column indexed at $i$ is $b$.
	
\end{lemma}

\paragraph{Determinants}
In our analysis, we often encounter determinants due to the application of Cramer's rule (Lemma \ref{lem:cramer}).
The Cauchy-Binet formula is a formula for the determinant of a matrix that can be expressed as a product of two matrices of transpose shapes.
Indeed, the entries of matrices in our analysis are often the inner products of two vectors.
\begin{lemma}[Cauchy–Binet formula]\label{lem:cb_formula}
	Let $b_0,b_1,\cdots$ and $c_0,c_1,\cdots$ be two infinite sequences of $n$ dimensional vectors.
	For any (ordered) subset $S=\{s_0<s_1<\cdots<s_{n-1}\}$ of $\mathbb{N}_{0}$, let $B_S$ (resp. $C_S$) be the matrix that the column indexed at $i$ is $b_{s_{i}}$ (resp. $c_{s_{i}}$) for $i\in[n]$.
	The determinant of $A = \sum_{i=0}^{\infty} b_ic_i^\top $ is 
	\begin{align*}
		\det A = \sum_{S\in\mathbb{N}_{0},\abs{S}=n} \big(\det B_S\big) \cdot \big(\det C_S\big)
	\end{align*}
	if the RHS converges.
\end{lemma}

\section{Proof for the Lower Bound}

In this section, we will prove Theorem~\ref{thm:main_lower}.
Recall that we have the following notations.
$g_\mu$ is the standard Gaussian centered at $\mu$.
Given any given interval $I$, recall that $\mathcal{G}_I$ is the set of pdfs $\setdef{f\in \mathcal{P}_{\mathbb{R}}}{f = \int_{\mu\in I}\nu(\mu)g_{\mu}\dir\mu \text{ where $\nu\in\mathcal{P}_I$}}$ which we call each element in this set an interval Gaussian.
From now on, we will consider $L^2$ norm instead of $L^1$ norm for analytical convenience, and resolve this issue later.
Given any $\Delta>0$, recall that $\Gr(\Delta)$ is the set $\setdef{j\cdot \Delta}{\text{$j$ is an integer}}$.
Namely, it is an infinite grid of cell width $\Delta$ on the real line.
Without loss of generality, we consider $\Delta$ where $\frac{1}{\Delta}$ is an integer and denote the number $\frac{1}{\Delta}$ by $m$.
Here, we abuse the notations that we treat $g_\mu$ as a vector and recall that the inner product $\inner{g_{\mu_1}}{g_{\mu_2}} = \int_{-\infty}^\infty g_{\mu_1}(x)g_{\mu_2}(x)\dir x$ for any $\mu_1,\mu_2\in\mathbb{R}$ in the usual way.
By a straightforward calculation, we have $\inner{g_{\mu_1}}{g_{\mu_2}} = \frac{1}{\sqrt{4\pi}}e^{-\frac{1}{4}(\mu_1-\mu_2)^2}$ for any $\mu_1,\mu_2\in\mathbb{R}$.

As we mentioned before, the key observation in our construction is the fact that any Gaussian can be approximated by a linear combination of Gaussians centered at points inside an interval.
The approximation is based on the projection of a Gaussian onto a specific subspace.
For simplicity, we will illustrate how to express $g_{-1}$, the Gaussian centered at $-1$, as a linear combination of the Gaussians centered at points inside $[0,1]$.
It is easy to generalize our argument to the case of arbitrary centers and intervals.

We break the proof down into the following steps.
In Section \ref{sec:construct}, we will explicitly construct two mixtures of two interval Gaussians, $f$ and $f'$.
In Section \ref{sec:analysis}, we will analyze the convergence rate of $\norm{f-f'}_2$, $\norm{f_1-f'_1}_2$ and $\norm{f_2-f'_2}_2$.
Finally, in Section \ref{sec:thm}, we will state and prove our main theorem (Theorem \ref{thm:main_lower}).
We will further defer the detailed calculations to Section \ref{sec:detail}.

\subsection{Construction}\label{sec:construct}

Let $v$ be $g_{-1}$ and $u_{i}$ be $g_{i\cdot \Delta}$ for $i\in[m+1]$.
Let $\mathcal{V}$ be the subspace $\spn\{g_\mu\mid\mu\in \Gr(\Delta)\cap [0,1]\}$.
Let $\Pi_{\mathcal{V}}(v)$ be the projection of $v$ onto the subspace $\mathcal{V}$.
We can express $\Pi_{\mathcal{V}}(v)$ as 
\begin{align*}
	\Pi_{\mathcal{V}}(v) = \sum_{i=0}^{m} \alpha_{i} u_{i}
\end{align*}
for some coefficients $\alpha_i$.
Recall that it is a linear combination and therefore each $\alpha_i$ can be negative and large in magnitude.
We now split this linear combination into two parts: the Gaussians with positive coefficients and the Gaussians with negative coefficients.
Let $J_+ = \setdef{i}{\alpha_i\geq 0}$ and $J_-=\setdef{i}{\alpha_i<0}$.
We rewrite $\Pi_{\mathcal{V}}(v)$ as
\begin{align*}
	\Pi_{\mathcal{V}}(v) = \sum_{i\in J_+} \alpha_{i} u_{i} - \sum_{i\in J_-}(-\alpha_i) u_i.
\end{align*}

By symmetry, let $v'=g_0$ and $u'_i=g_{-1-i\cdot \Delta}$ for $i\in[m+1]$.
Let $\mathcal{V}'$ be the subspace $\spn\setdef{g_\mu}{\mu\in\Gr(\Delta)\cap[-2,-1]}$. 
We rewrite $\Pi_{\mathcal{V}'}(v')$ as
\begin{align*}
	\Pi_{\mathcal{V}'}(v') = \sum_{i\in J_+} \alpha_{i} u'_{i} - \sum_{i\in J_-}(-\alpha_i) u'_i.
\end{align*}

Let 
\begin{align*}
	C_{\Delta,+} = \sum_{i\in J_+} \alpha_i,
	\quad
	C_{\Delta,-} = -\sum_{i\in J_-} \alpha_i,
\end{align*}
i.e. $C_{\Delta,+}$ is the sum of the positive coefficients and $-C_{\Delta,-}$ is the sum of the negative coefficients.
Now, we are ready to construct two mixtures of two interval Gaussians.
They are 
\begin{align*}
	f = \frac{1}{2}\underbrace{\frac{1}{C_{\Delta,+}}\bigg(\sum_{i\in J_+} \alpha_{i} u_{i}\bigg)}_{\in\mathcal{G}_{[0,1]}} + \frac{1}{2}\underbrace{\frac{1}{C_{\Delta,+}}\bigg(\sum_{i\in J_+} \alpha_{i} u'_{i}\bigg)}_{\in\mathcal{G}_{[-2,-1]}}
\end{align*}
and
\begin{align*}
	f' & = \frac{1}{2}\underbrace{\frac{1}{C_{\Delta,+}}\bigg(\sum_{i\in J_-}(-\alpha_i) u_i +(C_{\Delta,+}-C_{\Delta,-}) u_0\bigg)}_{\in\mathcal{G}_{[0,1]}} + \frac{1}{2}\underbrace{\frac{1}{C_{\Delta,+}}\bigg(\sum_{i\in J_-}(-\alpha_i) u'_i +(C_{\Delta,+}-C_{\Delta,-}) u'_0\bigg)}_{\in\mathcal{G}_{[-2,-1]}}
\end{align*}
To ease the notations, we define
\begin{align*}
	f_1 &= \frac{1}{C_{\Delta,+}}\bigg(\sum_{i\in J_+} \alpha_{i} u_{i}\bigg), \qquad
	f'_1 = \frac{1}{C_{\Delta,+}}\bigg(\sum_{i\in J_-}(-\alpha_i) u_i +(C_{\Delta,+}-C_{\Delta,-}) u_0\bigg) \\
	f_2 &= \frac{1}{C_{\Delta,+}}\bigg(\sum_{i\in J_+} \alpha_{i} u'_{i}\bigg), \qquad
	f'_2 = \frac{1}{C_{\Delta,+}}\bigg(\sum_{i\in J_-}(-\alpha_i) u'_i +(C_{\Delta,+}-C_{\Delta,-}) u'_0\bigg)
\end{align*}
and therefore we have
\begin{align*}
	f = \frac{1}{2}f_1 + \frac{1}{2}f_2, \qquad f'=\frac{1}{2}f'_1 + \frac{1}{2}f'_2 \numberthis \label{eqn:hardinstance}
\end{align*}

\subsection{Analysis of the Convergence Rate}\label{sec:analysis}

Recall that our objective is to show that $\norm{f-f'}_2$ is small while $\norm{f_1-f'_1}_2$ and $\norm{f_2-f'_2}_2$ are large.
We will now examine the terms $f_1-f'_1$, $f_2-f'_2$ and $f-f'$.
For the term $f_1-f'_1$,
\begin{align*}
	f_1-f'_1
	& =
	\frac{1}{C_{\Delta,+}} \sum_{i\in J_+} \alpha_{i} u_{i} - \frac{1}{C_{\Delta,+}}\bigg(\sum_{i\in J_-}(-\alpha_i) u_i + (C_{\Delta,+}-C_{\Delta,-})u_0\bigg) \\
	& =
	\frac{1}{C_{\Delta,+}}\Pi_{\mathcal{V}}(v) - \frac{C_{\Delta,+}-C_{\Delta,-}}{C_{\Delta,+}}u_0
\end{align*}
since the definition of $\Pi_{\mathcal{V}}(v)$ is $\sum_{i\in J_+} \alpha_{i} u_{i} - \sum_{i\in J_-}(-\alpha_i) u_i$.
Similarly,
\begin{align*}
	f_2-f'_2
	& =
	\frac{1}{C_{\Delta,+}}\Pi_{\mathcal{V'}}(v')- \frac{C_{\Delta,+}-C_{\Delta,-}}{C_{\Delta,+}}u'_0.
\end{align*}
Now, we analyze $\norm{f_1-f'_1}_2$.
\begin{align*}
	\norm{f_1-f'_1}_2
	& =
	\norm{\frac{1}{C_{\Delta,+}}\Pi_{\mathcal{V}}(v) - \frac{C_{\Delta,+}-C_{\Delta,-}}{C_{\Delta,+}}u_0 }_2 \\
	& =
	\norm{\frac{1}{C_{\Delta,+}}(\Pi_{\mathcal{V}}(v)-v) - \frac{C_{\Delta,+}-C_{\Delta,-}-1}{C_{\Delta,+}}u_0 + \frac{1}{C_{\Delta,+}}(v-u_0)}_2\\
	& \geq
	\frac{1}{C_{\Delta,+}}\norm{v-u_0}_2 - \frac{1}{C_{\Delta,+}}\norm{\Pi_{\mathcal{V}}(v)-v}_2 - \frac{\abs{C_{\Delta,+}-C_{\Delta,-}-1}}{C_{\Delta,+}}\norm{u_0}_2\numberthis\label{eqn:compdiff}
\end{align*}
For the term $f-f'$, we have
\begin{align*}
	f-f'
	& =
	\frac{1}{2C_{\Delta,+}}\Pi_{\mathcal{V}}(v) - \frac{C_{\Delta,+}-C_{\Delta,-}}{2C_{\Delta,+}}u_0 + \frac{1}{2C_{\Delta,+}}\Pi_{\mathcal{V'}}(v')- \frac{C_{\Delta,+}-C_{\Delta,-}}{2C_{\Delta,+}}u'_0\\
	& =
	\frac{1}{2C_{\Delta,+}}\bigg((\Pi_{\mathcal{V}}(v) - u'_0) + (\Pi_{\mathcal{V'}}(v') - u_0) - (C_{\Delta,+}-C_{\Delta,-}-1)(u_0+u'_0)\bigg)
\end{align*}
Since $\Pi_{\mathcal{V}}(v)$ is the projection of $v$ onto the subspace $\mathcal{V}$, we have $\norm{\Pi_{\mathcal{V}}(v)}_2^2+\norm{\Pi_{\mathcal{V}}(v) - v}_2^2 = \norm{v}_2^2$ by Pythagorean theorem.
Let 
\begin{align*}
	\beta_\Delta := \frac{\norm{v-\Pi_{\mathcal{V}}(v)}_2}{\norm{v}_2}.
\end{align*}
The term $\beta_\Delta$ defines how close is $v$ to $\Pi_{\mathcal{V}}(v)$.
By symmetry,  $\norm{v'-\Pi_{\mathcal{V'}}(v')}_2=\beta_\Delta\norm{v'}_2$.
Hence,
\begin{align*}
	\norm{f-f'}_2
	& \leq
	\frac{1}{2C_{\Delta,+}}\bigg(\norm{\Pi_{\mathcal{V}}(v) - u'_0}_2 + \norm{\Pi_{\mathcal{V'}}(v') - u_0} + \abs{C_{\Delta,+}-C_{\Delta,-}-1}\cdot\norm{u_0+u'_0}_2\bigg) \\
	& =
	\frac{1}{2C_{\Delta,+}}\bigg(\beta_\Delta\norm{v}_2 + \beta_\Delta\norm{v'}_2 + \abs{C_{\Delta,+}-C_{\Delta,-}-1}\cdot\norm{u_0+u'_0}_2\bigg) \\
	& =
	O\bigg(\frac{1}{C_{\Delta,+}}\max\{\beta_\Delta,\abs{C_{\Delta,+}-C_{\Delta,-}-1}\}\bigg). \numberthis\label{eq:totaldiff}
\end{align*}
In other words, we reduce the problem of bounding the terms $\norm{f-f'}_2,\norm{f_0-f'_0}_2,\norm{f_1-f'_1}_2$ to the problem of analyzing the terms $\beta_\Delta,C_{\Delta,+},C_{\Delta,-}$.
To analyze the terms $C_{\Delta,+}$, $C_{\Delta,-}$ and $\beta_\Delta$, we need to express these terms more explicitly.
Recall that these terms are related to the coefficients of the linear combination for the projection of $v$ onto the subspace $\mathcal{V}$.
When we project a vector onto a subspace, it is useful to first find out an orthogonal basis for the subspace.
By Gram-Schmidt process, we define the orthogonal basis $\widetilde u_0,\dots,\widetilde u_{m}$ as follows.
\begin{align*}
	\widetilde u_0  = u_0 \qquad \text{and}\qquad \widetilde u_{i} = u_{i} - \sum_{j=0}^{i-1}\frac{\inner{u_{i}}{\widetilde u_j}}{\inner{\widetilde u_j}{\widetilde u_j}}\widetilde u_j \qquad\text{for $i>0$}
\end{align*}
Note that they are not normalized, i.e. $\norm{\widetilde u_i}_2$ may not be $1$.
Another way of expressing $\Pi_{\mathcal{V}}(v)$ is through the orthogonal basis $\widetilde u_0,\dots,\widetilde u_{m}$.
Namely, 
\begin{align*}
	\Pi_{\mathcal{V}}(v) = \sum_{i=0}^{m} \frac{\inner{v}{\widetilde u_i}}{\norm{\widetilde u_i}_2^2} \widetilde u_{i}
\end{align*}
The advantage of this expression is that we can compute the coefficients $\frac{\inner{v}{\widetilde u_i}}{\norm{\widetilde u_i}_2^2}$ explicitly as we will show below. 
Lemma \ref{lem:recur} gives an explicit formula for $\inner{g_a}{\widetilde u_i}$ that depends on $\Delta$ only and it further gives an explicit formula for the coefficients $\frac{\inner{v}{\widetilde u_i}}{\norm{\widetilde u_i}_2^2}$.

\begin{lemma}[Lemma \ref{lem:recur:detail} in Section \ref{sec:detail}]\label{lem:recur}
	For any $a\in\mathbb{R}$, we have
	\begin{align*}
		\inner{g_a}{\widetilde u_i} = \frac{1}{\sqrt{4\pi}}e^{-\frac{1}{4}a^2}e^{-\frac{i}{4}\Delta^2}\prod_{k=1}^i(e^{\frac{1}{2}a\Delta-\frac{(k-1)}{2}\Delta^2}-1)
	\end{align*}
	for $i\in[m+1]$.
	In particular, if we set $a=-1$ we have
	\begin{align*}
		\inner{v}{\widetilde u_i} = \frac{1}{\sqrt{4\pi}}e^{-\frac{1}{4}}e^{-\frac{i}{4}\Delta^2}\prod_{k=1}^i(e^{-\frac{1}{2}\Delta-\frac{(k-1)}{2}\Delta^2}-1)
	\end{align*}
	and if we set $a=i\cdot\Delta$ we have
	\begin{align*}
		\norm{\widetilde u_i}_2^2 = \frac{1}{\sqrt{4\pi}}\prod_{k=1}^i(1-e^{-\frac{k}{2}\Delta^2}).
	\end{align*}
	
\end{lemma}

Now, we are ready to analyze the terms $C_{\Delta,+},C_{\Delta,-}$ and $\beta_\Delta$ explicitly through the orthogonal basis $\widetilde u_i$ since Lemma \ref{lem:recur} gives us an explicit formula in terms of $\Delta$ only.
Lemma \ref{lem:sumofcoeff}, Lemma \ref{lem:beta} and Lemma \ref{lem:diffofcoeff} give us the bounds we need to bound the terms $\norm{f-f'}_2$, $\norm{f_1-f'_1}_2$ and $\norm{f_2-f'_2}_2$.

\begin{lemma}[Lemma \ref{lem:sumofcoeff:detail} in Section \ref{sec:detail}]\label{lem:sumofcoeff}
	For any sufficiently small $\Delta>0$, we have
	\begin{align*}
		C_{\Delta,+}+C_{\Delta,-} \leq 2^{O(1/\Delta)}
	\end{align*}
\end{lemma}

\begin{lemma}[Lemma \ref{lem:beta:detail} in Section \ref{sec:detail}]\label{lem:beta}
	For any sufficiently small $\Delta>0$, we have
	\begin{align*}
		\beta_\Delta
		\leq
		\frac{1}{2^{\Omega((1/\Delta)\log(1/\Delta))}}
	\end{align*}
\end{lemma}
Note that, by the definition of $\beta_\Delta$, Lemma \ref{lem:beta} is equivalent to Lemma \ref{lem:gauss:approx}.
Moreover, we want to analyze how close to $1$ the term $C_{\Delta,+}-C_{\Delta,-}$ is.
\begin{lemma}[Lemma \ref{lem:diffofcoeff:detail} in Section \ref{sec:detail}]\label{lem:diffofcoeff}
	For any sufficiently small $\Delta>0$, we have
	\begin{align*}
		\abs{C_{\Delta,+}-C_{\Delta,-}-1} \leq \frac{1}{2^{\Omega((1/\Delta)\log(1/\Delta))}}.
	\end{align*}
\end{lemma}

We want to show that $\norm{f-f'}_2$ is small while $\norm{f_1-f'_1}_2$ and $\norm{f_2-f'_2}_2$ are large.
As mentioned before, these terms  $\norm{f-f'}_2$, $\norm{f_1-f'_1}_2$ and $\norm{f_2-f'_2}_2$ can be expressed in terms of $C_{\Delta,+}$, $C_{\Delta,-}$ and $\beta_\Delta$.
We have explicitly analyzed $C_{\Delta,+}$, $C_{\Delta,-}$ and $\beta_\Delta$.
Recall that, in \eqref{eqn:compdiff}, we have
\begin{align*}
    \norm{f_1-f'_1}_2
	& \geq
	\frac{1}{C_{\Delta,+}}\norm{v-u_0}_2 - \frac{1}{C_{\Delta,+}}\norm{\Pi_{\mathcal{V}}(v)-v}_2 - \frac{\abs{C_{\Delta,+}-C_{\Delta,-}-1}}{C_{\Delta,+}}\norm{u_0}_2 \numberthis\label{eq:compdiff_2}
\end{align*}
By Lemma \ref{lem:sumofcoeff} and Lemma \ref{lem:diffofcoeff}, we have $C_{\Delta,+} \leq 2^{O(1/\Delta)}$.
	For the first term $\frac{1}{C_{\Delta,+}}\norm{v-u_0}_2$, $\frac{1}{C_{\Delta,+}} \geq 1/2^{O(1/\Delta)}$ and $\norm{v-u_0}_2 = \norm{g_{-1}-g_0}_2 = \Omega(1)$.
	For the second term $\frac{1}{C_{\Delta,+}}\norm{\Pi_{\mathcal{V}}(v)-v}_2$, $\norm{\Pi_{\mathcal{V}}(v)-v}_2 = \beta_\Delta\norm{v}_2 \leq 1/2^{\Omega((1/\Delta)\log(1/\Delta))}$ by Lemma \ref{lem:beta} and hence $\frac{1}{C_{\Delta,+}}\norm{\Pi_{\mathcal{V}}(v)-v}_2 \leq 1/2^{\Omega((1/\Delta)\log(1/\Delta))}$.
	For the third term $\frac{\abs{C_{\Delta,+}-C_{\Delta,-}-1}}{C_{\Delta,+}}\norm{u_0}_2$, $\abs{C_{\Delta,+}-C_{\Delta,-}-1} \leq 1/2^{\Omega((1/\Delta)\log(1/\Delta))}$ by Lemma \ref{lem:diffofcoeff} and hence $\frac{\abs{C_{\Delta,+}-C_{\Delta,-}-1}}{C_{\Delta,+}}\norm{u_0}_2 \leq  1/2^{\Omega((1/\Delta)\log(1/\Delta))}$.
	Plugging them into \eqref{eq:compdiff_2}, we have
	\begin{align*}
		\norm{f_1-f'_1}_2 \geq \frac{1}{2^{O(1/\Delta)}}. \numberthis\label{eq:compdiff_3}
	\end{align*}
Also, in \eqref{eq:totaldiff}, we have
\begin{align*}
    \norm{f-f'}_2
	& \leq
	O\bigg(\frac{1}{C_{\Delta,+}}\max\{\beta_\Delta,\abs{C_{\Delta,+}-C_{\Delta,-}-1}\}\bigg). \numberthis\label{eq:totaldiff_2}
\end{align*}
By plugging Lemma \ref{lem:sumofcoeff}, Lemma \ref{lem:beta} and Lemma \ref{lem:diffofcoeff} into \eqref{eq:totaldiff_2}, we conclude that 
	\begin{align*}
		\norm{f-f'}_2 \leq \frac{1}{2^{\Omega((1/\Delta)\log(1\Delta))}}. \numberthis\label{eq:totaldiff_3}
	\end{align*}

\subsection{Main Theorem}\label{sec:thm}

In our analysis, we have been using $L^2$ norm instead of $L^1$ norm for analytical convenience.
We now resolve this issue in Lemma \ref{lem:loneltwo}.

\begin{lemma}[Lemma \ref{lem:loneltwo:detail} in Section \ref{sec:detail}]\label{lem:loneltwo}
	We have
	\begin{align*}
		\norm{f_1-f'_1}_2\leq O(\sqrt{\norm{f_1-f'_1}_1}) \qquad \text{and} \qquad
		\norm{f - f'}_1 = O(\norm{f - f'}_2^{2/3})
	\end{align*}
\end{lemma}

Theorem \ref{thm:main_lower} is the main theorem to show that estimating components from a mixture of Gaussians requires super-polynomially many samples.
We reduce it to the problem of distinguishing two distributions given a finite number of samples.
The two distributions are $f$ and $f'$ defined as in \eqref{eqn:hardinstance}.
From the previous lemmas, these two pdfs are very close in $L^1$ norm and the components in each corresponding pair are relatively far away in $L^1$ norm.
Combining with Lemma \ref{lem:infothe}, we will prove Theorem \ref{thm:main_lower}.

\begin{theorem}[Restated Theorem \ref{thm:main_lower}]
    Let $\eps>0$ be a sufficiently small error and $I_1,I_2$ be two known disjoint intervals.
    There exists a distribution whose pdf is $f^*=\frac{1}{2}f^*_1+\frac{1}{2}f^*_2$ where $f^*_i\in \mathcal{G}_{i}$ such that no algorithm taking  a set of $n$ i.i.d. samples drawn from $f^*$ as input returns two pdfs $\widehat f_1,\widehat f_2$ such that $\norm{f^*_i - \widehat f_i}_1 < \eps$ with probability at least $1-\frac{1}{100}$ whenever $n<(\frac{1}{\eps})^{C\log\log \frac{1}{\eps}}$ where $C$ is an absolute constant.
\end{theorem}

\begin{proof}
    Take $f^*$ to be $f$ defined in  \eqref{eqn:hardinstance}.
	Suppose there is an algorithm $\mathcal{A}$ that takes $P$ as the input and returns two pdfs $\widehat f_1,\widehat f_2$ such that $\norm{f^*_i - \widehat f_i}_1 \leq \eps$ with probability at least $1-\frac{1}{100}$. 
	We reduce it to the problem of distinguishing $f$ and $f'$ defined in \eqref{eqn:hardinstance}.
	From \eqref{eq:compdiff_3}, \eqref{eq:totaldiff_3} and Lemma \ref{lem:loneltwo}, we have
	\begin{align*}
		\norm{f-f'}_1 \leq \frac{1}{2^{\Omega((1/\Delta)\log(1/\Delta))}} \qquad \text{and}\qquad \norm{f_1-f'_1}_1 \geq \frac{1}{2^{O(1/\Delta)}}
	\end{align*}
	for any sufficiently small $\Delta>0$.
	By choosing $\frac{1}{\Delta}=\Theta(\log\frac{1}{\eps})$, we have
	\begin{align*}
		\norm{f-f'}_1 \leq \eps^{C_1\log \frac{1}{\eps}} \qquad \text{and}\qquad \norm{f_1-f'_1}_1 \geq 3\eps
	\end{align*}
	where $C_1$ is an absolute constant.
	If we are given a set of $n < (\frac{1}{\eps})^{C\log\log\frac{1}{\eps}}$ i.i.d. samples from one of $f$ and $f'$, we can apply the algorithm $\mathcal{A}$ on these samples.
	From the assumption, $\mathcal{A}$ returns two pdfs $\widehat f_1,\widehat f_2$ such that $\norm{f_i - \widehat f_i}_1\leq\eps$ or $\norm{f'_i - \widehat f_i}_1\leq\eps$.
	Since $\norm{f_1-f'_1}_1 \geq 3\eps$, we can use $\widehat f_1$ to determine  which of $f$ and $f'$ the samples are drawn from.
	It implies that we can distinguish $f$ and $f'$ with $(\frac{1}{\eps})^{C\log\log\frac{1}{\eps}}$ samples while $\norm{f-f'}_1 \leq \eps^{C_1\log\log \frac{1}{\eps}}$.
	It contradicts Lemma \ref{lem:infothe}.
\end{proof}

\section{Proof for the Upper Bound}

In this section, we prove Theorem~\ref{thm:main_upper}.
As we mentioned before, our algorithm is to reformulate the problem as solving a system of linear equations with samples.
Then, we carefully analyze the error and obtain the desired bounds.

For simplicity, we let $I_1=[r_1-\frac{1}{2},r_1+\frac{1}{2}]$ and $I_2=[r_2-\frac{1}{2},r_2+\frac{1}{2}]$ for some $r_1,r_2\in\mathbb{R}$.
Suppose we have a distribution whose pdf is
\begin{align*}
	f = w_1 f_1 + w_2 f_2
\end{align*}
where $f_i=\int_{\mu\in I_i} \nu_i(\mu) g_{\mu}(x)\dir \mu\in\mathcal{G}_{I_i}$.
WLOG, we set $r_1=0$ and $r_2=r$.
We assume that $I_1,I_2$ are known and hence $r$ is an absolute constant.
It is easy to extend our result to arbitrary intervals as long as the separation condition is satisfied.

We break the proof down into the following steps.
In Section \ref{sec:formulating}, we will formulate an appropriate class of functions to approximate the components.
In Section \ref{sec:reduction}, we will reduce the problem to the problem of solving a system of linear equations.
In Section \ref{sec:errorbound}, we will analyze the error induced by the approximation.
In Section \ref{sec:full}, we will give the full algorithm and analyze the sample complexity of the algorithm.
We will further defer the detailed calculations to Section \ref{sec:detail}.

\subsection{Formulating the Approximation}\label{sec:formulating}

We first formulate an appropriate approximation to the infinite-dimensional components $f_i$.
Note that $f_i$ are square-integrable functions.
We can expand $f_i$ in terms of the Hermite function basis as
\begin{align*}
	f_i = \sum_{j=0}^\infty \alpha_{i,j} \psi_{j,r_i}\numberthis\label{eq:fi_hermite}
\end{align*}
where $\alpha_{i,j} = \inner{f_i}{ \psi_{i,r_i}}$ for $i=1,2$ and $j\in \mathbb{N}_{0}$.
Define 
\begin{align*}
    \lambda_{i,j}:=w_i\alpha_{i,j}. 
\end{align*}

Let $\ell$ be any nonnegative integer.
Suppose we manage to approximate each $\lambda_{i,j}$ for $i=1,2$ and $j=0,1\cdots,\ell-1$, i.e. we have another $\widetilde f_{i,j} = \sum_{j=0}^{\ell-1}\widehat \lambda_{i,j} \psi_{j,r_i}$ where $\widehat \lambda_{i,j}$ is the approximation of $\lambda_{i,j}$.
We can show that $\widetilde f_i$ is close to the true $w_if_i$.
Note that $\widetilde f_i$ is just a linear combination of Hermite functions which can be negative; we will handle this issue later.

The following lemma quantifies how $\alpha_{i,j}$ decays, and exposes one of the crucial ingredients in our analysis: Indeed, instead of assuming the convolutional model \eqref{eq:npmix}, it is enough to assume that the $\alpha_{i,j}$ decay as below and the analysis goes through.
\begin{lemma}[Lemma \ref{lem:hermite_coeff:detail} in Section \ref{sec:detail}]\label{lem:hermite_coeff}
	For $i=1,2$ and any $j\in\mathbb{N}_{0}$, $\abs{\alpha_{i,j}} \leq O(1)\cdot\frac{1}{\sqrt{j!}(2\sqrt{2})^j}$.
\end{lemma}

The next lemma shows that the quality of the approximation on $\lambda_{i,j}$ implies the quality of the approximation on $w_if_i$ as $\ell\rightarrow\infty$.
\begin{lemma}[Lemma \ref{lem:tail:detail} in Section \ref{sec:detail}]\label{lem:tail}
	Let $\Delta>0$ and $\ell$ be a nonnegative integer that $\ell=\Omega(1)$.
	If $\abs{\lambda_{i,j} - \widehat \lambda_{i,j}} < \Delta$ for all $j\in[\ell]$, then $\norm{w_if_i - \widetilde f_i}_1 = O(\Delta \ell^{5/4} + \frac{w_i}{10^\ell})$.
\end{lemma}

\subsection{Reduction to Solving a System of Linear Equations}\label{sec:reduction}

Previously, in \eqref{eq:fi_hermite}, we expanded 
\begin{align*}
	f_1 = \sum_{j=0}^\infty \alpha_{1,j} \psi_{j,r_1}, \qquad f_2 = \sum_{j=0}^\infty \alpha_{2,j} \psi_{j,r_2}.
\end{align*}
Let $\lambda_{i,j}=w_i\alpha_{i,j}$ for $i=1,2$ and $j\in\mathbb{N}_{0}$.
It follows that
\begin{align*}
	f = \sum_{j=0}^\infty \lambda_{1,j} \psi_{j,r_1} + \sum_{j=0}^\infty \lambda_{2,j} \psi_{j,r_2}
\end{align*}
or, by projecting $f$ onto each $ \psi_{k,r_1}$ for $i=1,2$ and $k\in [\ell]$,
\begin{align*}
	\inner{f}{\psi_{k,r_i}} = \sum_{j=0}^\infty \lambda_{1,j}\inner{\psi_{j,r_1}}{\psi_{k,r_i}}+\sum_{j=0}^\infty \lambda_{2,j}\inner{\psi_{j,r_2}}{\psi_{k,r_i}}.
\end{align*}
Then, we have
\begin{align*}
    y = A \lambda + \sum_{j=\ell}^\infty\lambda_{1,j} z_{1,j} + \sum_{j=\ell}^\infty \lambda_{2,j}z_{2,j}\numberthis \label{eq:hermite_expand_ell}
\end{align*}
where 
\begin{align*}
    &\text{$A$ is the $2\ell$-by-$2\ell$ matrix whose entries are given by } \inner{\psi_{j_1,r_{i_1}}}{\psi_{j_2,r_{i_2}}}  \text{ for $i_1,i_2=1,2$; $j_1,j_2\in[\ell]$,} \\
    &\text{$\lambda$ is the $2\ell$-dimensional vector whose entries are given by } w_i\alpha_{i,j} \text{ for $i=1,2$; $j\in[\ell]$,} \\
	&\text{$y$ is the $2\ell$-dimensional vector whose entries are given by } \inner{f}{\psi_{j,r_i}} \text{ for $i=1,2$ ; $j\in[\ell]$,} \\
	&z_{1,j} = \begin{bmatrix}0 \\ \vdots \\ 0 \\ \inner{\psi_{j,r_1}}{\psi_{0,r_2}} \\ \vdots\\ \inner{\psi_{j,r_1}}{\psi_{\ell-1,r_2}} \end{bmatrix}\qquad \text{and} \qquad
    z_{2,j} = \begin{bmatrix}\inner{\psi_{j,r_2}}{\psi_{0,r_1}} \\ \vdots \\ \inner{\psi_{j,r_2}}{\psi_{\ell-1,r_1}} \\ 0 \\ \vdots \\ 0 \end{bmatrix} \qquad \text{for $j\geq \ell$}.
\end{align*}
Let $f':\mathbb{R}\rightarrow \mathbb{R}$ be any square-integrable function.
The function $f'$ is expected to be the approximation of $f$ from the samples that we will specify later.
Consider the following system of $2\ell$ linear equations with $2\ell$ variables
\begin{align}
\label{eq:linsys}
	A x = y'.
\end{align}
where $y'$ is the $2\ell$-dimensional vector whose entries are given by $\inner{f'}{\psi_{j,r_i}}$ for $i=1,2$ and $j\in[\ell]$.
Let $\widehat{\lambda} = A^{-1}y'$ be the solution of the above system. 
Then, we have
\begin{align*}
    \widehat{\lambda} = A^{-1}y' = A^{-1}y + A^{-1}y^{\Delta} \qquad \text{where $y^{\Delta}:=y'-y$.}
\end{align*}
Note that each entry of $y^{\Delta}$ is $\inner{\Delta f}{\psi_{k,0}}$ for $k\in\mathbb{N}_0$ where 
\begin{align*}
    \Delta f :=f'-f.
\end{align*}
Plugging \eqref{eq:hermite_expand_ell} into the above equation,
\begin{align*}
	\widehat \lambda
	& =
	A^{-1}\underbrace{\bigg(A \lambda + \sum_{j=\ell}^\infty\lambda_{1,j} z_{1,j} + \sum_{j=\ell}^\infty \lambda_{2,j}z_{2,j} \bigg)}_{\text{from \eqref{eq:hermite_expand_ell}}} + A^{-1}y^{\Delta} \\
	& =
	\lambda+\sum_{j=\ell}^\infty \lambda_{1,j}A^{-1}z_{1,j}+\sum_{j=\ell}^\infty \lambda_{2,j}A^{-1}z_{2,j} + A^{-1}y^{\Delta} \numberthis\label{eq:lambda_error}
\end{align*}
Let
\begin{align*}
	\mathcal{E}_t & = \sum_{j=\ell}^\infty \lambda_{1,j}A^{-1}z_{1,j}+\sum_{j=\ell}^\infty \lambda_{2,j}A^{-1}z_{2,j}, \\
	\mathcal{E}_a & = A^{-1}y^{\Delta}.
\end{align*}
These are the truncation error $\mathcal{E}_t$ and the approximation error $\mathcal{E}_a$ introduced in \eqref{eq:approx:error}.
In other words, \eqref{eq:lambda_error} can be rewritten as
\begin{align*}
	\widehat \lambda - \lambda = \mathcal{E}_t + \mathcal{E}_a
\end{align*}
and we need to bound the entries of $\mathcal{E}_t$ and $\mathcal{E}_a$ to invoke Lemma \ref{lem:tail}.
Note that $\widehat \lambda, \lambda, \mathcal{E}_t, \mathcal{E}_a$ have $2\ell$ entries and each one corresponds to a coefficient of $ \psi_{k,r_i}$ for $i=1,2$ and $k\in[\ell]$.

\subsection{Bounding the Error}\label{sec:errorbound}

In this section, we will bound the truncation error $\mathcal{E}_t$ and the approximation error $\mathcal{E}_a$.

We  first bound the entries of $\mathcal{E}_t$.
Let $\mathcal{E}_{t,i,j}$ be the vector $A^{-1}z_{1,j}$ for $i=1,2$ and $j\geq \ell$.
Namely, we have
\begin{align*}
	\mathcal{E}_t = \sum_{j=\ell}^\infty\lambda_{1,j}\mathcal{E}_{t,1,j} + \sum_{j=\ell}^\infty\lambda_{2,j}\mathcal{E}_{t,2,j}.
\end{align*}
We only need to analyze the first sum $\sum_{j=\ell}^\infty\lambda_{1,j}\mathcal{E}_{t,1,j}$ and by symmetry we can conclude a similar bound for the second sum $\sum_{j=\ell}^\infty\lambda_{2,j}\mathcal{E}_{t,2,j}$.
By Cramer's rule (Lemma \ref{lem:cramer}), the entry of $ \mathcal{E}_{t,1,j} = A^{-1}z_{1,j}$ indexed at $(i,k)$ is given by 
\begin{align*}
    \frac{\det(A^{(i,k)\to j})}{\det(A)}
\end{align*}
where $A^{(i,k)\to j}$ is the $2\ell$-by-$2\ell$ matrix same as $A$ except that the column indexed at $(i,k)$ is replaced with $z_{1,j}$ for $i=1,2$, $k\in[\ell]$ and $j\geq \ell$.
Lemma \ref{lem:det_t_numor} gives a bound on $\abs{\det\big(A^{(i,k)\to j}\big)}$ when comparing to $\det(A)$.
\begin{lemma}[Lemma \ref{lem:det_t_numor:detail} in Section \ref{sec:detail}]\label{lem:det_t_numor}
    For any $\ell\in\mathbb{N}_0$, $k\in[\ell]$ and $j\geq \ell$, we have
    \begin{align*}
        \abs{\det\big(A^{(1,k)\to j}\big)} & \leq 	\bigg(e^{\frac{1}{4}r^2}\cdot 4^{\ell}\cdot \bigg(\frac{1}{\sqrt{j!}}(\frac{r}{2\sqrt{2}})^j\bigg)^{-1}\bigg)\det(A) \qquad \text{and}\\
        \abs{\det\big(A^{(2,k)\to j}\big)} & \leq \bigg(e^{\frac{5}{4}r^2}\cdot4^\ell \cdot\bigg(\frac{1}{\sqrt{j!}}(\frac{r}{2\sqrt{2}})^j\bigg)^{-1}\bigg)\det(A).
    \end{align*}
    Hence, the absolute values of the entries of $\mathcal{E}_{t,1,j}$ are bounded by $O\big(4^\ell \cdot\big(\frac{1}{\sqrt{j!}}(\frac{r}{2\sqrt{2}})^j\big)^{-1}\big)$, i.e.
    \begin{align*}
        \abs{(\mathcal{E}_{t,1,j})_{i,k}} \leq C\cdot 4^\ell \cdot\big(\frac{1}{\sqrt{j!}}(\frac{r}{2\sqrt{2}})^j\big)^{-1} \qquad \text{for $i=1,2$ and $k\in[\ell]$.}
    \end{align*}
    Here, $C$ is an absolute constant.
\end{lemma}
The key idea is to express $A^{(i,k)\to j}$ as a product of two "infinite-dimensional" matrices of transpose shapes and apply Cauchy-Binet formula (Lemma \ref{lem:cb_formula}).

Finally, we combine Lemma \ref{lem:det_t_numor} and Lemma \ref{lem:hermite_coeff} to prove that the absolute value of each entry of $\mathcal{E}_t$ is small.
By Lemma~\ref{lem:det_t_numor}, each entry of $\mathcal{E}_{t,1,j}$ is bounded by 
\begin{align*}
	O\big(4^\ell \cdot\big(\frac{1}{\sqrt{j!}}(\frac{r}{2\sqrt{2}})^j\big)^{-1}\big)
\end{align*}
Also, by Lemma~\ref{lem:hermite_coeff}, each entry of $\lambda_{1,j}\mathcal{E}_{t,1,j}$ is bounded by 
\begin{align*}
	O\big(\frac{w_1}{\sqrt{j!}(2\sqrt{2})^j}\big) \cdot O\big(4^\ell \cdot\big(\frac{1}{\sqrt{j!}}(\frac{r}{2\sqrt{2}})^j\big)^{-1}\big)
	& =
	O\big(w_1\cdot 4^\ell\cdot \frac{1}{r^j}\big).
\end{align*}
Hence, each entry of $\sum_{j=\ell}^\infty\lambda_{1,j}\mathcal{E}_{t,1,j}$ is bounded by $O\big(w_1\cdot(\frac{4}{r})^{\ell}\big)$.
By the similar argument in Lemma~\ref{lem:det_t_numor}, the absolute value of each entry of $\sum_{j=\ell}^\infty\lambda_{2,j}\mathcal{E}_{t,2,j}$ is bounded by $O\big(w_2\cdot(\frac{4}{r})^{\ell}\big)$.
Hence, the absolute value of each entry of $\mathcal{E}_t$ is bounded by $O\big((\frac{4}{r})^{\ell}\big)$, i.e.
\begin{align*}
    \abs{(\mathcal{E}_t)_{i,k}} \leq C\cdot(\frac{4}{r})^{\ell} \qquad \text{for $i=1,2$ and $k\in[\ell]$.} \numberthis\label{eq:truncate_error}
\end{align*}
Here, $C$ is an absolute constant.

Now, we will analyze $\mathcal{E}_a$.
Recall that 
\begin{align*}
	\mathcal{E}_a = A^{-1}y^{\Delta}.
\end{align*}
By Cramer's rule (Lemma~\ref{lem:cramer}), the entry of $\mathcal{E}_a$ indexed at $(i,k)$ is 
\begin{align*}
	\frac{\det\big(A^{(i,k)\to \Delta}\big)}{\det(A) }
\end{align*}
where $A^{(i,k)\to \Delta}$ is the $2\ell$-by-$2\ell$ matrix same as $A$ except that the column indexed at $(i,k)$ is replaced with $y^{\Delta}$ for $i=1,2$ and $k\in[\ell]$.
Lemma \ref{lem:det_a_numor} gives a bound on $\abs{\det\big(A^{(i,k)\to \Delta}\big)}$ when comparing to $\det(A)$.

\begin{lemma}[Lemma \ref{lem:det_a_numor:detail} in Section \ref{sec:detail}]\label{lem:det_a_numor}
    For any $\ell\in\mathbb{N}_0$, $i=1,2$ and $k\in[\ell]$, we have
    \begin{align*}
        \abs{\det\big(A^{(i,k)\to \Delta}\big)} < 2^{O(\ell\log \ell)} \cdot \norm{\Delta f}_2\cdot\det(A).
    \end{align*}
\end{lemma}
Similar to the proof of Lemma \ref{lem:det_t_numor}, the key idea is to express $A^{(i,k)\to \Delta}$ as a product of two "infinite-dimensional" matrices of transpose shapes and apply Cauchy-Binet formula (Lemma \ref{lem:cb_formula}).

Hence, the absolute values of the entries of $\mathcal{E}_a$ are bounded by $2^{O(\ell\log \ell)} \cdot \norm{\Delta f}_2$, i.e.
\begin{align*}
    \abs{(\mathcal{E}_a)_{i,k}} \leq 2^{C\cdot \ell\log \ell} \cdot \norm{\Delta f}_2 \qquad \text{for $i=1,2$ and $k\in[\ell]$.} \numberthis\label{eq:approx_error}
\end{align*}
Here, $C$ is an absolute constant.
With Lemma \ref{lem:det_a_numor}, we conclude that the entries of $\mathcal{E}_a$ are small as long as we have a good approximation on $f$, i.e. $\norm{\Delta f}_2$ is small.

\subsection{Full Algorithm}\label{sec:full}

Recall that the central idea of our algorithm is to estimate the coefficients in the Hermite function expansion for each component.
In Section \ref{sec:formulating}, we proved that we only need to consider the first $\ell$ coefficients in the Hermite function expansion for each component for sufficiently large $\ell$.
In Section \ref{sec:reduction}, we reduced the problem of estimating the coefficients to the problem of solving a system of linear equations.
We are now ready to describe our full algorithm.
Algorithm \ref{alg:main} is an algorithm that takes samples drawn from a mixture of two interval Gaussians as an input and returns two pdfs with the desired guarantee.

\begin{algorithm}[t]
	\caption{Estimating components}\label{alg:main}
	{\bf Input:} A set of samples $P$ drawn from $f$, an nonnegative integer $\ell$
	\begin{algorithmic}[1]
		\STATE Construct $f'$ from the samples $P$ such that $\norm{f'-f}_2 = \norm{\Delta f}_2 < \poly(n^{-1})$ where $n$ is the size of $P$ \label{line:first}
		\STATE Solve the system of linear equations \eqref{eq:linsys}
		\STATE Let $\widehat \lambda$ be the solution and $\widetilde f_i$ be the function $\sum_{j=0}^{\ell-1}\widehat\lambda_{i,j} \psi_{j,r_i}$ for $i=1,2$
		\STATE For any function $q:\mathbb{R}\rightarrow\mathbb{R}$, let $(q(x))_+$ be the function $\max\{q(x),0\}$ for all $x\in\mathbb{R}$
	\end{algorithmic}
	{\bf Output:} $\widehat f_i = (\widetilde f_i)_+/\norm{(\widetilde f_i)_+}_1$ for $i=1,2$
\end{algorithm}
The approximation $f'$ in line \ref{line:first} of Algorithm \ref{alg:main} can be constructed in many ways, e.g. a standard kernel density estimator is sufficient.

\begin{theorem}[Restated Theorem \ref{thm:main_upper}]
	Let $\eps>0$ be  a sufficiently small error and $I_1,I_2$ be two known disjoint intervals of length $1$ such that $r>4$ where $r$ is the distance between the centers of the intervals.
    There exists an algorithm such that, for any distribution whose pdf is $f=w_1f_1+w_2f_2$ where $f_i\in \mathcal{G}_{i}$, $w_i=\Omega(\eps)$ and $w_1+w_2=1$, the algorithm taking a set of $n$ i.i.d. samples from $f$ as input returns two pdfs $\widehat f_1,\widehat f_2$ such that $\norm{f_i - \widehat f_i}_1 < \eps$ with probability at least $1-\frac{1}{100}$ whenever $n>(\frac{1}{\eps})^{C\log\log\frac{1}{\eps}}$ where $C$ is an absolute constant.
\end{theorem}

\begin{proof}
	Set $\ell=\Theta(\log \frac{1}{\eps})$.
	We will show that the pdfs $\widehat f_1, \widehat f_2$ outputted by Algorithm \ref{alg:main} that takes $P$ and $\ell$ as the input satisfy the guarantees.

	By \eqref{eq:truncate_error}, \eqref{eq:approx_error} and the assumption of $r>4$, we have 
	\begin{align*}
		\abs{\widehat \lambda_{i,j}-\lambda_{i,j}} 
		& < 
		O((\frac{4}{r})^{\ell}) + \norm{\Delta f}_2 \cdot 2^{O(\ell \log \ell)}
		<
		\frac{1}{2^{\Omega(\ell)}} + \norm{\Delta f}_2 \cdot 2^{O(\ell \log \ell)}
	\end{align*}
	for $i=1,2$ and $j\in[\ell]$ and, by Lemma \ref{lem:tail}, we further have
	\begin{align*}
		\norm{ \tilde f_i - w_if_i }_1 
		<
		\norm{\Delta f}_2 \cdot 2^{O(\ell \log \ell)} + \frac{1}{2^{\Omega(\ell)}}.
	\end{align*}
	By plugging $\ell = \Theta(\log \frac{1}{\eps})$ and $\norm{\Delta f}_2 < \poly(n^{-1})=  \eps^{\Theta(\log\log \frac{1}{\eps})}$ when $n>(\frac{1}{\eps})^{\Theta(\log\log\frac{1}{\eps})}$, we have
	\begin{align*}
	   \norm{ \tilde f_i - w_if_i }_1  <\eps^{\Omega(1)}.
	\end{align*}
	Since $f_i$ is a pdf which implies it is a positive function, we have
	\begin{align*}
		\norm{ (\tilde f_i)_+ - w_if_i }_1 
		& < 
		\eps^{\Omega(1)}.
	\end{align*}
	Therefore, by the assumption $w_i=\Omega(\eps)$, we conclude that
	\begin{align*}
		\norm{\widehat f_i - f_i }_1
		& =
		\norm{\frac{(\tilde f_i)_+}{\norm{(\tilde f_i)_+}_1} - \frac{(\tilde f_i)_+}{w_i} + \frac{(\tilde f_i)_+}{w_i} - f_i}_1 \\
		& \leq
		\norm{\frac{(\tilde f_i)_+}{\norm{(\tilde f_i)_+}_1}-\frac{(\tilde f_i)_+}{w_i}}_1 + \norm{\frac{(\tilde f_i)_+}{w_i} - f_i}_1\\
		& <
		\eps.
	\end{align*}
	
\end{proof}

\section{Extension to Unknown Intervals} \label{sec:extension}

We have assumed that the intervals in \eqref{eq:main:model} are known.
However, these intervals may not be known in applications.
In this section, we show how we can remove this assumption by slightly strengthening the separation condition.
Suppose we have a pdf $f$ that satisfies the following condition:
\begin{itemize}
    \item There exist two (unknown) intervals $I_1=[r_1-s,r_1+s],I_2=[r_2-s,r_2+s]$ for some $r_1,r_2,s>0$ such that 
    \begin{align}
    \label{eq:int:sep}
        r > 12\max\big\{2s+\sqrt{2\log(\tfrac{4}{0.4\sqrt{2\pi}}\tfrac{1}{w_{\min}})},4s+\sqrt{2\log(\tfrac{4}{0.4\sqrt{2\pi}})}\big\} + 16s
    \end{align}  
    where $r$ is the distance between the centers of the intervals and $f$ can be written as $w_1f_1+w_2f_2$ for some $f_i\in\mathcal{G}_{i}$ and $w_i>0$ that $w_1+w_2=1$.
\end{itemize}
Thus, when the intervals are known as in Theorem~\ref{thm:main_upper}, the distance $r$ can be made independent of $w_{\min}$, whereas when they are unknown the distance $r$ depends on $w_{\min}$ as in \eqref{eq:int:sep}. Note that under the assumptions of Theorem~\ref{thm:main_upper}, we have $s=1/2$.

\begin{algorithm}[t]
	\caption{Finding approximate intervals}\label{alg:unknown}
	{\bf Input:} A set of samples $P$ drawn from $f$
	\begin{algorithmic}[1]
	    \STATE Partition $P$ into $m=O(\log \frac{r}{s_{\min}})$ sets evenly, $P_0,\cdots,P_{m}$
		\FOR{$j=0,1,\cdots,m$}
		    \STATE Let $s' = 2^j\cdot s_{\min}$ and $\Gr_{s'}$ be the set $\setdef{j\cdot s' }{\text{$j$ is an integer}}$
            \STATE Let $N_{y}$ be the number of samples that lie inside the interval $[y-t,y+t]$ where 
            \begin{align*}
                t=\max\big\{s'+\sqrt{2\log(\tfrac{4}{0.4\sqrt{2\pi}}\tfrac{1}{w_{\min}})},2s'+\sqrt{2\log(\tfrac{4}{0.4\sqrt{2\pi}})}\big\}
            \end{align*}
            \STATE Construct the point set $Q = \setdef{y\in\Gr_{s'}}{N_{y}>0.5w_{\min}n'}$ where $n'=\abs{P_j}$ 
            \STATE Check if $Q_j$ can be partitioned into two sets $Q_{j,1},Q_{j,2}$ such that 
            \begin{align*}
                \abs{p-q} &< 4t \qquad \text{when $p,q$ are in the same $Q_{j,i}$}\\
                \abs{p-q} &> 4t \qquad \text{when $p,q$ are in different $Q_{j,i}$}.
            \end{align*}
            \STATE Break the loop if $j$ is the largest integer satisfying the above condition and let $j'$ be this integer 
		\ENDFOR
		\STATE Construct two intervals $I'_i$ for $i=1,2$ such that $I'_i = [p_i-s',q_i+s']$ where $p_i,q_i$ are the farthest two points in $Q_{j',i}$
	\end{algorithmic}
	{\bf Output:} Two intervals $I'_1,I'_2$.
\end{algorithm}

Assume that a lower bound on $s$ is known.
Let the lower bound be $s_{\min}$. Algorithm~\ref{alg:unknown} presents an algorithm for approximating the unknown intervals $I_i$ that comes with the following guarantee:
\begin{restatable}{lemma}{lemunknown}\label{lem:unknown}
    Given a set of samples drawn from a distribution whose pdf is $f$ described above satisfying \eqref{eq:int:sep}.
    Then as long as 
    $n=\Omega(\frac{1}{w_{\min}}\log\frac{r}{s_{\min}}\log\log\frac{r}{s_{\min}})$, Algorithm~\ref{alg:unknown} returns two intervals $I'_1,I'_2$ such that $I_i\subset I'_i$ and the distance between the centers of $I'_1,I'_2$ is larger than $4L$ where $L$ is the maximum length of $I'_1,I'_2$.
\end{restatable}

\begin{proof}
    We will prove that Algorithm \ref{alg:unknown} returns two intervals satisfying the desired properties.
    We first give a useful inequality.
	Note that, for any $\alpha>1$, we have
	\begin{align*}
		\int_{\alpha}^{\infty} \frac{1}{\sqrt{2\pi}} e^{-\frac{1}{2}x^2}\dir x 
		& <
		\int_{\alpha}^{\infty} \frac{1}{\sqrt{2\pi}} e^{-\frac{\alpha}{2}x}\dir x 
		=
		\frac{2}{\alpha\sqrt{2\pi}} e^{-\frac{1}{2}\alpha^2}
		\leq
		\frac{2}{\sqrt{2\pi}} e^{-\frac{1}{2}\alpha^2}.
	\end{align*}
    
    For any $s'>0$ and large $t>0$, we first show that
    \begin{align*}
         \int_{r'_i-t}^{r'_i+t} f(x) \dir x & > w_{\min}\big( 1-\frac{4}{\sqrt{2\pi}} e^{-\frac{1}{2}(t-s-s')^2}\big) \quad \text{for $i=1,2$}\\
         1-\sum_{i=1}^2w_i\int_{r_i-t}^{r_i+t} f_i(x) \dir x  & <\frac{4}{\sqrt{2\pi}} e^{-\frac{1}{2}(t-s)^2}
    \end{align*}
    where $r'_i\in\Gr(s')$ is the closest point to $r_i$.
    The numbers $s',t$ are expected to be the numbers described in Algorithm \ref{alg:unknown}.
    
    We first give a bound for the first inequality.
    \begin{align*}
        \int_{r'_i-t}^{r'_i+t} f(x) \dir x
        & \geq 
        \int_{r'_i-t}^{r'_i+t} w_if_i(x) \dir x \\
        & =
        \int_{r'_i-t}^{r'_i+t} w_i\big(\int_{r_i-s}^{r_i+s} \nu_i(\mu)g_{\mu}(x)\dir \mu\big) \dir x \\
        & =
        w_i\int_{r_i-s}^{r_i+s} \nu_i(\mu) \bigg( 1-\int_{-\infty}^{r'_i-t} g_{\mu}(x) \dir x  - \int_{r'_i+t}^\infty g_{\mu}(x) \dir x\bigg)   \dir \mu
    \end{align*}
    Since $\abs{x-r'_i} > t$, $\abs{r'_i-r_i} < s'$ and $\abs{r_i-\mu}< s$, we have $\abs{x-\mu}>t-s-s'$ and it implies
    \begin{align*}
        \int_{r'_i+t}^\infty g_{\mu}(x) \dir x
        & \leq
        \int_{t-s-\Delta}^\infty g_{0}(x) \dir x
        \leq
        \frac{2}{\sqrt{2\pi}} e^{-\frac{1}{2}(t-s-s')^2}.
    \end{align*}
    Hence, we have
    \begin{align*}
        \int_{r'_i-t}^{r'_i+t} f(x) \dir x
        & \geq
        w_i\int_{r_i-s}^{r_i+s} \nu_i(\mu) \big( 1-\frac{4}{\sqrt{2\pi}} e^{-\frac{1}{2}(t-s-s')^2}\big)   \dir \mu
        =
        w_{i}\big( 1-\frac{4}{\sqrt{2\pi}} e^{-\frac{1}{2}(t-s-s')^2}\big).
    \end{align*}
    
    We now give a bound for the second inequality.
    By the similar argument, we have
    \begin{align*}
        \int_{r_i-t}^{r_i+t} f(x) \geq w_{i}\big( 1-\frac{4}{\sqrt{2\pi}} e^{-\frac{1}{2}(t-s)^2}\big)
    \end{align*}
    and hence
    \begin{align*}
        1-\sum_{i=1}^2\int_{r_i-t}^{r_i+t} f(x) \dir x & <\frac{4}{\sqrt{2\pi}} e^{-\frac{1}{2}(t-s)^2}.
    \end{align*}
    
    In particular, when $s'> s$ we have the following.
    By using 
    \begin{align*}
        t=\max\big\{s'+\sqrt{2\log(\tfrac{4}{0.4\sqrt{2\pi}}\tfrac{1}{w_{\min}})},2s'+\sqrt{2\log(\tfrac{4}{0.4\sqrt{2\pi}})}\big\} \qquad \text{as described in Algorithm \ref{alg:unknown},}
    \end{align*} 
    we have
    \begin{align*}
        \int_{r'_i-t}^{r'_i+t} f(x) \dir x & > 0.6w_{\min} \quad \text{for $i=1,2$ and }\quad
        1-\sum_{i=1}^2w_i\int_{r_i-t}^{r_i+t} f_i(x) \dir x  < 0.4w_{\min}
    \end{align*}
    It is easy to see that the expectation of $N_{r'_i}/n'$ is $\int_{r'_i-t}^{r'_i+t} f(x) \dir x$ and the expectation of $N'/n'$ is bounded from above by $1-\sum_{i=1}^2w_i\int_{r_i-t}^{r_i+t} f_i(x)$ where $N'$ is the number of samples that do not lie in $[r_1-t,r_1+t]\cup[r_2-t,r_2+t]$ and $n$ is the total number of samples.
    By Chernoff bound, we have
    \begin{align*}
        N_{r'_i}/n > 0.6w_{\min} \quad \text{for $i=1,2$ and }\quad N'/n<0.4w_{\min}
    \end{align*}
    with probability $1-\delta$ when $n=\Omega(\frac{1}{w_{\min}}\log \frac{1}{\delta})$.

    In Algorithm \ref{alg:unknown}, it enumerates $s'=s_{\min},2s_{\min},\cdots,2^j\cdot s_{\min},\cdots$ and there exists a $j_0$ such that $s\leq 2^{j_0}\cdot s_{\min} \leq 2s$.
    Set $s'=2^{j_0}\cdot s_{\min}$.
    It implies that, for $y\in\Gr_{s'}$ that is at least $2t$ away from all $r_i$, $y$ is not in $Q = \setdef{y\in\Gr_{s'}}{N_{y}>0.5w_{\min}n}$ as described in Algorithm \ref{alg:unknown}.
    Recall that 
    \begin{align*}
        r > 12\max\big\{2s+\sqrt{2\log(\tfrac{4}{0.4\sqrt{2\pi}}\tfrac{1}{w_{\min}})},4s+\sqrt{2\log(\tfrac{4}{0.4\sqrt{2\pi}})}\big\} + 16s.
    \end{align*}
    Hence, we have 
    \begin{align*}
        r 
        & > 
        12\max\big\{2s+\sqrt{2\log(\tfrac{4}{0.4\sqrt{2\pi}}\tfrac{1}{w_{\min}})},4s+\sqrt{2\log(\tfrac{4}{0.4\sqrt{2\pi}})}\big\} + 16s \\
        & >
        12\max\big\{s'+\sqrt{2\log(\tfrac{4}{0.4\sqrt{2\pi}}\tfrac{1}{w_{\min}})},2s'+\sqrt{2\log(\tfrac{4}{0.4\sqrt{2\pi}})}\big\} + 8s'. \\
        & =
        12t+8s'
    \end{align*}
    which implies $r>8t$ and $r-4t>4(2t+2s')$.
    From $r>8t$, it means $Q_{j_0}$ can be partitioned into two sets $Q_{j_0,1},Q_{j_0,2}$ such that
    \begin{align*}
        \abs{p-q} &< 4t \qquad \text{when $p,q$ are in the same $Q_{j_0,i}$}\\
        \abs{p-q} &> 4t \qquad \text{when $p,q$ are in different $Q_{j_0,i}$}.
    \end{align*}
    Moreover, from $r-4t>4(2t+2s')$, the distance between the centers of $[p_1-s',q_1+s'],[p_2-s',q_2+s']$ where $p_i,q_i$ are the farthest points in $Q_{j_0,i}$ is larger than $r-4t> 4(2t+2s')$.
    Note that the lengths of $[p_1-s',q_1+s'],[p_2-s',q_2+s']$ are bounded by $2t+2s'$.
    Clearly, we also have $I_i\subset [p_i-s',q_i+s']$
    
    It is easy to see that, when $j$ is too large (say $2^js_{\min}>\Omega(r)$), $Q_j$ cannot be partitioned into two sets with the desired properties.
    Therefore, by union bound, we need $\Omega(\frac{1}{w_{\min}}\log\frac{r}{s_{\min}}\log\log\frac{r}{s_{\min}})$ samples to find the approximate intervals with probability at least $1-\frac{1}{100}$.
\end{proof}

With the extra assumption on the separation condition, we first run Algorithm \ref{alg:unknown} to find the approximate support intervals and then run Algorithm \ref{alg:main} to return two pdfs without assuming the intervals are known.
No attempt has been made to optimize the separation condition \eqref{eq:int:sep}, and we leave it to future work to optimize this lower bound.

\section{Proofs for the Lemmas} \label{sec:detail}

\begin{lemma}[Restated Lemma \ref{lem:recur}]\label{lem:recur:detail}
	For any $a\in\mathbb{R}$, we have
	\begin{align*}
		\inner{g_a}{\widetilde u_i} = \frac{1}{\sqrt{4\pi}}e^{-\frac{1}{4}a^2}e^{-\frac{i}{4}\Delta^2}\prod_{k=1}^i(e^{\frac{1}{2}a\Delta-\frac{(k-1)}{2}\Delta^2}-1)
	\end{align*}
	for $i\in[m+1]$.
	In particular, if we set $a=-1$ we have
	\begin{align*}
		\inner{v}{\widetilde u_i} = \frac{1}{\sqrt{4\pi}}e^{-\frac{1}{4}}e^{-\frac{i}{4}\Delta^2}\prod_{k=1}^i(e^{-\frac{1}{2}\Delta-\frac{(k-1)}{2}\Delta^2}-1)
	\end{align*}
	and if we set $a=i\cdot\Delta$ we have
	\begin{align*}
		\norm{\widetilde u_i}_2^2 = \frac{1}{\sqrt{4\pi}}\prod_{k=1}^i(1-e^{-\frac{k}{2}\Delta^2}).
	\end{align*}
	
\end{lemma}

\begin{proof}
	We will prove the lemma by induction.	
	Recall that 
	\begin{align*}
		\widetilde u_0  = u_0 \qquad \text{and}\qquad \widetilde u_{i} = u_{i} - \sum_{j=0}^{i-1}\frac{\inner{u_{i}}{\widetilde u_j}}{\inner{\widetilde u_j}{\widetilde u_j}}\widetilde u_j \qquad\text{for $i>0$}
	\end{align*}
	Assuming that the statement holds for $j<i$, i.e.
	\begin{align*}
		\inner{g_a}{\widetilde u_j} = \frac{1}{\sqrt{4\pi}}e^{-\frac{1}{4}a^2}e^{-\frac{j}{4}\Delta^2}\prod_{k=1}^j(e^{\frac{1}{2}a\Delta-\frac{(k-1)}{2}\Delta^2}-1) \qquad \text{for any $a\in\mathbb{R}$}
	\end{align*}
	Note that we also have $\inner{g_a}{u_i} = \frac{1}{\sqrt{4\pi}}e^{-\frac{1}{4}(a-i\Delta)^2} = \frac{1}{\sqrt{4\pi}}e^{-\frac{1}{4}a^2}e^{\frac{1}{2}ia\Delta}e^{-\frac{1}{4}i^2\Delta^2}$.
	Now, we express $\inner{g_a}{\widetilde u_i}$.
	\begin{align*}
		\inner{g_a}{\widetilde u_i}
		& =
		\inner{g_a}{u_i-\sum_{j=0}^{i-1}\frac{\inner{u_i}{\widetilde u_j}}{\inner{\widetilde u_j}{\widetilde u_j}}\widetilde u_j} \\
		& =
		\frac{1}{\sqrt{4\pi}}e^{-\frac{1}{4}a^2}e^{\frac{1}{2}ia\Delta}e^{-\frac{1}{4}i^2\Delta^2} - \sum_{j=0}^{i-1}\frac{1}{\sqrt{4\pi}}e^{-\frac{1}{4}a^2}e^{-\frac{j}{4}\Delta^2}\prod_{k=1}^j(e^{\frac{1}{2}a\Delta-\frac{(k-1)}{2}\Delta^2}-1)\frac{\inner{u_i}{\widetilde u_j}}{\inner{\widetilde u_j}{\widetilde u_j}} 
	\end{align*}
	Consider the polynomial $\mathcal{P}(y) = e^{-\frac{1}{4}i^2\Delta^2}y^i - \sum_{j=0}^{i-1}e^{-\frac{j}{4}\Delta^2}\prod_{k=1}^j(e^{-\frac{(k-1)}{2}\Delta^2}y-1)\frac{\inner{u_i}{\widetilde u_j}}{\inner{\widetilde u_j}{\widetilde u_j}} $.
	Namely, we can rewrite $\inner{g_a}{\widetilde u_i}$ as
	\begin{align*}
		\inner{g_a}{\widetilde u_i}
		=
		\frac{1}{\sqrt{4\pi}}e^{-\frac{1}{4}a^2}\mathcal{P}(e^{\frac{1}{2}a\Delta}) \qquad \text{for any $a\in\mathbb{R}$}
	\end{align*}
	$\mathcal{P}$ is a polynomial of degree $i$ and we will argue that $1,e^{\frac{1}{2}\Delta^2},\dots,e^{\frac{i-1}{2}\Delta^2}$ are its roots.
	By setting $a=j\cdot \Delta$, we have
	\begin{align*}
		\inner{u_j}{\widetilde u_i}
		=
		\frac{1}{\sqrt{4\pi}}e^{-\frac{1}{4}j^2\Delta^2}\mathcal{P}(e^{\frac{1}{2}j\Delta^2})
	\end{align*}
	for $j=0,\dots,i-1$.
	Since $u_j$ lies on $\spn\{\widetilde u_0,\widetilde u_1,\dots,\widetilde u_j\}$, we have $\inner{u_j}{\widetilde u_i}=0$ by the definition of Gram-Schmidt process.
	It implies that $\mathcal{P}(e^{\frac{1}{2}j\Delta^2})=0$.
	
	Since we have all roots of $\mathcal{P}$, we can rewrite $\mathcal{P}(y)$ as 
	\begin{align*}
		\mathcal{P}(y)=C\cdot \prod_{j=0}^{i-1}(y-e^{\frac{1}{2}j\Delta^2})
	\end{align*}
	for some absolute constant $C$.
	By comparing the coefficient of $y^i$, we have $C=e^{-\frac{1}{4}i^2\Delta^2}$.
	Therefore,
	\begin{align*}
		\inner{g_a}{\widetilde u_i}
		& =
		\frac{1}{\sqrt{4\pi}}e^{-\frac{1}{4}a^2}\mathcal{P}(e^{\frac{1}{2}a\Delta}) \\
		& =
		\frac{1}{\sqrt{4\pi}}e^{-\frac{1}{4}a^2}e^{-\frac{1}{4}i^2\Delta^2}\prod_{j=0}^{i-1}(e^{\frac{1}{2}a\Delta}-e^{\frac{1}{2}j\Delta^2}) \\
		& =
		\frac{1}{\sqrt{4\pi}}e^{-\frac{1}{4}a^2}e^{-\frac{i}{4}\Delta^2}\prod_{k=1}^i(e^{\frac{1}{2}a\Delta-\frac{(k-1)}{2}\Delta^2}-1)
	\end{align*}
	and we proved the statement.
\end{proof}

\begin{lemma}[Restated Lemma \ref{lem:sumofcoeff}]\label{lem:sumofcoeff:detail}
	For any sufficiently small $\Delta>0$, we have
	\begin{align*}
		C_{\Delta,+}+C_{\Delta,-} \leq 2^{O(1/\Delta)}
	\end{align*}
\end{lemma}

\begin{proof}
	Recall that $\mathcal{V}$ is the subspace $\spn\{u_0,\dots,u_{m}\}$ and $\Pi_{\mathcal{V}}(g_a)$ is the projection of $g_a$ on the subspace $\mathcal{V}$.
	We have
	\begin{align*}
		\Pi_{\mathcal{V}}(g_a)
		& =
		\sum_{i=0}^{m}\inner{g_a}{\frac{\widetilde u_i}{\norm{\widetilde u_i}^2}}\widetilde u_i
	\end{align*}
	From Lemma \ref{lem:recur}, we plug the expressions into each coefficient $\inner{g_a}{\frac{\widetilde u_i}{\norm{\widetilde u_i}^2}}$ in the above formula.
	\begin{align*}
		\inner{g_a}{\frac{\widetilde u_i}{\norm{\widetilde u_i}^2}}
		& =
		e^{-\frac{1}{4}a^2}e^{-\frac{i}{4}\Delta^2}\frac{\prod_{k=1}^i(e^{\frac{1}{2}a\Delta-\frac{(k-1)}{2}\Delta^2}-1)}{\prod_{k=1}^i(1-e^{-\frac{k}{2}\Delta^2})}
	\end{align*}
	Consider the polynomial $\mathcal{P}_i(y) = e^{-\frac{i}{4}\Delta^2}\frac{\prod_{k=1}^i(e^{-\frac{(k-1)}{2}\Delta^2}y-1)}{\prod_{k=1}^i(1-e^{-\frac{k}{2}\Delta^2})}$ of degree $i<m$.
	Namely, we can rewrite $\inner{g_a}{\frac{\widetilde u_i}{\norm{\widetilde u_i}^2}}$ as $e^{-\frac{1}{4}a^2}\mathcal{P}_i(e^{\frac{1}{2}a})$.
	Hence, we now rewrite $\Pi_{\mathcal{V}}(g_a)$ as
	\begin{align*}
		\Pi_{\mathcal{V}}(g_a)
		& =
		\sum_{i=0}^{m}e^{-\frac{1}{4}a^2}\mathcal{P}_i(e^{\frac{1}{2}a\Delta})\widetilde u_i
	\end{align*}
	Recall that Gram-Schmidt process gives us the following connection between $u_i$ and $\widetilde u_i$.
	\begin{align*}
		\widetilde u_0  = u_0 \qquad \text{and}\qquad \widetilde u_{i} = u_{i} - \sum_{j=0}^{i-1}\frac{\inner{u_{i}}{\widetilde u_j}}{\inner{\widetilde u_j}{\widetilde u_j}}\widetilde u_j \qquad\text{for $i>0$}
	\end{align*}
	Hence, using the formula for $\tilde u_m$, we have
	\begin{align*}
		\Pi_{\mathcal{V}}(g_a)
		& =
		\sum_{i=0}^{m}e^{-\frac{1}{4}a^2}\mathcal{P}_i(e^{\frac{1}{2}a\Delta})\widetilde u_i
		=
		\sum_{i=0}^{m-1}e^{-\frac{1}{4}a^2}\mathcal{P}_i(e^{\frac{1}{2}a\Delta})\widetilde u_i + e^{-\frac{1}{4}a^2}\mathcal{P}_m(e^{\frac{1}{2}a\Delta})\widetilde u_m\\
		& =
		\sum_{i=0}^{m-1}e^{-\frac{1}{4}a^2}\mathcal{P}_i(e^{\frac{1}{2}a\Delta})\widetilde u_i + e^{-\frac{1}{4}a^2}\mathcal{P}_m(e^{\frac{1}{2}a\Delta})\big(u_{m} - \sum_{i=0}^{m-1}\frac{\inner{u_{m}}{\widetilde u_i}}{\inner{\widetilde u_i}{\widetilde u_i}}\widetilde u_i\big) \\
		& =
		\sum_{i=0}^{m-1}e^{-\frac{1}{4}a^2}\underbrace{\big(\mathcal{P}_i(e^{\frac{1}{2}a\Delta}) - \frac{\inner{u_{m}}{\widetilde u_i}}{\inner{\widetilde u_i}{\widetilde u_i}}\mathcal{P}_m(e^{\frac{1}{2}a\Delta})\big)}_{\text{a polynomial in $e^{\frac{1}{2}a\Delta}$ of degree $m$}}\widetilde u_i + e^{-\frac{1}{4}a^2}\mathcal{P}_m(e^{\frac{1}{2}a\Delta})u_{m}
	\end{align*}
	and if we recursively plug in the formula for $\widetilde u_i$ then it is easy to see that we can further rewrite $\Pi_{\mathcal{V}}(g_a)$ as 
	\begin{align*}
		\Pi_{\mathcal{V}}(g_a)
		& =
		\sum_{i=0}^{m}e^{-\frac{1}{4}a^2}\mathcal{Q}_i(e^{\frac{1}{2}a\Delta}) u_i
	\end{align*}
	for some polynomials $\mathcal{Q}_i$ of degree $m$.
	Since $u_j \in \mathcal{V}$, we have $\Pi_{\mathcal{V}}(u_j) = u_j$ by definition.
	It implies that $e^{-\frac{1}{4}j^2\Delta^2}\mathcal{Q}_i(e^{\frac{1}{2}j\Delta^2}) = 0$ for all $j=0,1,\dots,i-1,i+1,\dots,m$ by the fact that $u_0,\dots,u_m$ are linearly independent.
	It means that we have all roots of $\mathcal{Q}_i$ and hence
	\begin{align*}
		\mathcal{Q}_i(y)
		& =
		C_i\prod_{j=0, i\neq j}^m(y - e^{\frac{1}{2}j\Delta^2})
	\end{align*}
	for some absolute constant $C_i$.
	We also have $e^{-\frac{1}{4}i^2\Delta^2}\mathcal{Q}_i(e^{\frac{1}{2}i\Delta^2}) = 1$ and therefore
	\begin{align*}
		\mathcal{Q}_i(y)
		& =
		e^{\frac{1}{4}i^2\Delta^2}\frac{\prod_{j=0, i\neq j}^m(y - e^{\frac{1}{2}j\Delta^2})}{\prod_{j=0, i\neq j}^m(e^{\frac{1}{2}i\Delta^2} - e^{\frac{1}{2}j\Delta^2})}
	\end{align*}
	In particular, 
	\begin{align*}
		\alpha_i = e^{-\frac{1}{4}}\mathcal{Q}_i(e^{-\frac{1}{2}\Delta}) = e^{-\frac{1}{4}}e^{\frac{1}{4}i^2\Delta^2}\frac{\prod_{j=0, i\neq j}^m(e^{-\frac{1}{2}\Delta} - e^{\frac{1}{2}j\Delta^2})}{\prod_{j=0, i\neq j}^m(e^{\frac{1}{2}i\Delta^2} - e^{\frac{1}{2}j\Delta^2})}
	\end{align*}
	
	Furthermore, recall that $C_{\Delta,+}$ and $C_{\Delta,-}$ are defined as $C_{\Delta,+} = \sum_{i\in J_+} \alpha_i$ and $C_{\Delta,-} = \sum_{i\in J_-} -\alpha_i$ respectively.
	In other words, $C_{\Delta,+}+C_{\Delta,-} = \sum_{i=0}^m \abs{\alpha_i}$.
	We need to give a bound for each $\abs{\alpha_i}$.
	A useful inequality is $1+x \leq e^{x} \leq 1+x+x^2$ for any sufficiently small $x$.
	For $i>j$, the factors in the denominator of the fraction in $\abs{\alpha_i}$ is
	\begin{align*}
		\abs{e^{\frac{1}{2}i\Delta^2} - e^{\frac{1}{2}j\Delta^2}}
		& \geq
		1+\frac{1}{2}i\Delta^2 - 1-\frac{1}{2}j\Delta^2-(\frac{1}{2}j\Delta)^2
		\geq
		\frac{1}{2}(i-j)\Delta^2-\frac{1}{4}j^2\Delta^4.
	\end{align*}
	Since $j\leq m=\frac{1}{\Delta}$ which implies $j^2\Delta^4 \leq \Delta^2 \leq (i-j)\Delta^2$, we have $\abs{e^{\frac{1}{2}i\Delta^2} - e^{\frac{1}{2}j\Delta^2}} \geq \frac{1}{4}(i-j)\Delta^2$.
	Similarly, for $j>i$, we have $\abs{e^{\frac{1}{2}i\Delta^2} - e^{\frac{1}{2}j\Delta^2}} \geq \frac{1}{4}(j-i)\Delta^2$.
	Hence,
	\begin{align*}
		\abs{\prod_{j=0, i\neq j}^m(e^{\frac{1}{2}i\Delta^2} - e^{\frac{1}{2}j\Delta^2})}
		& \geq 
		\big(\frac{\Delta^2}{4}\big)^m\cdot i!(m-i)!
	\end{align*}
	On the other hand, the factors in the numerator of the fraction in $\abs{\alpha_i}$ is
	\begin{align*}
		\abs{e^{-\frac{1}{2}\Delta} - e^{\frac{1}{2}j\Delta^2}}
		& \leq 
		1+\frac{1}{2}j\Delta^2+(\frac{1}{2}j\Delta^2)^2 - 1+\frac{1}{2}\Delta
		=
		\frac{1}{2}(m+j)\Delta^2+\frac{1}{4}j^2\Delta^4
	\end{align*}
	Since $j \leq \frac{1}{\Delta}$ which implies $j^2\Delta^4\leq \Delta^2 \leq(m+j)\Delta^2$, we have $\abs{e^{-\frac{1}{2}\Delta} - e^{\frac{1}{2}j\Delta^2}} \leq \frac{3}{4}(m+j)\Delta^2$.
	Hence,
	\begin{align*}
		\abs{\prod_{j=0, i\neq j}^m(e^{-\frac{1}{2}\Delta} - e^{\frac{1}{2}j\Delta^2})}
		& \leq
		\frac{1}{m+i}\big(\frac{3\Delta^2}{4}\big)^m \cdot \frac{(2m+1)!}{m!}
		\leq
		\big(\frac{3\Delta^2}{4}\big)^m \cdot \frac{(2m+1)!}{m!}
	\end{align*}
	Combining the above inequalities,
	\begin{align*}
		\abs{\alpha_i}
		& \leq
		e^{-\frac{1}{4}}\cdot e^{\frac{1}{4}i^2\Delta^2} \cdot 3^m \cdot(m+1)\cdot \bigg(\frac{(2m+1)!}{m!(m+1)!}\bigg) \cdot \bigg(\frac{m!}{i!(m-i)!}\bigg)
		=
		2^{O(1/\Delta)}
	\end{align*}
	and therefore
	\begin{align*}
		C_{\Delta,+}+C_{\Delta,-} = 2^{O(1/\Delta)}
	\end{align*}
\end{proof}

\begin{lemma}[Restated Lemma \ref{lem:beta}]\label{lem:beta:detail}
	For any sufficiently small $\Delta>0$, we have
	\begin{align*}
		\beta_\Delta
		\leq
		\frac{1}{2^{\Omega((1/\Delta)\log(1/\Delta))}}
	\end{align*}
\end{lemma}

\begin{proof}
	Recall that the definition of $\beta_\Delta$ is $\frac{\norm{\Pi_{\mathcal{V}}(v)-v}_2}{\norm{v}_2}$.
	Hence, $\beta_\Delta^2 = 1-\frac{\norm{\Pi_{\mathcal{V}}(v)}_2^2}{\norm{v}_2^2}$.
	We first express $\frac{\norm{\Pi_{\mathcal{V}}(v)}_2^2}{\norm{v}_2^2}$ in an explicit formulation.
	Recall that 
	\begin{align*}
		\Pi_{\mathcal{V}}(v)
		& =
		\sum_{i=0}^{m}\inner{v}{\frac{\widetilde u_i}{\norm{\widetilde u_i}^2}}\widetilde u_i.
	\end{align*}
	By Pythagorean theorem, we have
	\begin{align*}
		\frac{\norm{\Pi_{\mathcal{V}}(v)}_2^2}{\norm{v}_2^2}
		& =
		\sum_{i=0}^{m}\inner{\frac{v}{\norm{v}_2}}{\frac{\widetilde u_i}{\norm{\widetilde u_i}_2}}^2
	\end{align*}
	From Lemma \ref{lem:recur}, each term can be expressed as
	\begin{align*}
		\inner{\frac{v}{\norm{v}_2}}{\frac{\widetilde u_i}{\norm{\widetilde u_i}_2}}^2
		& =
		e^{-\frac{1}{2}}\cdot e^{-\frac{i}{2}\Delta^2}\cdot \prod_{j=1}^i \frac{(1-e^{-\frac{1}{2}\Delta - \frac{j-1}{2}\Delta^2})^2}{1-e^{-\frac{j}{2}\Delta^2}}
	\end{align*}
	To ease the notations, we consider the following sequences.
	For any $y<1$,
	\begin{align*}
		S_0 & = 1-y^{m^2} \qquad \text{and} \qquad S_i = S_{i-1} - y^{m^2+i} T_i \qquad \text{for $i>0$}
	\end{align*}
	where 
	\begin{align*}
	    T_i = \prod_{j=1}^i\frac{(1-y^{m-1+j})^2}{1-y^j} \qquad \text{for $i>0$}
	\end{align*}
	and recall that $m=\frac{1}{\Delta}$.
	In other words, we replace $e^{-\frac{1}{2}\Delta^2}$ with $y$ and $1-S_i$ is the sum of the first $i+1$ terms of $\sum_{j=0}^{m}\inner{\frac{v}{\norm{v}_2}}{\frac{\widetilde u_j}{\norm{\widetilde u_j}_2}}^2$ which means $S_{m} = 1- \frac{\norm{\Pi_{\mathcal{V}}(v)}_2^2}{\norm{v}_2^2} = \beta_\Delta^2$.

	In Lemma \ref{lem:closedform}, we have the closed form of $S_i$ and $T_i$,
	\begin{align*}
		S_{i-1} &= \prod_{j=1}^i(1-y^{m-1+j}) \cdot \bigg(\sum_{(k_1,k_2,\dots,k_i)\in \mathcal{I}_i} y^{\sum_{j=1}^{i} (m-1+j)k_j}\bigg) \\
		T_{i} &= \prod_{j=1}^i(1-y^{m-1+j}) \cdot \bigg(\sum_{(k_1,k_2,\dots,k_{i})\in \mathcal{I}_i} y^{\sum_{j=1}^{i} jk_j}\bigg) 
	\end{align*}
	where $\mathcal{I}_i = \setdef{(k_1,k_2,\dots,k_{i})}{0\leq k_1 \leq k_1+k_2 \leq \dots \leq k_1+\dots+k_{i} \leq m-1}$.
	It is easy to see that $S_{i-1}\leq T_{i}$ since there is a one-to-one correspondence between the terms $y^{\sum_{j=1}^{i} (m-1+j)k_j}$ and $y^{\sum_{j=1}^{i} jk_j}$ and $\sum_{j=1}^{i} (m-1+j)k_j \geq \sum_{j=1}^{i} jk_j$ and the fact $y<1$.
	Now, 
	\begin{align*}
		\beta_\Delta^2
		& =
		S_m 
		\leq 
		T_{m+1}
		=
		\prod_{j=1}^{m+1}\frac{(1-y^{m-1+j})^2}{1-y^j}
		=
		\prod_{j=1}^{m+1}\frac{(1-e^{-\frac{1}{2}\Delta-\frac{j-1}{2}\Delta^2})^2}{1-e^{-\frac{j}{2}\Delta^2}} \qquad \text{when $y=e^{-\frac{1}{2}\Delta^2}$}
	\end{align*}
	A useful inequality is $1+x \leq e^x \leq 1+2x$ for any sufficiently small $x$.
	Then, we have
	\begin{align*}
		\prod_{j=1}^{m+1}(1-e^{-\frac{j}{2}\Delta^2})
		& \geq
		\prod_{j=1}^{m+1}(1-(1-2\frac{j}{2}\Delta^2))
		=
		\Delta^{2(m+1)}\cdot (m+1)!
	\end{align*}
	and, since $j-1\leq m= \frac{1}{\Delta}$,
	\begin{align*}
		\prod_{j=1}^{m+1}(1-e^{-\frac{1}{2}\Delta-\frac{j-1}{2}\Delta^2})^2
		& \leq
		\prod_{j=1}^{m+1}(1-(1-\frac{1}{2}\Delta-\frac{j-1}{2}\Delta^2))^2
		\leq
		\Delta^{2(m+1)}		
	\end{align*}
	We conclude that 
	\begin{align*}
		\beta_{\Delta} \leq \sqrt{\frac{\Delta^{2(m+1)}	}{\Delta^{2(m+1)}\cdot (m+1)!}}
		=
		\sqrt{\frac{1}{(m+1)!}}
		=
		\frac{1}{2^{\Omega(m\log m)}}
		=
		\frac{1}{2^{\Omega((1/\Delta)\log(1/\Delta))}}
	\end{align*}
\end{proof}

\begin{lemma}\label{lem:closedform}
		Let $S_i$ and $T_i$ be the recurrence sequence defined in the proof of Lemma \ref{lem:beta:detail}.
		We have
		\begin{align*}
			S_{i-1} &= \prod_{j=1}^i(1-y^{m-1+j}) \cdot \bigg(\sum_{(k_1,k_2,\dots,k_i)\in \mathcal{I}_i} y^{\sum_{j=1}^{i} (m-1+j)k_j}\bigg) \\
			T_{i} &= \prod_{j=1}^i(1-y^{m-1+j}) \cdot \bigg(\sum_{(k_1,k_2,\dots,k_{i})\in \mathcal{I}_i} y^{\sum_{j=1}^{i} jk_j}\bigg) 
		\end{align*}
		where $\mathcal{I}_i = \setdef{(k_1,k_2,\dots,k_{i})}{0\leq k_1 \leq k_1+k_2 \leq \dots \leq k_1+\dots+k_{i} \leq m-1}$.
	\end{lemma}

\begin{proof}
	We first prove the expression for $T_{i}$.
	When $i=1$, we have
	\begin{align*}
		T_{1} = \frac{(1-y^{m})^2}{1-y} = (1-y^{m})(\sum_{k_1=0}^{m-1} y^{k_1})
	\end{align*}
	By induction, we have
	\begin{align*}
		T_{i+1}
		& =
		T_i \cdot \frac{(1-y^{m+i})^2}{1-y^{i+1}} \\
		& =
		\prod_{j=1}^i(1-y^{m-1+j}) \cdot \bigg(\sum_{(k_1,k_2,\dots,k_{i})\in \mathcal{I}_i} y^{\sum_{j=1}^{i} jk_j}\bigg)  \cdot \frac{(1-y^{m+i})^2}{1-y^{i+1}} \\
		& =
		\prod_{j=1}^{i+1}(1-y^{m-1+j}) \cdot \bigg(\sum_{(k_1,k_2,\dots,k_{i})\in \mathcal{I}_i} y^{\sum_{j=1}^{i} jk_j}\bigg)  \cdot \frac{1-y^{m+i}}{1-y^{i+1}}
	\end{align*}
	In other words, we want to prove 
	\begin{align*}
		\MoveEqLeft \bigg(\sum_{(k_1,k_2,\dots,k_{i})\in \mathcal{I}_i} y^{\sum_{j=1}^{i} jk_j}\bigg)  \cdot \frac{1-y^{m+i}}{1-y^{i+1}}
		=
		\sum_{(k_1,k_2,\dots,k_{i+1})\in \mathcal{I}_{i+1}} y^{\sum_{j=1}^{i+1} jk_j}
	\end{align*}
	We are now examining
	\begin{align*}
		\MoveEqLeft \bigg(\sum_{(k_1,k_2,\dots,k_{i+1})\in \mathcal{I}_{i+1}} y^{\sum_{j=1}^{i+1} jk_j}\bigg)  \cdot (1-y^{i+1}) \\
		& =
		\sum_{(k_1,k_2,\dots,k_{i+1})\in \mathcal{I}_{i+1}} y^{\sum_{j=1}^{i+1} jk_j} - \sum_{(k_1,k_2,\dots,k_{i+1})\in \mathcal{I}_{i+1}} y^{\sum_{j=1}^{i} jk_j + (i+1)(k_{i+1}+1)} \\
		& =
		\sum_{(k_1,k_2,\dots,k_{i})\in \mathcal{I}_{i}} \bigg(\sum_{k_{i+1}=0}^{\ell} y^{\sum_{j=1}^{i+1} jk_j} - \sum_{k_{i+1}=0}^{\ell} y^{\sum_{j=1}^{i} jk_j + (i+1)(k_{i+1}+1)} \bigg) \numberthis\label{eq:temp}
	\end{align*}
	where $\ell = m-1-\sum_{j=1}^ik_j$.
	We fix the indices $k_1,\dots,k_i$ and consider the summation with the index $k_{i+1}$. 
	We have
	\begin{align*}
		\sum_{k_{i+1}=0}^{\ell} y^{\sum_{j=1}^{i+1} jk_j} - \sum_{k_{i+1}=0}^{\ell} y^{\sum_{j=1}^{i} jk_j + (i+1)(k_{i+1}+1)} 
		& =
		y^{\sum_{j=1}^{i} jk_j} -  y^{\sum_{j=1}^{i} jk_j + (i+1)(\ell+1)}\numberthis\label{eq:temp1}
	\end{align*}
	Therefore, we plug \eqref{eq:temp1} into \eqref{eq:temp}.
	\begin{align*}
		\bigg(\sum_{(k_1,k_2,\dots,k_{i+1})\in \mathcal{I}_{i+1}} y^{\sum_{j=1}^{i+1} jk_j}\bigg)  \cdot (1-y^{i+1}) 
		& =
		\sum_{(k_1,k_2,\dots,k_{i})\in \mathcal{I}_{i}} \bigg(y^{\sum_{j=1}^{i} jk_j} -  y^{\sum_{j=1}^{i} jk_j + (i+1)(\ell+1)} \bigg) \numberthis\label{eq:temp3}
	\end{align*}
	Note that $\sum_{j=1}^{i} jk_j + (i+1)(\ell+1) =m+i+\sum_{j=2}^{i}(j-1)k_{j} + i(m-1-\sum_{j=1}^{i} k_j)$, the term $y^{\sum_{j=1}^{i} jk_j + (i+1)(\ell+1)}$ becomes $y^{m+i+\sum_{j=2}^{i}(j-1)k_{j} + i(m-1-\sum_{j=1}^{i} k_j)}$.
	By change of variables, we have
	\begin{align*}
	    \mathcal{I}_i = \setdef{(k_2,\dots,k_i, m-1-\sum_{j=1}^ik_j)}{(k_1,k_2,\dots,k_i)\in \mathcal{I}_i}.
	\end{align*}
	and hence we have
	\begin{align*}
	    \sum_{(k_1,k_2,\dots,k_{i})\in \mathcal{I}_{i}} y^{m+i+\sum_{j=2}^{i}(j-1)k_{j} + i(m-1-\sum_{j=1}^{i} k_j)}
	    & =
	    \sum_{(k_1,k_2,\dots,k_{i})\in \mathcal{I}_{i}} y^{m+i+\sum_{j=1}^{i} jk_j}\numberthis\label{eq:temp2}.
	\end{align*}
	By plugging \eqref{eq:temp2} into \eqref{eq:temp3}, we have
	\begin{align*}
		\bigg(\sum_{(k_1,k_2,\dots,k_{i+1})\in \mathcal{I}_{i+1}} y^{\sum_{j=1}^{i+1} jk_j}\bigg)  \cdot (1-y^{i+1}) 
		& =
		\bigg(\sum_{(k_1,k_2,\dots,k_{i})\in \mathcal{I}_i} y^{\sum_{j=1}^{i} jk_j}\bigg)( 1-y^{m+i})
	\end{align*}

	Now, we will prove the expression for $S_{i-1}$.
	When $i=0$,
	\begin{align*}
		S_0 = 1-y^{m^2} = (1-y^{m})(\sum_{k_1=0}^{m-1} y^{k_1\cdot m})
	\end{align*}
	By induction and the expression for $T_i$, we have
	\begin{align*}
		S_i 
		& =
		S_{i-1} - y^{m^2+i}T_i \\
		& =
		\prod_{j=1}^i(1-y^{m-1+j}) \cdot \bigg(\sum_{(k_1,k_2,\dots,k_i)\in \mathcal{I}_i} y^{\sum_{j=1}^{i} (m-1+j)k_j}\bigg) \\
		& 
		\qquad- y^{m^2+i}\prod_{j=1}^i(1-y^{m-1+j}) \cdot \bigg(\sum_{(k_1,k_2,\dots,k_{i})\in \mathcal{I}_i} y^{\sum_{j=1}^{i} jk_j}\bigg) \\
		& =
		\prod_{j=1}^i(1-y^{m-1+j}) \cdot \bigg(\sum_{(k_1,k_2,\dots,k_i)\in \mathcal{I}_i} y^{\sum_{j=1}^{i} (m-1+j)k_j} - \sum_{(k_1,k_2,\dots,k_{i})\in \mathcal{I}_i} y^{m^2+i+\sum_{j=1}^{i} jk_j}\bigg) 
	\end{align*}
	In other words, we need to prove
	\begin{align*}
		\MoveEqLeft \sum_{(k_1,k_2,\dots,k_i)\in \mathcal{I}_i} y^{\sum_{j=1}^{i} (m-1+j)k_j} - \sum_{(k_1,k_2,\dots,k_{i})\in \mathcal{I}_i} y^{m^2+i+\sum_{j=1}^{i} jk_j} \\
		& =
		(1-y^{m+i})\bigg(\sum_{(k_1,k_2,\dots,k_{i+1})\in \mathcal{I}_{i+1}} y^{\sum_{j=1}^{i} (m-1+j)k_j}\bigg)
	\end{align*}
	Note that, by change of variables, we have
	\begin{align*}
	    \mathcal{I}_i = \setdef{(k_2,\dots,k_i, m-1-\sum_{j=1}^ik_j)}{(k_1,k_2,\dots,k_i)\in \mathcal{I}_i}
	\end{align*}and therefore
	\begin{align*}
		\sum_{(k_1,k_2,\dots,k_{i})\in \mathcal{I}_i} y^{m^2+i+\sum_{j=1}^{i} jk_j} 
		& =
		\sum_{(k_1,k_2,\dots,k_{i})\in \mathcal{I}_i} y^{m^2+i+\sum_{j=2}^{i} (j-1)k_j + i(m-1-\sum_{j=1}^ik_j)}. 
	\end{align*}
	Also, 
	\begin{align*}
		m^2+i+\sum_{j=2}^{i} (j-1)k_j + i(m-1-\sum_{j=1}^ik_j)
		& =
		\sum_{j=1}^{i} (m-1+j)k_j + (m+i)(m-\sum_{j=1}^ik_j)
	\end{align*}
	Hence, we have
	\begin{align*}
		\MoveEqLeft \sum_{(k_1,k_2,\dots,k_i)\in \mathcal{I}_i} y^{\sum_{j=1}^{i} (m-1+j)k_j} - \sum_{(k_1,k_2,\dots,k_{i})\in \mathcal{I}_i} y^{m^2+i+\sum_{j=1}^{i} jk_j} \\
		& =
		\sum_{(k_1,k_2,\dots,k_i)\in \mathcal{I}_i} y^{\sum_{j=1}^{i} (m-1+j)k_j} - \sum_{(k_1,k_2,\dots,k_{i})\in \mathcal{I}_i} y^{\sum_{j=1}^{i} (m-1+j)k_j + (m+i)(m-\sum_{j=1}^ik_j)} \\
		& =
		(1-y^{m+i})\bigg(\sum_{(k_1,k_2,\dots,k_{i+1})\in \mathcal{I}_{i+1}} y^{\sum_{j=1}^{i} (m-1+j)k_j}\bigg)
	\end{align*}
\end{proof}

\begin{lemma}[Restated Lemma \ref{lem:diffofcoeff}]\label{lem:diffofcoeff:detail}
	For any sufficiently small $\Delta>0$, we have
	\begin{align*}
		\abs{C_{\Delta,+}-C_{\Delta,-}-1} \leq \frac{1}{2^{\Omega((1/\Delta)\log(1/\Delta))}}.
	\end{align*}
\end{lemma}

\begin{proof}
	By the definition of $C_{\Delta,+}$ and $C_{\Delta,-}$, 
	\begin{align*}
		C_{\Delta,+} = \sum_{i\in J_+} \alpha_i \qquad \text{and}\qquad C_{\Delta,-} = \sum_{i\in J_-} -\alpha_i
	\end{align*}
	where $J_+ = \setdef{i}{\alpha_i\geq 0}$ and $J_-=\setdef{i}{\alpha_i<0}$.
	Also, the coefficients $\alpha_i$ satisfy $\Pi_{\mathcal{V}}(v) = \sum_{i=0}^{m} \alpha_{i} u_{i}$.
	If we take the integral,
	\begin{align*}
		\int_{x\in\mathbb{R}} \Pi_{\mathcal{V}}(v)\dir x = \int_{x\in\mathbb{R}} \bigg(\sum_{i=0}^{m} \alpha_{i} u_{i}\bigg)\dir x = \sum_{i=0}^m \alpha_i = C_{\Delta,+}-C_{\Delta,-}
	\end{align*}
	by the fact that $u_i=g_{i\cdot\Delta}$ are Gaussians and hence $\int_{x\in\mathbb{R}}u_i \dir x=1$.
	Since $v=g_{-1}$ is also a Gaussian, it implies
	\begin{align*}
		\abs{C_{\Delta,+}-C_{\Delta,-}-1}
		& =
		\abs{\int_{x\in\mathbb{R}} (\Pi_{\mathcal{V}}(v) - v)\dir x}.
	\end{align*}
	By triangle inequality, we have
	\begin{align*}
		\abs{C_{\Delta,+}-C_{\Delta,-}-1}
		& =
		\abs{\int_{x\in\mathbb{R}} (\Pi_{\mathcal{V}}(v) - v)\dir x} 
		\leq 
		\int_{x\in\mathbb{R}} \abs{\Pi_{\mathcal{V}}(v) - v} \dir x 
	\end{align*}
	For any $L>0$, we split the integral into two parts.
	\begin{align*}
		\MoveEqLeft \abs{C_{\Delta,+}-C_{\Delta,-}-1}\\
		& \leq
		\int_{x\in[-L,L]} \abs{\Pi_{\mathcal{V}}(v)-v }\dir x + \int_{x\notin[-L,L]} \abs{\Pi_{\mathcal{V}}(v)-v }\dir x 
	\end{align*}
	
	We first analyze the second term $\int_{x\notin[-L,L]} \abs{\Pi_{\mathcal{V}}(v)-v }\dir x $. 
	By triangle inequality, we express the term $\abs{\Pi_{\mathcal{V}}(v)-v}$.
	\begin{align*}
		\abs{\Pi_{\mathcal{V}}(v)-v} = \abs{\sum_{i=0}^{m} \alpha_{i} u_{i} - v} \leq \sum_{i=0}^{m} \abs{\alpha_{i}} u_{i} + v
	\end{align*}
	Since all Gaussians in $\Pi_{\mathcal{V}}(v)-v$ centered in $[-1,1]$, all integrals $\int_{x\notin[-L,L]} u_i\dir x, \int_{x\notin[-L,L]} v\dir x$ are bounded by $\int_{x\notin[-(L-1),L-1]} g_0 \dir x$.
	Namely, we have
	\begin{align*}
		\int_{x\notin[-L,L]} \abs{\Pi_{\mathcal{V}}(v)-v }\dir x 
		& \leq
		\int_{x\notin[-L,L]} (\sum_{i=0}^{m} \abs{\alpha_{i}} u_{i} + v)\dir x \\
		&\leq
		(C_{\Delta,+}+C_{\Delta,-}+1)(\int_{x\notin[-(L-1),L-1]} g_0 \dir x)
	\end{align*}
	A straightforward calculation gives 
	\begin{align*}
		\int_{x\notin[-(L-1),L-1]} g_0 \dir x
		& =
		\int_{x\notin[-(L-1),L-1]} \frac{1}{\sqrt{2\pi}}e^{-\frac{1}{2}x^2} \dir x
		\leq
		\frac{4}{\sqrt{2\pi}(L-1)}e^{-\frac{1}{2}(L-1)^2}
	\end{align*}
	which means
	\begin{align*}
		\int_{x\notin[-L,L]} \abs{\Pi_{\mathcal{V}}(v)-v }\dir x 
		\leq
		(C_{\Delta,+}+C_{\Delta,-}+1)\frac{4}{\sqrt{2\pi}(L-1)}e^{-\frac{1}{2}(L-1)^2}
	\end{align*}
	
	Now, we analyze the first term $\int_{x\in[-L,L]} \abs{\Pi_{\mathcal{V}}(v)-v }\dir x$.
	By Cauchy inequality, 
	\begin{align*}
		\int_{x\in[-L,L]} \abs{\Pi_{\mathcal{V}}(v)-v }\dir x
		& \leq
		\sqrt{2L\int_{x\in[-L,L]} (\Pi_{\mathcal{V}}(v)-v )^2\dir x }.
	\end{align*}
	Moreover, 
	\begin{align*}
		\int_{x\in[-L,L]} (\Pi_{\mathcal{V}}(v)-v )^2\dir x 
		& \leq
		\int_{x\in\mathbb{R}} (\Pi_{\mathcal{V}}(v)-v )^2\dir x 
		=
		\norm{\Pi_{\mathcal{V}}(v)-v}_2^2
		=
		\beta_\Delta^2\norm{v}_2^2.
	\end{align*}
	It means
	\begin{align*}
		\int_{x\in[-L,L]} \abs{\Pi_{\mathcal{V}}(v)-v }\dir x
		\leq
		\sqrt{2L}\beta_\Delta\norm{v}_2
	\end{align*}
	
	By taking $L=\Theta\big(\sqrt{\log \frac{C_{\Delta,+}+C_{\Delta,-}}{\beta_\Delta}}\big)$, we have
	\begin{align*}
		\abs{C_{\Delta,+}-C_{\Delta,-}-1}
		& =
		O\bigg(\big(\log \frac{C_{\Delta,+}+C_{\Delta,-}}{\beta_\Delta}\big)^{1/4}\cdot\beta_\Delta\bigg).
	\end{align*}
	By Lemma \ref{lem:sumofcoeff} and \ref{lem:beta}, we conclude that 
	\begin{align*}
		\abs{C_{\Delta,+}-C_{\Delta,-}-1} \leq \frac{1}{2^{\Omega((1/\Delta)\log(1/\Delta))}}
	\end{align*}
\end{proof}

\begin{lemma}[Restated Lemma \ref{lem:loneltwo}]\label{lem:loneltwo:detail}
	We have
	\begin{align*}
		\norm{f_1-f'_1}_2\leq O(\sqrt{\norm{f_1-f'_1}_1}) \qquad \text{and} \qquad
		\norm{f - f'}_1 = O(\norm{f - f'}_2^{2/3})
	\end{align*}
\end{lemma}

\begin{proof}
	First, we have
	\begin{align*}
		\norm{f_1-f'_1}_2^2
		& =
		\int_{x\in\mathbb{R}} (f_1(x) - f'_1(x))^2\dir x
		\leq
		\int_{x\in\mathbb{R}} \abs{f_1(x) - f'_1(x)}(f_1(x) + f'_1(x))\dir x.
	\end{align*}
	Since $f_1$ and $f'_1$ are mixtures of Gaussians, $f_1(x), f'_1(x) \leq \frac{1}{\sqrt{2\pi}}$ for all $x\in\mathbb{R}$.
	Therefore, we have
	\begin{align*}
		\norm{f_1-f'_1}_2
		& \leq
		\sqrt{\frac{2}{\sqrt{2\pi}}\norm{f_1-f'_1}_1}
		=
		O(\sqrt{\norm{f_1-f'_1}_1}).
	\end{align*}
	
	Also, in Lemma 6 of \cite{nguyen2013convergence}, they showed that 
	\begin{align*}
		\norm{f-f'}_1 = O((E+E')^{1/3}\norm{f-f'}_2^{2/3})
	\end{align*}
	where $E = \int_{x\in\mathbb{R}} \abs{x}f(x)\dir x$ and $E' = \int_{x\in\mathbb{R}} \abs{x}f'(x)\dir x$.
	We first bound the term $\int_{x\in\mathbb{R}} \abs{x}g_{\mu}(x) \dir x$ for any $\mu\in\mathbb{R}$.
	We have
	\begin{align*}
		\int_{x\in\mathbb{R}} \abs{x}g_{\mu}(x) \dir x
		& =
		\int_{x\geq 0} xg_{\mu}(x) \dir x - \int_{x \leq 0} xg_{\mu}(x) \dir x
	\end{align*}
	Note that $x g_{\mu}(x) = (x-\mu) g_{\mu}(x)+ \mu g_{\mu}(x)$.
	Plugging it into the equation,
	\begin{align*}
		\int_{x\in\mathbb{R}} \abs{x}g_{\mu}(x) \dir x
		& \leq
		\int_{x\in\mathbb{R}} \abs{x}g_0(x) \dir x + \int_{x\in\mathbb{R}} \abs{\mu}g_{\mu}(x) \dir x 
		=
		O(\abs{\mu}).
	\end{align*}
	Since all Gaussians in both $f$ and $f$ are in $[-2,1]$, we have both $E,E' = O(1)$.
	Therefore, we have $\norm{f - f'}_1 = O(\norm{f - f'}_2^{2/3})$.
\end{proof}

\begin{lemma}[Restated Lemma \ref{lem:hermite_coeff}]\label{lem:hermite_coeff:detail}
	For $i=1,2$ and any $j\in\mathbb{N}_{0}$, $\abs{\alpha_{i,j}} \leq O(1)\cdot\frac{1}{\sqrt{j!}(2\sqrt{2})^j}$.
\end{lemma}

\begin{proof}
	By definition, we have
	\begin{align*}
		\alpha_{i,j}
		=
		\inner{f_i}{ \psi_{j,r_i}} 
		& =
		\int_{-\infty}^{\infty}\big(\int_{r_i-\frac{1}{2}}^{r_i+\frac{1}{2}} \nu_i(\mu) g_{\mu}(x)\dir \mu\big)  \psi_j(x-r_i) \dir x \\
		& =
		\int_{r_i-\frac{1}{2}}^{r_i+\frac{1}{2}} \nu_i(\mu)\big(\int_{-\infty}^{\infty} g_{\mu}(x) \psi_j(x-r_i) \dir x\big) \dir \mu \\
		& =
		\int_{-\frac{1}{2}}^{\frac{1}{2}} \nu_i(\mu+r_i)\inner{g_\mu}{ \psi_{j,0}} \dir \mu
	\end{align*}
	and, by triangle inequality, 
	\begin{align*}
		\abs{\alpha_{i,j}} \leq \int_{-\frac{1}{2}}^{\frac{1}{2}} \nu_i(\mu+r_i)\abs{\inner{g_\mu}{ \psi_{j,0}}} \dir \mu.
	\end{align*}
	By \eqref{eq:hermite_inner_g}, we have $\abs{\inner{g_\mu}{ \psi_{j,0}}} = \frac{1}{\sqrt{2^{j+1}j!\sqrt{\pi}}}e^{-\frac{1}{4}\mu^2} \abs{\mu}^j$ and, for $\mu\in[-\frac{1}{2},\frac{1}{2}]$, we have $e^{-\frac{1}{4}\mu^2} \abs{\mu}^j \leq \frac{1}{2^j}$.
	This implies 
	\begin{align*}
		\abs{\inner{g_\mu}{ \psi_{j,0}}}
		& \leq
		\frac{1}{\sqrt{2^{j+1}j!\sqrt{\pi}}}\cdot\frac{1}{2^j}
		=
		O\Big(\frac{1}{\sqrt{j!}(2\sqrt{2})^j}\Big).
	\end{align*}
	Hence, we have
	\begin{align*}
		\abs{\alpha_{i,j}}
		& \leq
		\int_{-\frac{1}{2}}^{\frac{1}{2}} \nu_i(\mu+r_i)\cdot O\Big(\frac{1}{\sqrt{j!}(2\sqrt{2})^j}\Big)\dir \mu
		=
		O\Big(\frac{1}{\sqrt{j!}(2\sqrt{2})^j}\Big).
	\end{align*}
	The last equality is due to the fact that $\nu_i$ is a pdf whose support is $I_i=[r_i-\frac{1}{2},r_i+\frac{1}{2}]$.
\end{proof}

\begin{lemma}[Restated Lemma \ref{lem:tail}]\label{lem:tail:detail}
	Let $\Delta>0$ and $\ell$ be a nonnegative integer that $\ell=\Omega(1)$.
	If $\abs{\lambda_{i,j} - \widehat \lambda_{i,j}} < \Delta$ for all $j\in[\ell]$, then $\norm{w_if_i - \widetilde f_i}_1 = O(\Delta \ell^{5/4} + \frac{w_i}{10^\ell})$.
\end{lemma}

\begin{proof}
	By the triangle inequality, we can bound the term $\norm{w_if_i-\widetilde f_i}_1$ as follows:
	\begin{align*}
		\norm{w_if_i-\widetilde f_i}_1
		&=
		\norm{\sum_{j=0}^{\ell-1}(\lambda_{i,j}-\widehat \lambda_{i,j}) \psi_{j,r_i} + \sum_{j=\ell}^\infty\lambda_{i,j} \psi_{j,r_i}}_1 \\
		&\leq
		\sum_{j=0}^{\ell-1}\abs{\lambda_{i,j}-\widehat \lambda_{i,j}}\norm{ \psi_{j,r_i}}_1 + \sum_{i=\ell}^\infty\abs{\lambda_{i,j}}\norm{ \psi_{j,r_i}}_1
	\end{align*}
	\noindent
	By Lemma \ref{lem:hermite_l1} and the assumption that $\abs{\lambda_{i,j} - \widehat \lambda_{i,j}} < \Delta$, the first summation can be bounded by
	\begin{align*}
		\sum_{j=0}^{\ell-1}\abs{\lambda_{i,j}-\widehat \lambda_{i,j}}\norm{ \psi_{j,r_i}}_1
		& \leq
		O(\Delta\ell^{5/4}).
	\end{align*}
	By Lemma \ref{lem:hermite_l1} and Lemma \ref{lem:hermite_coeff}, the second summation can be bounded by
	\begin{align*}
		\sum_{j=\ell}^\infty\abs{\lambda_{i,j}}\norm{ \psi_{j,r_i}}_1
		& \leq
		\sum_{j=\ell}^{\infty} O\Big(w_i \cdot\frac{1}{\sqrt{j!}(2\sqrt{2})^j} \cdot j^{1/4}\Big)
		=
		O\Big(\frac{w_i}{10^\ell}\Big).
	\end{align*}
	In other words, we have
	\begin{align*}
		\norm{w_if_i-\widetilde f_i}_1
		& \leq
		O\Big(\Delta \ell^{5/4} + \frac{w_i}{10^\ell}\Big).
	\end{align*}
\end{proof}

\begin{lemma}[Restated Lemma \ref{lem:det_t_numor}]\label{lem:det_t_numor:detail}
    For any $\ell\in\mathbb{N}_0$, $k\in[\ell]$ and $j\geq \ell$, we have
    \begin{align*}
        \abs{\det\big(A^{(1,k)\to j}\big)} & \leq 	\bigg(e^{\frac{1}{4}r^2}\cdot 4^{\ell}\cdot \bigg(\frac{1}{\sqrt{j!}}(\frac{r}{2\sqrt{2}})^j\bigg)^{-1}\bigg)\det(A) \qquad \text{and}\\
        \abs{\det\big(A^{(2,k)\to j}\big)} & \leq \bigg(e^{\frac{5}{4}r^2}\cdot4^\ell \cdot\bigg(\frac{1}{\sqrt{j!}}(\frac{r}{2\sqrt{2}})^j\bigg)^{-1}\bigg)\det(A).
    \end{align*}
    Hence, the absolute values of the entries of $\mathcal{E}_{t,1,j}$ are bounded by $O\big(4^\ell \cdot\big(\frac{1}{\sqrt{j!}}(\frac{r}{2\sqrt{2}})^j\big)^{-1}\big)$, i.e.
    \begin{align*}
        \abs{(\mathcal{E}_{t,1,j})_{i,k}} \leq C\cdot 4^\ell \cdot\big(\frac{1}{\sqrt{j!}}(\frac{r}{2\sqrt{2}})^j\big)^{-1} \qquad \text{for $i=1,2$ and $k\in[\ell]$.}
    \end{align*}
    Here, $C$ is an absolute constant.
\end{lemma}

\begin{proof}
We first observe that the matrices $A$, $A^{(1,k)\to j}$ and $A^{(2,k)\to j}$ can be decomposed as
\begin{align*}
    A =V^\top V,\quad
    A^{(1,k)\to j} =V^\top V^{(1,k)\to j}\quad \text{and} \quad
    A^{(2,k)\to j} =V^\top V^{(2,k)\to j}
\end{align*}
where we abuse the notation to define $V$, $V^{(1,k)\to j}$ and $V^{(2,k)\to j}$ as follows.
Let $V$ be the $\abs{\mathbb{N}_0}$-by-$2\ell$ matrix whose column indexed at $(i,k)$ is the $\abs{\mathbb{N}_0}$ dimensional vector $v^{(i,k)}$ for $i=1,2$ and $k\in[\ell]$.
Here, for $i=1,2$ and $j\in \mathbb{N}_0$, $v^{(i,j)}$ is the $\abs{\mathbb{N}_0}$ dimensional vector whose $k$-th entry  is $\inner{\psi_{j,r_i}}{\psi_{k,0}}$ for $k\in\mathbb{N}_0$.
In particular, $v^{(1,j)}$ is the zero vector except that the $j$-th entry is $1$.
For example, when $\ell=2$, 
\begin{align*}
    V = \begin{bmatrix}
        1 & 0 & \inner{ \psi_{0,r}}{ \psi_{0,0}} & \inner{ \psi_{1,r}}{ \psi_{0,0}} \\
        0 & 1 & \inner{ \psi_{0,r}}{ \psi_{1,0}} & \inner{ \psi_{1,r}}{ \psi_{1,0}} \\
        0 & 0 & \inner{ \psi_{0,r}}{ \psi_{2,0}} & \inner{ \psi_{1,r}}{ \psi_{2,0}} \\
        \vdots & \vdots & \vdots & \vdots
    \end{bmatrix}.
\end{align*}
It is easy to check that, by the orthogonality of Hermite functions, 
\begin{align*}
    V^\top V = \begin{bmatrix}
        1 & 0 & \inner{\psi_{0,r_2}}{\psi_{0,r_1}} & \inner{\psi_{1,r_2}}{\psi_{0,r_1}}\\
        0 & 1 & \inner{\psi_{0,r_2}}{\psi_{1,r_1}} &  \inner{\psi_{1,r_2}}{\psi_{1,r_1}}\\
        \inner{\psi_{0,r_1}}{\psi_{0,r_2}} & \inner{\psi_{1,r_1}}{\psi_{0,r_2}} & 1 & 0\\
        \inner{\psi_{0,r_1}}{\psi_{1,r_2}}& \inner{\psi_{1,r_1}}{\psi_{1,r_2}} & 0 & 1
    \end{bmatrix}
    =A.
\end{align*}
Define $V^{(i,k)\to j}$ to be the same matrix as $V$ except that the column indexed at $(i,k)$ is replaced with $v^{(1,j)}$ for $i=1,2$, $k\in[\ell]$ and $j\geq \ell$.
For example, when $\ell=2$, $i=1$, $k=1$ and $j=3$, 
\begin{align*}
    V^{(i,k)\to j} = \begin{bmatrix}
        1 & 0 & \inner{ \psi_{0,r}}{ \psi_{0,0}} & \inner{ \psi_{1,r}}{ \psi_{0,0}} \\
        0 & 0 & \inner{ \psi_{0,r}}{ \psi_{1,0}} & \inner{ \psi_{1,r}}{ \psi_{1,0}} \\
        0 & 0 & \inner{ \psi_{0,r}}{ \psi_{2,0}} & \inner{ \psi_{1,r}}{ \psi_{2,0}} \\
        0 & 1 & \inner{ \psi_{0,r}}{ \psi_{3,0}} & \inner{ \psi_{1,r}}{ \psi_{3,0}} \\
        \vdots & \vdots & \vdots & \vdots
    \end{bmatrix}
\end{align*}
and when $\ell=2$, $i=2$, $k=1$ and $j=3$, 
\begin{align*}
    V^{(i,k)\to j} = \begin{bmatrix}
        1 & 0 & \inner{ \psi_{0,r}}{ \psi_{0,0}} & 0 \\
        0 & 1 & \inner{ \psi_{0,r}}{ \psi_{1,0}} & 0 \\
        0 & 0 & \inner{ \psi_{0,r}}{ \psi_{2,0}} & 0 \\
        0 & 0 & \inner{ \psi_{0,r}}{ \psi_{3,0}} & 1 \\
        \vdots & \vdots & \vdots & \vdots
    \end{bmatrix}.
\end{align*}
Again, we can check that $A^{(i,k)\to j} = V^TV^{(i,k)\to j}$.

By Cauchy-Binet formula (Lemma \ref{lem:cb_formula}) and expanding the determinant along the columns indexed at $(1,k)$, we have
\begin{align*}
    \det(A) = \sum_{K\in \mathcal{K}} \det(U_K)^2
\end{align*}
where $\mathcal{K}$ is the set of subsets of $\mathbb{N}_0$ of size $\ell$ whose elements are larger than or equal to $\ell$, i.e. $\mathcal{K} = \setdef{K}{K=\{\ell\leq a_0<\cdots<a_{\ell-1}\}}$ and $U_K$ is the $\ell$-by-$\ell$ matrix whose $(b,c)$-entry is $\inner{\psi_{c,r}}{\psi_{b,0}}$ for $b\in K$ and $c\in[\ell]$.
For example, when $\ell=2$,
\begin{alignat*}{4}
    \det(A)
    & =
    \big(\det\begin{bmatrix}
        \inner{\psi_{0,r}}{\psi_{{\color{red}2},0}} & \inner{\psi_{1,r}}{\psi_{{\color{red}2},0}} \\
        \inner{\psi_{0,r}}{\psi_{{\color{red}3},0}} & \inner{\psi_{1,r}}{\psi_{{\color{red}3},0}}
    \end{bmatrix}\big)^2 &+ \big(\det\begin{bmatrix}
        \inner{\psi_{0,r}}{\psi_{{\color{red}2},0}} & \inner{\psi_{1,r}}{\psi_{{\color{red}2},0}} \\
        \inner{\psi_{0,r}}{\psi_{{\color{red}4},0}} & \inner{\psi_{1,r}}{\psi_{{\color{red}4},0}}
    \end{bmatrix}\big)^2 &+ \cdots \\
    & &+\big(\det\begin{bmatrix}
        \inner{\psi_{0,r}}{\psi_{{\color{red}3},0}} & \inner{\psi_{1,r}}{\psi_{{\color{red}3},0}} \\
        \inner{\psi_{0,r}}{\psi_{{\color{red}4},0}} & \inner{\psi_{1,r}}{\psi_{{\color{red}4},0}}
    \end{bmatrix}\big)^2 &+ \cdots \\
    & & &+ \cdots
\end{alignat*}
where the numbers in red represent the set $K$.

We first give a bound on $\abs{\det(A^{(1,k)\to j})}$.
By a similar argument, we have
\begin{align*}
    \abs{\det(A^{(1,k)\to j})} \leq \sum_{K\in \mathcal{K}_j} \abs{\det(U_{K\cup\{k\}})}\abs{\det(U_{K\cup\{j\}})}
\end{align*}
where $\mathcal{K}_j$ is the set of subsets of $\mathbb{N}_0$ of size $\ell-1$ whose elements are larger than or equal to $\ell$ and not equal to $j$, i.e. $\mathcal{K}_j = \setdef{K}{K=\{\ell\leq a_0<\cdots<a_{\ell-2}\text{ and }a_i\neq j\}}$.
For example, when $\ell=2$, $k=1$ and $j=3$, 
\begin{align*}
    \abs{\det(A^{(1,k)\to j})} 
    & \leq
    \abs{\det\begin{bmatrix}
        \inner{\psi_{0,r}}{\psi_{{\color{red}2},0}} & \inner{\psi_{1,r}}{\psi_{{\color{red}2},0}} \\
        \inner{\psi_{0,r}}{\psi_{{\color{green}1},0}} & \inner{\psi_{1,r}}{\psi_{{\color{green}1},0}}
    \end{bmatrix}}\abs{\det\begin{bmatrix}
        \inner{\psi_{0,r}}{\psi_{{\color{red}2},0}} & \inner{\psi_{1,r}}{\psi_{{\color{red}2},0}} \\
        \inner{\psi_{0,r}}{\psi_{{\color{blue}3},0}} & \inner{\psi_{1,r}}{\psi_{{\color{blue}3},0}}
    \end{bmatrix}} \\
    & \qquad + 
    \abs{\det\begin{bmatrix}
        \inner{\psi_{0,r}}{\psi_{{\color{red}4},0}} & \inner{\psi_{1,r}}{\psi_{{\color{red}4},0}} \\
        \inner{\psi_{0,r}}{\psi_{{\color{green}1},0}} & \inner{\psi_{1,r}}{\psi_{{\color{green}1},0}}
    \end{bmatrix}}\abs{\det\begin{bmatrix}
        \inner{\psi_{0,r}}{\psi_{{\color{red}4},0}} & \inner{\psi_{1,r}}{\psi_{{\color{red}4},0}} \\
        \inner{\psi_{0,r}}{\psi_{{\color{blue}3},0}} & \inner{\psi_{1,r}}{\psi_{{\color{blue}3},0}}
    \end{bmatrix}} \\
    & \qquad \qquad +
    \abs{\det\begin{bmatrix}
        \inner{\psi_{0,r}}{\psi_{{\color{red}5},0}} & \inner{\psi_{1,r}}{\psi_{{\color{red}5},0}} \\
        \inner{\psi_{0,r}}{\psi_{{\color{green}1},0}} & \inner{\psi_{1,r}}{\psi_{{\color{green}1},0}}
    \end{bmatrix}}\abs{\det\begin{bmatrix}
        \inner{\psi_{0,r}}{\psi_{{\color{red}5},0}} & \inner{\psi_{1,r}}{\psi_{{\color{red}5},0}} \\
        \inner{\psi_{0,r}}{\psi_{{\color{blue}3},0}} & \inner{\psi_{1,r}}{\psi_{{\color{blue}3},0}}
    \end{bmatrix}} \\
    &\qquad \qquad \qquad + \cdots
\end{align*}
where the numbers in red represent the set $K$, the numbers in green represent $k$ and the numbers in blue represent $j$.
Furthermore, by Cauchy–Schwarz inequality, we have
\begin{align*}
    \abs{\det(A^{(1,k)\to j})}
    & \leq
    \sqrt{\bigg(\sum_{K\in \mathcal{K}_j} \det(U_{K\cup\{k\}})^2\bigg)\bigg(\sum_{K\in \mathcal{K}_j} \det(U_{K\cup\{j\}})^2\bigg)}\\
    & \leq
    \sqrt{\bigg(\sum_{K\in \mathcal{K}_j} \det(U_{K\cup\{k\}})^2\bigg)\cdot \det(A)}. \numberthis\label{eq:det_a_1kj_sum}
\end{align*}
The last line is due to the fact that the subset $K\cup\{j\}$ is in $\mathcal{K}$ for each $K\in \mathcal{K}_j$.
By Lemma \ref{lem:det_sum_i1} below, we have
\begin{align*}
    \sum_{K\in \mathcal{K}_j} \det(U_{K\cup\{k\}})^2
    \leq
    \bigg(e^{\frac{1}{4}r^2}\cdot 4^{\ell}\cdot \bigg(\frac{1}{\sqrt{j!}}(\frac{r}{2\sqrt{2}})^j\bigg)^{-1}\bigg)^2 \cdot \det(A)
\end{align*}
and, by plugging it into \eqref{eq:det_a_1kj_sum}, we have
\begin{align*}
	\abs{\det(A^{(1,k)\to j})}
    & \leq
    e^{\frac{1}{4}r^2}\cdot 4^{\ell}\cdot \bigg(\frac{1}{\sqrt{j!}}(\frac{r}{2\sqrt{2}})^j\bigg)^{-1} \cdot \det(A).
\end{align*}

We now give a bound on $\abs{\det(A^{(2,k)\to j})}$.
By a similar argument, we have
\begin{align*}
    \abs{\det(A^{(2,k)\to j})} \leq \sum_{K\in \mathcal{K}_j} \abs{\det(U^{(-k)}_{K})}\abs{\det(U_{K\cup\{j\}})}
\end{align*}
where $U^{(-k)}_K$ is the matrix is the $(\ell-1)$-by-$(\ell-1)$ matrix whose $(b,c)$-entry is $\inner{\psi_{c,r}}{\psi_{b,0}}$ for $b\in K$ and $c\in[\ell]\backslash\{k\}$.
For example, when $\ell=2$, $k=1$ and $j=3$, 
\begin{align*}
    \abs{\det(A^{(2,k)\to j})} 
    & \leq
    \abs{\det\begin{bmatrix}
        \inner{\psi_{0,r}}{\psi_{{\color{red}2},0}} 
    \end{bmatrix}}\abs{\det\begin{bmatrix}
        \inner{\psi_{0,r}}{\psi_{{\color{red}2},0}} & \inner{\psi_{{\color{green}1},r}}{\psi_{{\color{red}2},0}} \\
        \inner{\psi_{0,r}}{\psi_{{\color{blue}3},0}} & \inner{\psi_{{\color{green}1},r}}{\psi_{{\color{blue}3},0}}
    \end{bmatrix}} \\
    & \qquad + 
    \abs{\det\begin{bmatrix}
        \inner{\psi_{0,r}}{\psi_{{\color{red}4},0}} 
    \end{bmatrix}}\abs{\det\begin{bmatrix}
        \inner{\psi_{0,r}}{\psi_{{\color{red}4},0}} & \inner{\psi_{{\color{green}1},r}}{\psi_{{\color{red}4},0}} \\
        \inner{\psi_{0,r}}{\psi_{{\color{blue}3},0}} & \inner{\psi_{{\color{green}1},r}}{\psi_{{\color{blue}3},0}}
    \end{bmatrix}} \\
    & \qquad \qquad +
    \abs{\det\begin{bmatrix}
        \inner{\psi_{0,r}}{\psi_{{\color{red}5},0}} 
    \end{bmatrix}}\abs{\det\begin{bmatrix}
        \inner{\psi_{0,r}}{\psi_{{\color{red}5},0}} & \inner{\psi_{{\color{green}1},r}}{\psi_{{\color{red}5},0}} \\
        \inner{\psi_{0,r}}{\psi_{{\color{blue}3},0}} & \inner{\psi_{{\color{green}1},r}}{\psi_{{\color{blue}3},0}}
    \end{bmatrix}} \\
    &\qquad \qquad \qquad + \cdots
\end{align*}
where the numbers in red represent the set $K$, the numbers in green represent $k$ and the numbers in blue represent $j$.
Furthermore, by Cauchy–Schwarz inequality, we have
\begin{align*}
    \abs{\det(A^{(2,k)\to j})}
    & \leq
    \sqrt{\bigg(\sum_{K\in \mathcal{K}_j} \det(U^{(-k)}_{K})^2\bigg)\bigg(\sum_{K\in \mathcal{K}_j} \det(U_{K\cup\{j\}})^2\bigg)}\\
    & \leq
    \sqrt{\bigg(\sum_{K\in \mathcal{K}_j} \det(U^{(-k)}_{K})^2\bigg) \cdot \det(A)}. \numberthis\label{eq:det_a_2kj_sum}
\end{align*}
The last line is due to the fact that the subset $K\cup\{j\}$ is in $\mathcal{K}$ for each $K\in \mathcal{K}_j$.
By Lemma \ref{lem:det_sum_i2} below, we have
\begin{align*}
    \sum_{K\in \mathcal{K}_j} \det(U^{(-k)}_{K})^2
    \leq
    \bigg(e^{\frac{5}{4}r^2}\cdot 4^{\ell}\cdot \bigg(\frac{1}{\sqrt{j!}}(\frac{r}{2\sqrt{2}})^j\bigg)^{-1}\bigg)^2 \cdot \det(A)
\end{align*}
and, by plugging it into \eqref{eq:det_a_2kj_sum}, we have
\begin{align*}
	\abs{\det(A^{(2,k)\to j})}
    & \leq
    e^{\frac{5}{4}r^2}\cdot 4^{\ell}\cdot \bigg(\frac{1}{\sqrt{j!}}(\frac{r}{2\sqrt{2}})^j\bigg)^{-1} \cdot \det(A).
\end{align*}

\end{proof}

\begin{lemma}[Restated Lemma \ref{lem:det_a_numor}]\label{lem:det_a_numor:detail}
    For any $\ell\in\mathbb{N}_0$, $i=1,2$ and $k\in[\ell]$, we have
    \begin{align*}
        \abs{\det\big(A^{(i,k)\to \Delta}\big)} < 2^{O(\ell\log \ell)} \cdot \norm{\Delta f}_2\cdot\det(A).
    \end{align*}
\end{lemma}

\begin{proof}
We first observe that the matrix $A^{(i,k)\to \Delta}$ can be decomposed as 
\begin{align*}
    A^{(i,k)\to \Delta} = V^\top V^{(i,k)\to \Delta}
\end{align*}
where $V^{(i,k)\to \Delta}$ is the $\abs{\mathbb{N}_0}$-by-$2\ell$ matrix whose column indexed at $(i,k)$ is replaced with the $\abs{\mathbb{N}_0}$ dimensional vector $v^{\Delta}$ for $i=1,2$ and $k\in[\ell]$.
Here, $v^{\Delta}$ is the  $\abs{\mathbb{N}_0}$ dimensional vector whose $k$-th entry is $\inner{\Delta f}{\psi_{k,0}}$ for $k\in\mathbb{N}_0$.
Recall that $\Delta f-f'-f$.
For example, when $\ell=2$, $i=1$ and $k=1$,
\begin{align*}
    V^{(i,k)\to \Delta} = \begin{bmatrix}
        1 & \inner{\Delta f}{\psi_{k,0}} & \inner{ \psi_{0,r}}{ \psi_{0,0}} & \inner{ \psi_{1,r}}{ \psi_{0,0}} \\
        0 & \inner{\Delta f}{\psi_{k,1}} & \inner{ \psi_{0,r}}{ \psi_{1,0}} & \inner{ \psi_{1,r}}{ \psi_{1,0}} \\
        0 & \inner{\Delta f}{\psi_{k,2}} & \inner{ \psi_{0,r}}{ \psi_{2,0}} & \inner{ \psi_{1,r}}{ \psi_{2,0}} \\
        0 & \inner{\Delta f}{\psi_{k,3}} & \inner{ \psi_{0,r}}{ \psi_{3,0}} & \inner{ \psi_{1,r}}{ \psi_{3,0}} \\
        \vdots & \vdots & \vdots & \vdots
    \end{bmatrix}.
\end{align*}
Recall that $V$ is the $\abs{\mathbb{N}_0}$-by-$2\ell$ matrix whose column indexed at $(i,k)$ is the $\abs{\mathbb{N}_0}$ dimensional vector $v^{(i,k)}$ for $i=1,2$ and $k\in[\ell]$.
Here, for $i=1,2$ and $j\in \mathbb{N}_0$, $v^{(i,j)}$ is the $\abs{\mathbb{N}_0}$ dimensional vector whose $k$-th entry  is $\inner{\psi_{j,r_i}}{\psi_{k,0}}$ for $k\in\mathbb{N}_0$.

We first give a bound on $\abs{\det(A^{(1,k)\to \Delta})}$.
By Cauchy-Binet formula (Lemma \ref{lem:cb_formula}) and expanding the determinant along the columns with a single $1$, we have
\begin{align*}
    \abs{\det(A^{(1,k)\to \Delta})}
    & \leq
    \sum_{K\in \mathcal{K}} \abs{\det(U_K)}\abs{\det(U^{\Delta}_K)}
\end{align*}
where $U^{\Delta}_K$ is the $(\ell+1)$-by-$(\ell+1)$ whose $(b,c)$-entry is $\begin{cases}
    \inner{\psi_{c,r}}{\psi_{b,0}} \quad \text{if $c\in[\ell]$}\\
    \inner{\Delta f}{\psi_{b,0}} \quad \text{if $c=\Delta f$}
\end{cases}$ for $b\in K$ and $c\in[\ell]\cup\{\Delta f\}$.
Furthermore, by Cauchy–Schwarz inequality, we have
\begin{align*}
    \abs{\det(A^{(1,k)\to \Delta})}
    & \leq
    \sqrt{\bigg(\sum_{K\in \mathcal{K}} \det(U_K)^2\bigg)\bigg(\sum_{K\in \mathcal{K}} \det(U^{\Delta}_K)^2\bigg)}\\
    & \leq
    \sqrt{\bigg(\sum_{K\in \mathcal{K}} \det(U^{\Delta}_K)^2\bigg) \cdot \det(A)}. \numberthis\label{eq:det_a_1kd_sum}
\end{align*}
For each $K=\{\ell\leq a_0 < \cdots<a_{\ell-1}\}\in\mathcal{K}$, we first expand the determinant $\det(U^{\Delta}_K)$ along the column indexed at $\Delta f$.
\begin{align*}
    \det(U^{\Delta}_K)^2 
    & \leq
    \bigg(\abs{\inner{\Delta f}{\psi_{k,0}}}\abs{\det(U_K)} + \sum_{c=0}^{\ell-1}\abs{\inner{\Delta f}{\psi_{a_{c},0}}}\abs{\det(U_{K\cup\{k\}\backslash\{a_c\}})}\bigg)^2 \\
    & \leq
    (\ell+1)\bigg(\inner{\Delta f}{\psi_{k,0}}^2\abs{\det(U_K)}^2 + \sum_{c=0}^{\ell-1}\inner{\Delta f}{\psi_{a_{c},0}}^2\abs{\det(U_{K\cup\{k\}\backslash\{a_c\}})}^2\bigg).
\end{align*}
We now consider the summation $\sum_{K\in \mathcal{K}}\sum_{c=0}^{\ell-1} \inner{\Delta f}{\psi_{a_c,0}}^2\abs{\det(U_{K\cup\{k\}\backslash\{a_c\}})}^2$ and we have
\begin{align*}
    \sum_{K\in \mathcal{K}}\sum_{c=0}^{\ell-1} \inner{\Delta f}{\psi_{a_c,0}}^2\abs{\det(U_{K\cup\{k\}\backslash\{a_c\}})}^2 
    & =
    \sum_{K\in \mathcal{K}'}\bigg(\sum_{j\geq \ell,j\notin K} \inner{\Delta f}{\psi_{j,0}}^2\bigg)\abs{\det(U_{K\cup\{k\}})}^2 \\
    & \leq
    \norm{\Delta f}^2\sum_{K\in \mathcal{K}'}\abs{\det(U_{K\cup\{k\}})}^2
\end{align*}
where $\mathcal{K}'$ is the set of subsets of $\mathbb{N}_0$ of size $\ell-1$ whose elements are larger than or equal to $\ell$, i.e. $\mathcal{K}' = \setdef{K}{K=\{\ell\leq a_0<\cdots <a_{\ell-2}\}}$.
Since each $K\in \mathcal{K}'$ has $\ell-1$ elements, then $\{\ell,\cdots,2\ell-1\} \not\subseteq K$. 
We have
\begin{align*}
    \sum_{K\in \mathcal{K}'}\abs{\det(U_{K\cup\{k\}})}^2 
    & \leq
    \sum_{j=\ell}^{2\ell-1}\sum_{K\in \mathcal{K}_{j}}\abs{\det(U_{K\cup\{k\}})}^2.
\end{align*}
Recall that $\mathcal{K}_j$ is the set of subsets of $\mathbb{N}_0$ of size $\ell-1$ whose elements are larger than or equal to $\ell$ and not equal to $j$, i.e. $\mathcal{K}_j = \setdef{K}{K=\{\ell\leq a_0<\cdots<a_{\ell-2}\text{ and }a_i\neq j\}}$.
By  Lemma \ref{lem:det_sum_i1}, we have
\begin{align*}
	\sum_{K\in \mathcal{K}'}\abs{\det(U_{K\cup\{k\}})}^2
	& \leq
	\sum_{j=\ell}^{2\ell-1}\bigg(e^{\frac{1}{4}r^2}\cdot 4^{\ell}\cdot \bigg(\frac{1}{\sqrt{j!}}(\frac{r}{2\sqrt{2}})^j\bigg)^{-1}\bigg)^2\det\big(A\big)  \\
	& =
	2^{O(\ell\log \ell)}\cdot\det\big(A\big)
\end{align*}
which means
\begin{align*}
    \sum_{K\in \mathcal{K}} \det(U^{\Delta}_K)^2
    & \leq
    2^{O(\ell\log \ell)}\cdot \norm{\Delta f}^2 \cdot\det(A).
\end{align*}
By plugging it into \eqref{eq:det_a_1kd_sum}, we have
\begin{align*}
    \abs{\det(A^{(1,k)\to \Delta})}
    & \leq
    2^{O(\ell\log \ell)}\cdot \norm{\Delta f} \cdot\det(A).
\end{align*}

We now give a bound on $\abs{\det(A^{(2,k)\to \Delta})}$.
By a similar argument, we have
\begin{align*}
    \abs{\det(A^{(2,k)\to \Delta})}
    & \leq
    \sum_{K\in \mathcal{K}} \abs{\det(U_K)}\abs{\det(U^{k\to \Delta}_K)}
\end{align*}
where, for any $k\in[\ell]$ $U^{k\to\Delta}_K$ is the $\ell$-by-$\ell$ whose $(b,c)$-entry is $\begin{cases}
    \inner{\psi_{c,r}}{\psi_{b,0}} \quad \text{if $c\in[\ell]$}\\
    \inner{\Delta f}{\psi_{b,0}} \quad \text{if $c=\Delta f$}
\end{cases}$ for $b\in K$ and $c\in[\ell]\cup\{\Delta f\}\backslash\{k\}$.
Furthermore, by Cauchy–Schwarz inequality, we have
\begin{align*}
    \abs{\det(A^{(2,k)\to \Delta})}
    & \leq
    \sqrt{\bigg(\sum_{K\in \mathcal{K}} \det(U_K)^2\bigg)\bigg(\sum_{K\in \mathcal{K}} \det(U^{k\to\Delta}_K)^2\bigg)}\\
    & \leq
    \sqrt{\sum_{K\in \mathcal{K}} \det(U^{k\to\Delta}_K)^2 \cdot \det(A)}. \numberthis\label{eq:det_a_2kd_sum}
\end{align*}
For each $K=\{\ell\leq a_0<\cdots<a_{\ell-1}\}\in \mathcal{K}$, we expand the determinants along the column indexed at $\Delta f$.
\begin{align*}
    \det(U^{k\to\Delta}_K)^2
    & \leq
    \bigg(\sum_{c=0}^{\ell-1}\abs{\inner{\Delta f}{\psi_{a_c,0}}}\abs{\det(U^{(-k)}_{K\backslash\{a_c\}})}\bigg)^2 \\
    & \leq
    \ell\cdot \bigg(\sum_{c=0}^{\ell-1}\abs{\inner{\Delta f}{\psi_{a_c,0}}}^2\abs{\det(U^{(-k)}_{K\backslash\{a_c\}})}^2\bigg)
\end{align*}
We now consider the summation $\sum_{K\in\mathcal{K}}\sum_{c=0}^{\ell-1}\abs{\inner{\Delta f}{\psi_{a_c,0}}}^2\abs{\det(U^{(-k)}_{K\backslash\{a_c\}})}^2$ and we have
	\begin{align*}
		\sum_{K\in\mathcal{K}}\sum_{c=0}^{\ell-1}\abs{\inner{\Delta f}{\psi_{a_c,0}}}^2\abs{\det(U^{(-k)}_{K\backslash\{a_c\}})}^2
		& =
		\sum_{K\in\mathcal{K}'}\bigg(\sum_{j\geq \ell,j\notin K} \inner{\Delta f}{ \psi_{j,0}}^2\bigg)\abs{\det(U^{(-k)}_{K})}^2 \\
		&\leq
		\norm{\Delta f}_2^2\sum_{K\in\mathcal{K}'}\abs{\det(U^{(-k)}_{K})}^2.
	\end{align*}
where $\mathcal{K}'$ is the set of subsets of $\mathbb{N}_0$ of size $\ell-1$ whose elements are larger than or equal to $\ell$, i.e. $\mathcal{K}' = \setdef{K}{K=\{\ell\leq a_0<\cdots <a_{\ell-2}\}}$.
Since each $K\in\mathcal{K}'$ has $\ell-1$ elements, then $\{\ell,\cdots,2\ell-1\}\not\subseteq K$.
We have
\begin{align*}
	\sum_{K\in\mathcal{K}'}\abs{\det(U^{(-k)}_{K})}^2
	\leq
	\sum_{j=\ell}^{2\ell-1}\sum_{K\in\mathcal{K}_j}\abs{\det(U^{(-k)}_{K})}^2.
\end{align*}
Recall that $\mathcal{K}_j$ is the set of subsets of $\mathbb{N}_0$ of size $\ell-1$ whose elements are larger than or equal to $\ell$ and not equal to $j$, i.e. $\mathcal{K}_j = \setdef{K}{K=\{\ell\leq a_0<\cdots<a_{\ell-2}\text{ and }a_i\neq j\}}$.
By Lemma \ref{lem:det_sum_i2}, we have
\begin{align*}
	\sum_{K\in\mathcal{K}'}\abs{\det(U^{(-k)}_{K})}^2
	& \leq
	\sum_{j=\ell}^{2\ell-1}\bigg(e^{\frac{5}{4}r^2}\cdot4^\ell \cdot\bigg(\frac{1}{\sqrt{j!}}(\frac{r}{2\sqrt{2}})^j\bigg)^{-1}\bigg)^2\det(A) \\
	& =
	2^{O(\ell\log \ell)}\cdot\det(A)
\end{align*}
which means
\begin{align*}
    \sum_{K\in \mathcal{K}} \det(U^{k\to\Delta}_K)^2
    & \leq
    2^{O(\ell\log \ell)}\cdot\norm{\Delta f}^2\cdot\det(A).
\end{align*}
By plugging it into \eqref{eq:det_a_2kd_sum}
\begin{align*}
	\abs{\det(A^{(2,k)\to \Delta})}
	& \leq
	2^{O(\ell\log \ell)} \cdot \norm{\Delta f}_2\cdot\det(A).
\end{align*}

\end{proof}

\begin{lemma}\label{lem:det_sum_i1}
    For any $\ell\geq 1$, $k\in[\ell]$ and $j\geq \ell$, we have
	\begin{align*}
	    \sum_{K\in \mathcal{K}_j} \det(U_{K\cup\{k\}})^2
    \leq
    \bigg(e^{\frac{1}{4}r^2}\cdot 4^{\ell}\cdot \bigg(\frac{1}{\sqrt{j!}}(\frac{r}{2\sqrt{2}})^j\bigg)^{-1}\bigg)^2 \cdot \det(A).
	\end{align*}
	Recall that, for any $K=\{\ell \leq a_0<\cdots<a_{\ell-2}\text{ and }a_i\neq j\}\in\mathcal{K}_j$ and $k\in[\ell]$, $U_{K\cup\{k\}}$ is the $\ell$-by-$\ell$ matrix whose $(b,c)$-entry is $\inner{\psi_{c,r}}{\psi_{b,0}}$ for $b\in K\cup\{k\}$ and $c\in[\ell]$.
\end{lemma}

\begin{proof}
	For each $K=\{\ell \leq a_0<\cdots<a_{\ell-2}\text{ and }a_i\neq j\}\in\mathcal{K}_j$, we have the following.
	By Lemma \ref{lem:det_part}, we have
	\begin{align*}
		\frac{\abs{\det\big(U_{K\cup\{k\}}\big) }}{\abs{\det\big(U_{K\cup\{j\}}\big) }} \leq \sqrt{\frac{j!}{k!}}(\frac{r}{\sqrt{2}})^{k-j}\bigg(\prod_{c=0}^{\ell-2}\frac{\abs{a_c-k}}{\abs{a_c-j}}\bigg)
	\end{align*}
	By Lemma~\ref{lem:ratio}, we have $\prod_{c=0}^{\ell-2}\frac{\abs{a_c-k}}{\abs{a_c-j}} \leq \prod_{c=0}^{\ell-2}\frac{\abs{a_c}}{\abs{a_c-j}} \leq 2^j\cdot 4^\ell$ and hence
	\begin{align*}
		\abs{\det\big(U_{K\cup\{k\}}\big) } 
		& \leq 
		4^{\ell}\cdot2^j\sqrt{\frac{j!}{k!}}(\frac{r}{\sqrt{2}})^{k-j}\abs{\det\big(U_{K\cup\{j\}}\big) } \\
		&\leq
		e^{\frac{1}{4}r^2}\cdot 4^{\ell}\cdot \bigg(\frac{1}{\sqrt{j!}}(\frac{r}{2\sqrt{2}})^j\bigg)^{-1}\abs{\det\big(U_{K\cup\{j\}}\big)} 
	\end{align*}
	since $\frac{1}{\sqrt{k!}}(\frac{r}{\sqrt{2}})^k \leq \sqrt{\sum_{i=0}^\infty\frac{1}{i!}(\frac{r^2}{2})^i} = e^{\frac{1}{4}r^2}$.
	For each $K\in\mathcal{K}_{j}$, the set $K\cup \{j\}$ is in $\mathcal{K}$ since $j\notin K$.
	Hence, we conclude that 
	\begin{align*}
		\sum_{K\in \mathcal{K}_j} \det(U_{K\cup\{k\}})^2
		& \leq
		\bigg(e^{\frac{1}{4}r^2}\cdot 4^{\ell}\cdot \bigg(\frac{1}{\sqrt{j!}}(\frac{r}{2\sqrt{2}})^j\bigg)^{-1}\bigg)^2\sum_{K\in \mathcal{K}_j} \det(U_{K\cup\{j\}})^2 \\
		& \leq
		\bigg(e^{\frac{1}{4}r^2}\cdot 4^{\ell}\cdot \bigg(\frac{1}{\sqrt{j!}}(\frac{r}{2\sqrt{2}})^j\bigg)^{-1}\bigg)^2\det(A) .
	\end{align*}

\end{proof}

\begin{lemma}\label{lem:det_sum_i2}
    For any $\ell\geq 1$, $k\in[\ell]$ and $j\geq \ell$, we have
	\begin{align*}
    	\sum_{K\in \mathcal{K}_j} \det(U^{(-k)}_{K})^2
    \leq
    \bigg(e^{\frac{5}{4}r^2}\cdot 4^{\ell}\cdot \bigg(\frac{1}{\sqrt{j!}}(\frac{r}{2\sqrt{2}})^j\bigg)^{-1}\bigg)^2 \cdot \det(A).
	\end{align*}
	Recall that, for any $K=\{\ell \leq a_0<\cdots<a_{\ell-2}\text{ and }a_i\neq j\}\in\mathcal{K}_j$ and $k\in[\ell]$, $U^{(-k)}_K$ is the $(\ell-1)$-by-$(\ell-1)$ matrix whose $(b,c)$-entry is $\inner{\psi_{c,r}}{\psi_{b,0}}$ for $b\in K$ and $c\in[\ell]\backslash\{k\}$.
\end{lemma}

\begin{proof}
	For each $K=\{\ell \leq a_0<\cdots<a_{\ell-2}\text{ and }a_i\neq j\}\in\mathcal{K}_j$, we have the following.
	By Lemma \ref{lem:det_part} and \ref{lem:det_part_2}, we have
	\begin{align*}
		\frac{\abs{\det(U^{(-k)}_{K})}}{\abs{\det(U_{K\cup\{j\}}) }}
		& \leq
		e^{\frac{1}{4}r^2}\cdot 2^k \cdot \sqrt{\frac{j!}{k!}}(\frac{r}{\sqrt{2}})^{k-j}\bigg(\prod_{c=0}^{\ell-2}\frac{\abs{a_c}}{\abs{a_c-j}}\bigg)
	\end{align*}
	and, by Lemma \ref{lem:ratio}, we have
	\begin{align*}
		\frac{\abs{\det(U^{(-k)}_{K})}}{\abs{\det(U_{K\cup\{j\}}) }}
		& \leq
		2^j \cdot 4^{\ell} \cdot e^{\frac{1}{4}r^2}\cdot 2^k \cdot \sqrt{\frac{j!}{k!}}(\frac{r}{\sqrt{2}})^{k-j} \\
		& \leq
		e^{\frac{5}{4}r^2}\cdot4^\ell \cdot\bigg(\frac{1}{\sqrt{j!}}(\frac{r}{2\sqrt{2}})^j\bigg)^{-1}
	\end{align*}
	since $\frac{1}{\sqrt{k!}}(\sqrt{2}r)^k \leq \sqrt{\sum_{i=0}^\infty \frac{1}{i!}(2r^2)^i} \leq e^{r^2}$.
	For each $K\in\mathcal{K}_j$, the set  $K\cup\{j\}$ is  in $\mathcal{K}$ since $j\notin K$. 
	Hence, 
	\begin{align*}
		\sum_{K\in \mathcal{K}_j} \det(U^{(-k)}_{K})^2
		& \leq
		\bigg(e^{\frac{5}{4}r^2}\cdot4^\ell \cdot\bigg(\frac{1}{\sqrt{j!}}(\frac{r}{2\sqrt{2}})^j\bigg)^{-1}\bigg)^2\sum_{K\in \mathcal{K}_j} \det(U_{K\cup\{j\}})^2 \\
		& \leq
		\bigg(e^{\frac{5}{4}r^2}\cdot4^\ell \cdot\bigg(\frac{1}{\sqrt{j!}}(\frac{r}{2\sqrt{2}})^j\bigg)^{-1}\bigg)^2\det(A) .
	\end{align*}
\end{proof}

Before we show the lemmas below, we first define the following notations to simplify the expressions in our proof.
For any $r\in \mathbb{R}$, let $W_r$ be the double-indexed sequence such that
\begin{align*}
    (W_r)_{s,t} = e^{-\frac{1}{8}r^2}(-1)^s\sqrt{\frac{t!}{s!}}{s \choose t}(\frac{r}{\sqrt{2}})^{s-t}
\end{align*}
for $s,t \in \mathbb{N}_0$.
Here, ${s \choose t}$ is the binomial coefficient which is equal to $\begin{cases}
    0 & \text{if $s<t$}\\
    \frac{s!}{t!(s-t)!} & \text{if $s\geq t$}
\end{cases}$.
For any subsets $S,T\subset \mathbb{N}_0$ of same sizes, let $(W_r)_{S,T}$ be the matrix whose $(s,t)$-entry is $(W_r)_{s,t}$ for $s\in S$ and $t\in T$.
For any (ordered) set $K$ of $m$ nonnegative integers $a_0,\cdots,a_{m-1}$ such that $0\leq a_0<\cdots<a_{m-1}$, we define 
\begin{align*}
	\Sigma_K = \sum_{c=0}^{m-1}a_c,\qquad F_K = \prod_{c=0}^{m-1} a_c! ,\qquad C_K = \prod_{0\leq c_2<c_1\leq m-1} (a_{c_1}-a_{c_2}).
\end{align*}
Finally, for any $b\in\mathbb{N}_{0}$, we define
\begin{align*}
	\Gamma_{K,b} = \frac{1}{b!}\sum_{d=0}^b (-1)^d{b \choose d}\prod_{c=0}^{m-1}(a_c-d).
\end{align*}

\begin{lemma}\label{lem:det_part}
	Let $a_0,\cdots,a_{\ell-1}$ be $\ell$ nonnegative integers and $K$ be the set $\{0\leq a_0<\cdots<a_{\ell-1}\}$.
	Then, we have
	\begin{align*}
		\abs{\det(U_K)} = e^{-\frac{\ell}{4}r^2}\sqrt{\frac{1}{F_{[\ell]}F_K}}(\frac{r}{\sqrt{2}})^{\Sigma_K-\ell(\ell-1)/2}C_{K}.
	\end{align*}
	Recall that, for any $K=\{0\leq a_0<\cdots<a_{\ell-1}\}$, $U_{K}$ is the $\ell$-by-$\ell$ matrix whose $(b,c)$-entry is $\inner{\psi_{c,r}}{\psi_{b,0}}$ for $b\in K$ and $c\in[\ell]$.
\end{lemma}

\begin{proof}
Recall that, from \eqref{eq:hermite_inner}, 
\begin{align*}
    \inner{ \psi_{i,0}}{ \psi_{j,r}} 
	& =
	\sum_{k=0}^{\min\{i,j\}}\bigg(e^{-\frac{1}{8}{r}^2}(-1)^{i}\sqrt{\frac{k!}{i!}}{i \choose k} (\frac{r}{\sqrt{2}})^{i-k}\bigg)\bigg(e^{-\frac{1}{8}{r}^2}(-1)^{j}\sqrt{\frac{k!}{j!}}{j\choose k} (\frac{-{r}}{\sqrt{2}})^{j-k}\bigg).
\end{align*}
It means that $U_K$ can be decomposed as 
	\begin{align*}
		U_K = (W_r)_{K,[\ell]}(W_{-r} )_{[\ell],[\ell]}^\top.
	\end{align*}
By Lemma \ref{lem:det_w}, the determinant of $(W_{-r}^\top )_{[\ell],[\ell]}$ is 
\begin{align*}
	e^{-\frac{\ell}{8}r^2}(-1)^{\Sigma_{[\ell]}}\sqrt{\frac{1}{F_{[\ell]}F_{[\ell]}}}(-\frac{r}{\sqrt{2}})^{\Sigma_{[\ell]}-\ell(\ell-1)/2}C_{[\ell]} = e^{-\frac{\ell}{8}r^2}(-1)^{\ell(\ell-1)/2}
\end{align*}
since $\Sigma_{[\ell]}=\ell(\ell-1)/2$ and $C_{[\ell]} = F_{[\ell]}$.
Also, the determinant of $(W_r)_{K,[\ell]}$ is 
\begin{align*}
	e^{-\frac{\ell}{8}r^2}(-1)^{\Sigma_{K}}\sqrt{\frac{1}{F_{[\ell]}F_K}}(\frac{r}{\sqrt{2}})^{\Sigma_K-\ell(\ell-1)/2}C_{K}.
\end{align*}
Hence, we have
\begin{align*}
	\det(U_K)
	& =
	e^{-\frac{\ell}{8}r^2}(-1)^{\ell(\ell-1)/2} \cdot e^{-\frac{\ell}{8}r^2}(-1)^{\Sigma_{K}}\sqrt{\frac{1}{F_{[\ell]}F_K}}(\frac{r}{\sqrt{2}})^{\Sigma_K-\ell(\ell-1)/2}C_{K} \\
	& =
	e^{-\frac{\ell}{4}r^2}(-1)^{\Sigma_{K}+\ell(\ell-1)/2}\sqrt{\frac{1}{F_{[\ell]}F_K}}(\frac{r}{\sqrt{2}})^{\Sigma_K-\ell(\ell-1)/2}C_{K}.
\end{align*}
\end{proof}

\begin{lemma}\label{lem:det_part_2}
	Let $a_0,\cdots,a_{\ell-2}$ be $\ell-1$ nonnegative integers and $K$ be the set $\{0\leq a_0<\cdots<a_{\ell-2}\}$.
	Then, for any $k\in[\ell]$, we have
	\begin{align*}
		\abs{\det(U^{(-k)}_{K})} \leq e^{-\frac{\ell-1}{4}r^2}2^{k}\sqrt{\frac{1}{k!F_{[\ell]}F_K}}(\frac{r}{\sqrt{2}})^{\Sigma_K-\ell(\ell-1)/2+k}C_{K} \prod_{c=0}^{\ell-2}a_c.
	\end{align*}
	Recall that, for any $K=\{0\leq a_0<\cdots<a_{\ell-2}\}\in\mathcal{K}$ and $k\in[\ell]$, $U^{(-k)}_K$ is the $(\ell-1)$-by-$(\ell-1)$ matrix whose $(b,c)$-entry is $\inner{\psi_{c,r}}{\psi_{b,0}}$ for $b\in K$ and $c\in[\ell]\backslash\{k\}$.
\end{lemma}

\begin{proof}
Recall that, from \eqref{eq:hermite_inner}, 
\begin{align*}
    \inner{ \psi_{i,0}}{ \psi_{j,r}} 
	& =
	\sum_{k=0}^{\min\{i,j\}}\bigg(e^{-\frac{1}{8}{r}^2}(-1)^{i}\sqrt{\frac{k!}{i!}}{i \choose k} (\frac{r}{\sqrt{2}})^{i-k}\bigg)\bigg(e^{-\frac{1}{8}{r}^2}(-1)^{j}\sqrt{\frac{k!}{j!}}{j\choose k} (\frac{-{r}}{\sqrt{2}})^{j-k}\bigg).
\end{align*}
It means that  $U^{(-k)}_{K}$ can be decomposed as 
	\begin{align*}
		U^{(-k)}_{K} =  (W_r)_{K,[\ell]}(W_{-r} )_{[\ell],[\ell]\backslash\{k\}}^\top
	\end{align*}
	By Cauchy–Binet formula (Lemma \ref{lem:cb_formula}), we have
	\begin{align*}
		\det\big(U^{(-k)}_{K}\big)  = \sum_{b=0}^{\ell-1} (\det\big((W_r)_{K,[\ell]\backslash\{b\}}\big) )(\det\big((W_{-r} )^\top_{[\ell]\backslash\{b\},[\ell]\backslash\{k\}}\big) )\numberthis \label{eq:cbformula_2}
	\end{align*}
	By Lemma~\ref{lem:det_w}, the determinant of $(W_{-r}^\top )_{[\ell]\backslash\{b\},[\ell]\backslash\{k\}}$ is 
	\begin{align*}
		\MoveEqLeft e^{-\frac{\ell-1}{8}r^2}(-1)^{\ell(\ell-1)/2-b}\sqrt{\frac{k!}{F_{[\ell]}F_{[\ell]\backslash\{b\}}}}(-\frac{r}{\sqrt{2}})^{\ell(\ell-1)/2-b-\ell(\ell-1)/2+k}C_{[\ell]\backslash\{b\}}\cdot \Gamma_{[\ell]\backslash\{b\},k} \\
		& =
		\begin{cases}
			0 &\text{if $k<b$}\\
			e^{-\frac{\ell-1}{8}r^2}(-1)^{\ell(\ell-1)/2+k}\frac{\sqrt{k!b!}}{(k-b)!b!}(\frac{r}{\sqrt{2}})^{k-b} & \text{if $k\geq b$}
		\end{cases}
	\end{align*}
	since $C_{[\ell]\backslash\{b\}}=\frac{F_{[\ell]}}{(\ell-1-b)!b!}$ and $\Gamma_{[\ell]\backslash\{b\},k}=\begin{cases}
		0 &\text{if $k<b$}\\
		\frac{(\ell-1-b)!}{(k-b)!} & \text{if $k\geq b$}
	\end{cases}$.
	Also, the determinant of $(W_r)_{K,[\ell]\backslash\{b\}}$ is
	\begin{align*}
		e^{-\frac{\ell-1}{8}r^2}(-1)^{\Sigma_{K}}\sqrt{\frac{b!}{F_{[\ell]}F_K}}(\frac{r}{\sqrt{2}})^{\Sigma_K-\ell(\ell-1)/2+b}C_{K}\cdot \Gamma_{K,b}.
	\end{align*}
	Hence, when $b\leq k$, we have
	\begin{align*}
		\MoveEqLeft \abs{(\det\big((W_r)_{K,[\ell]\backslash\{b\}}\big) )(\det\big((W_{-r}^\top )_{[\ell]\backslash\{b\},[\ell]\backslash\{k\}})\big) } \\
		& =
		e^{-\frac{\ell-1}{4}r^2}\frac{1}{(k-b)!}\sqrt{\frac{k!}{F_{[\ell]}F_K}}(\frac{r}{\sqrt{2}})^{\Sigma_K-\ell(\ell-1)/2+k}C_{K}\cdot \Gamma_{K,b} \\
		& \leq
		e^{-\frac{\ell-1}{4}r^2}\frac{1}{(k-b)!b!}\sqrt{\frac{k!}{F_{[\ell]}F_K}}(\frac{r}{\sqrt{2}})^{\Sigma_K-\ell(\ell-1)/2+k}C_{K} \prod_{c=0}^{\ell-2}a_c.
	\end{align*}
	The last inequality is due to Lemma \ref{lem:gamma}.
	By plugging it into~\eqref{eq:cbformula_2}, we have
	\begin{align*}
		\abs{\det\big(U^{(-k)}_{K}\big)} 
		& \leq
		\sum_{b=0}^{k} \abs{\det\big((W_r)_{K,[\ell]\backslash\{b\}}\big) \det\big((W_{-r}^\top )_{[\ell]\backslash\{b\},[\ell]\backslash\{k\}}\big) } \\
		& \leq
		\sum_{b=0}^{k} e^{-\frac{\ell-1}{4}r^2}\frac{1}{(k-b)!b!}\sqrt{\frac{k!}{F_{[\ell]}F_K}}(\frac{r}{\sqrt{2}})^{\Sigma_K-\ell(\ell-1)/2+k}C_{K} \prod_{c=0}^{\ell-2}a_c.
	\end{align*}
	Since $
	\sum_{b=0}^k \frac{1}{(k-b)!b!} 
	=
	\frac{2^{k}}{k!}$, we have
	\begin{align*}
		\abs{\det\big(U^{(-k)}_{K}\big)} 
		& \leq
		e^{-\frac{\ell-1}{4}r^2}2^{k}\sqrt{\frac{1}{k!F_{[\ell]}F_K}}(\frac{r}{\sqrt{2}})^{\Sigma_K-\ell(\ell-1)/2+k}C_{K} \prod_{c=0}^{\ell-2}a_c.
	\end{align*}
\end{proof}

\begin{lemma}\label{lem:ratio}
	Let $a_0,\cdots,a_{\ell-1}$ be $\ell$ nonnegative integers such that $0\leq a_0<\cdots<a_{\ell-1}$ and $a_c\neq j$ for some $j\in\mathbb{N}_{0}$.
	Then, we have
	\begin{align*}
		\prod_{c=0}^{\ell-1}\abs{\frac{a_c}{a_c-j}} \leq 2^j\cdot 4^\ell.
	\end{align*}
\end{lemma}

\begin{proof}
	Suppose $a_0,\cdots,a_{b-1}$ are the integers less than $j$ and the rest of them are larger than $j$.
	For $c\leq b-1$, we have
	\begin{align*}
		\abs{\frac{a_c}{a_c-j}} = \frac{a_c}{j-a_c} = \frac{j}{j-a_c}-1 \leq \frac{j}{b-c}-1 = \frac{j-(b-c)}{b-c}
	\end{align*}
	since $a_c \leq j-(b-c)$.
	For $c\geq b$,  we have
	\begin{align*}
		\abs{\frac{a_c}{a_c-j}} = \frac{a_c}{a_c-j} = 1+\frac{j}{a_c-j} \leq 1+ \frac{j}{c-b+1} = \frac{j+1+(c-b)}{c-b+1}
	\end{align*}
	since $a_c \geq j+1+(c-b)$.
	Hence, we have
	\begin{align*}
		\prod_{c=0}^{\ell-1}\abs{\frac{a_c}{a_c-j}} 
		& \leq 
		\prod_{c=0}^{b-1}\frac{j-(b-c)}{b-c} \cdot \prod_{c=b}^{\ell-1}\frac{j+1+(c-b)}{c-b+1} \\
		& =
		\frac{(j-1)!}{(j-b-1)!b!}\cdot \frac{(j+\ell-b)!}{(\ell-b)!j!} \\
		& \leq
		\frac{(j+\ell-b)!}{(\ell-b)!b!(j-b)!}
		=
		\frac{(j+\ell-b)!}{\ell!(j-b)!}\cdot\frac{\ell!}{(\ell-b)!b!} \\
		& \leq
		2^{j+\ell-b}\cdot 2^\ell
		\leq
		2^j\cdot 4^\ell.
	\end{align*}
\end{proof}

\begin{lemma}\label{lem:det_w}
	Let $a_0,\cdots,a_{\ell-1}$ be $\ell$ nonnegative integers such that $0\leq a_0<\cdots<a_{\ell-1}$ and $K$ be the ordered set $\{a_0,\cdots,a_{\ell-1}\}$.
	Then, for any integer $b\in[\ell+1]$, we have
	\begin{align*}
		\det \big((W_r)_{K,[\ell+1]\backslash\{b\}}\big) = e^{-\frac{\ell}{8}r^2}(-1)^{\Sigma_{K}}\sqrt{\frac{b!}{F_{[\ell+1]}F_K}}(\frac{r}{\sqrt{2}})^{\Sigma_K-\ell(\ell+1)/2+b}C_{K}\cdot \Gamma_{K,b}.
	\end{align*}
	In particular, when $b=\ell$, we have
	\begin{align*}
		\det\big((W_r)_{K,[\ell]}\big) = e^{-\frac{\ell}{8}r^2}(-1)^{\Sigma_{K}}\sqrt{\frac{1}{F_{[\ell]}F_K}}(\frac{r}{\sqrt{2}})^{\Sigma_K-\ell(\ell-1)/2}C_{K}.
	\end{align*}
\end{lemma}

\begin{proof}
	By factoring the common terms in each row and column, we have
	\begin{align*}
		\MoveEqLeft \det\big((W_r)_{K,[\ell+1]\backslash\{b\}}\big) \\
		& =
		e^{-\frac{\ell}{8}r^2}(-1)^{\Sigma_{K}}\sqrt{\frac{b!}{\prod_{c=0}^{\ell}c!\prod_{c=0}^{\ell-1}a_c!}}(\frac{r}{\sqrt{2}})^{\Sigma_K-\ell(\ell+1)/2+b}\det\big((W')_{K,[\ell+1]\backslash\{b\}}\big) \\
		& =
		e^{-\frac{\ell}{8}r^2}(-1)^{\Sigma_{K}}\sqrt{\frac{b!}{F_{[\ell+1]}F_K}}(\frac{r}{\sqrt{2}})^{\Sigma_K-\ell(\ell+1)/2+b}\det\big((W')_{K,[\ell+1]\backslash\{b\}}\big)
	\end{align*}
	where $W'_{K,[\ell+1]\backslash\{b\}}$ is a $\ell$-by-$\ell$ matrix that
	\begin{align*}
		\MoveEqLeft W'_{K,[\ell+1]\backslash\{b\}} \\
		& = 
		\begin{bmatrix}
			1 & \cdots & \prod_{d=0}^{b-2}(a_0-d) & \prod_{d=0}^{b}(a_0-d) & \cdots & \prod_{d=0}^{\ell-1}(a_0-d) \\
			\vdots & \vdots & \vdots &\vdots & \vdots & \vdots \\
			1 & \cdots & \prod_{d=0}^{b-2}(a_{\ell-1}-d) & \prod_{d=0}^{b}(a_{\ell-1}-d) & \cdots & \prod_{d=0}^{\ell-1}(a_{\ell-1}-d) 
		\end{bmatrix}
	\end{align*}

	We will apply column operations on $\det\big(W'_{K,[\ell+1]\backslash\{b\}}\big)$.
	For the column indexed at $c\neq 0,b+1$, subtract the column indexed at $c-1$ multiplied by $(a_0-c)$ to it and, for the column indexed at $b+1$, subtract the column indexed at $b-1$ multiplied by $(a_0-b)(a_0-(b-1))$ to it.
	We have the first row to be a zero row except the first entry is $1$ and expand the determinant along the first row.
	Then, the row indexed at $a_c$ has a factor $a_c-a_0$ and in particular the entry indexed at $b+1$ has an extra factor $a_c+a_0-(2b-1) = (a_c-(b-1))+(a_0-b)$.
	By factoring out $a_c-a_0$, we have
	\begin{align*}
		\MoveEqLeft \det\big(W'_{K,[\ell+1]\backslash\{b\}}\big)\\
		& = 
		\bigg(\prod_{c=1}^{\ell}(a_c-a_0)\bigg)\bigg(\det\big((W'(r))_{K\backslash\{a_0\},[\ell]\backslash\{b-1\}}\big) + (a_0-b)\det\big((W'(r))_{K\backslash\{a_0\},[\ell]\backslash\{b\}}\big) \bigg)
	\end{align*}
	By induction, we conclude 
	\begin{align*}
		\det \big(W'_{K,[\ell+1]\backslash\{b\}}\big) = C_{K}\cdot \Gamma_{K,b}
	\end{align*}
	Therefore, 
	\begin{align*}
		\det\big((W_r)_{K,[\ell+1]\backslash\{b\}}\big) = e^{-\frac{\ell}{8}r^2}(-1)^{\Sigma_{K}}\sqrt{\frac{b!}{F_{[\ell+1]}F_K}}(\frac{r}{\sqrt{2}})^{\Sigma_K-\ell(\ell+1)/2+b}C_{K}\cdot \Gamma_{K,b}
	\end{align*}
	and in particular
	\begin{align*}
		\det\big((W_r)_{K,[\ell]}\big)  = e^{-\frac{\ell}{8}r^2}(-1)^{\Sigma_{K}}\sqrt{\frac{1}{F_{[\ell]}F_K}}(\frac{r}{\sqrt{2}})^{\Sigma_K-\ell(\ell-1)/2}C_{K}.
	\end{align*}
\end{proof}

\begin{lemma}\label{lem:gamma}
	Let $a_0,\cdots,a_{\ell-1}$ be $\ell$ nonnegative integers such that $0\leq a_0<\cdots<a_{\ell-1}$ and $K$ be the set $\{a_0,\cdots,a_{\ell-1}\}$.
	Then, for any $b\in\mathbb{N}_{0}$, we have
	\begin{align*}
		0\leq \Gamma_{K,b} \leq \frac{1}{b!}\prod_{c=0}^{\ell-1} a_c.
	\end{align*}
\end{lemma}
\begin{proof}
	We first prove $\Gamma_{K,b} \geq 0$ and will use induction on $\ell$ to prove the statement.
	Suppose $\ell=1$.
	When $b=0$, we have $\Gamma_{K,0} = a_0 \geq 0$.
	When $b=1$, we have $\Gamma_{K,1} = a_0-(a_0-1) = 1>0$.
	When $b\geq 2$, we have $\Gamma_{K,b}= \frac{1}{b!}\sum_{d=0}^b (-1)^d{b \choose d} (a_0 - d) = 0$ since both $\sum_{d=0}^b (-1)^d{b \choose d}$  and $\sum_{d=0}^b (-1)^d{b \choose d} d = b\cdot \sum_{d=1}^b (-1)^d{b-1 \choose d-1}$ are zero.
	Suppose $\ell\geq 2$.
	We view $\Gamma_{K,b}$ as a function of $a_0,a_1,\cdots,a_{\ell-1}$.
	The partial derivative of $\Gamma_{K,b}$ is 
	\begin{align*}
		\frac{\partial}{\partial a_{i}}\Gamma_{K,b} = \frac{1}{b!}\sum_{d=0}^b (-1)^d{b \choose d}\prod_{c=0,c\neq i}^{\ell-1}(a_c-d) = \Gamma_{K\backslash\{a_{i}\},b} \geq 0
	\end{align*}
	by the induction assumption.
	It means that $\Gamma_{K,b}$ is an increasing function and it implies $\Gamma_{K,b} \geq \Gamma_{[\ell],b} = 0$ by direct calculation.
	
	Now, we will prove $\Gamma_{K,b} \leq \frac{1}{b!} \prod_{c=0}^{\ell-1} a_c$ and will use induction on $b$ to prove the statement.
	Suppose $b=1$.
	We have $\Gamma_{K,0} = \prod_{c=0}^{\ell-1} a_c$.
	Suppose $b\geq 2$.
	Note that ${b\choose d} = {b-1 \choose d} + {b-1 \choose d-1}$.
	We have
	\begin{align*}
		\Gamma_{K,b}
		& =
		\frac{1}{b!}\bigg(\sum_{d=0}^{b-1} (-1)^d{b-1 \choose d}\prod_{c=0}^{\ell-1}(a_c-d) - \sum_{d=1}^b (-1)^{d-1}{b-1 \choose d-1}\prod_{c=0}^{\ell-1}(a_c-d)\bigg) \\
		& \leq
		\frac{1}{b!}\sum_{d=0}^{b-1} (-1)^d{b-1 \choose d}\prod_{c=0}^{\ell-1}(a_c-d)
		\leq
		\frac{1}{b}\frac{1}{(b-1)!}\prod_{c=0}^{\ell-1} a_c
		 =
		\frac{1}{b!}\prod_{c=0}^{\ell-1} a_c.
	\end{align*}
	The inequalities are due to the induction assumption.
\end{proof}

\end{document}